%% file: PNAS-main.tex
\documentclass[9pt,twocolumn,twoside]{pnas-new}

\setboolean{displaywatermark}{false}

\templatetype{pnasresearcharticle} 

\usepackage{bm, mathrsfs, graphics,float,amssymb,amsmath,subeqnarray,setspace,graphicx,amsthm,epstopdf,subfigure, color}
\usepackage{enumitem}
\usepackage{soul}

\newtheorem{lemma}{Lemma}
\newtheorem{theorem}{Theorem}
\newtheorem{theorem*}{Finding}
\newtheorem{property}{Property}
\newtheorem{proposition}{Proposition}
\newtheorem{definition}{Definition}

\theoremstyle{remark}

\newtheorem{remark}[theorem]{Remark}

\renewcommand{\theta}{h}
\renewcommand{\Theta}{H}
\newcommand{\bH}{\bm{H}}
\newcommand{\tbH}{\tilde{\bm{H}}}

\renewcommand{\mathbf}{\bm}
\renewcommand{\epsilon}{\varepsilon}
\newcommand{\kk}{{k'}}
\newcommand{\kkk}{{k''}}
\newcommand{\RR}{\mathbb{R}}
\newcommand{\bx}{\mathbf{x}}

\newcommand{\bu}{\mathbf{u}}
\newcommand{\bX}{\mathbf{X}}
\newcommand{\bW}{\mathbf{W}}
\newcommand{\bw}{\mathbf{w}}
\newcommand{\by}{\mathbf{y}}
\newcommand{\bh}{\mathbf{h}}
\newcommand{\bz}{\mathbf{z}}
\newcommand{\bZ}{\mathbf{Z}}
\newcommand{\bS}{\mathbf{S}}
\newcommand{\bI}{\mathbf{I}}
\newcommand{\bv}{\mathbf{v}}
\newcommand{\cI}{\mathbf{1}}
\newcommand{\bM}{{\mathbf{M}^\star}}
\newcommand{\bmm}{\mathbf{m}}
\newcommand{\bXi}{\mathbf{\xi}}
\newcommand{\bP}{\mathbf{P}}
\newcommand{\bU}{\mathbf{U}}
\newcommand{\bV}{\mathbf{V}}
\newcommand{\bSigma}{\mathbf{\Sigma}}
\newcommand{\ba}{\mathbf{a}}
\newcommand{\cS}{\mathcal{S}}
\newcommand{\E}{\mathbb{E}}

\newcommand{\Etheta}{E_{\Theta}}
\newcommand{\Ew}{E_{W}}

\newcommand{\btheta}{\bm{\theta}}
\newcommand{\barN}{n}

\newcommand{\cL}{\mathcal{L}}

\newcommand\bff{\bm{f}}
\newcommand\wf{\bm{W}_{\textnormal{full}}}

\newcommand{\bTheta}{\mathbf{\Theta}}
\newcommand{\bartheta}{\bar{\btheta}}
\usepackage{xcolor}

\DeclareMathOperator*{\argmax}{arg\,max}
\DeclareMathOperator*{\argmin}{arg\,min}
\newcommand{\bcdot}{
	\mathop{
		\mathchoice{\vcenter{\hbox{\LARGE$\cdot$}}}
		{\vcenter{\hbox{\LARGE$\cdot$}}}
		{\vcenter{\hbox{\normalsize$\cdot$}}}
		{\vcenter{\hbox{\small$\cdot$}}}
	}
}

\title{Exploring Deep Neural Networks via Layer-Peeled Model: Minority Collapse in Imbalanced Training}

\author[a]{Cong Fang}
\author[b]{Hangfeng He} 
\author[c]{Qi Long}
\author[d,1]{Weijie J.~Su}

\affil[a]{Department of Key Laboratory of Machine Perception (MOE), Peking University}
\affil[b]{Department of Computer and Information Science, University of Pennsylvania}
\affil[c]{Department of Biostatistics, Epidemiology and Informatics, University of Pennsylvania} 
\affil[d]{Department of Statistics and Data Science, University of Pennsylvania}

\significancestatement{The remarkable development of deep learning over the past decade relies heavily on sophisticated heuristics and tricks. To better exploit its potential in the coming decade, perhaps a rigorous framework for reasoning about deep learning is needed, which however is not easy to build due to the intricate details of neural networks. For near-term purposes, a practical alternative is to develop a mathematically tractable surrogate model yet maintaining many characteristics of neural networks. This paper proposes a model of this kind that we term the Layer-Peeled Model. The effectiveness of this model is evidenced by, among others, its ability to reproduce a known empirical pattern and to predict a hitherto unknown phenomenon when training deep learning models on imbalanced datasets.}

\authorcontributions{
C.F., H.H, Q.L., and W.J.S.~designed research; C.F., H.H, and W.J.S.~performed research and analyzed data; and C.F., H.H, Q.L., and W.J.S.~wrote the paper.}
\authordeclaration{The authors declare no competing interest.}
\correspondingauthor{\textsuperscript{1}To whom correspondence should be addressed. E-mail: suw@wharton.upenn.edu.}


\begin{abstract}
In this paper, we introduce the \textit{Layer-Peeled Model}, a nonconvex yet analytically tractable optimization program, in a quest to better understand deep neural networks that are trained for a sufficiently long time. As the name suggests, this new model is derived by isolating the topmost layer from the remainder of the neural network, followed by imposing certain constraints separately on the two parts of the network. We demonstrate that the Layer-Peeled Model, albeit simple, inherits many characteristics of well-trained neural networks, thereby offering an effective tool for explaining and predicting common empirical patterns of deep learning training. First, when working on class-balanced datasets, we prove that any solution to this model forms a simplex equiangular tight frame, which in part explains the recently discovered phenomenon of neural collapse~\cite{papyan2020prevalence}. More importantly, when moving to the imbalanced case, our analysis of the Layer-Peeled
Model reveals a hitherto unknown phenomenon that we term \textit{Minority Collapse}, which fundamentally limits the performance of deep learning models on the minority classes. In addition, we use the Layer-Peeled Model to gain insights into how to mitigate Minority Collapse. Interestingly, this phenomenon is first predicted by the Layer-Peeled Model before being confirmed by our computational experiments. 
\end{abstract}

\dates{This manuscript was compiled on September 8, 2021}

\begin{document}

\maketitle
\thispagestyle{firststyle}
\ifthenelse{\boolean{shortarticle}}{\ifthenelse{\boolean{singlecolumn}}{\abscontentformatted}{\abscontent}}{}


\input{01-intro}

\input{02-model}

\input{03-balanced}
\input{04-imbalanced}

\input{05-upsampling}

\input{06-discussion}

\showmatmethods{} 

\acknow{We are grateful to X.Y.~Han for helpful discussions about some results of \cite{papyan2020prevalence} and feedback on an early version of the manuscript. We thank Gang Wen and Qinqing Zheng for helpful comments. We thank the two anonymous referees for their constructive comments that helped improve the presentation of this work. This work was supported in part by NIH through RF1AG063481, NSF through
CAREER DMS-1847415 and CCF-1934876, an Alfred Sloan Research Fellowship, and the Wharton Dean's Research Fund.}

\showacknow{} 

\bibliography{reference}

\onecolumn
\clearpage
\appendix
\input{07-proofs}

\input{08-additional}

\end{document}

%% file: 01-intro.tex
\section{Introduction}\label{sec:intro}

In the past decade, deep learning has achieved remarkable performance across a range of scientific and engineering domains \citep{krizhevsky2017imagenet,lecun2015deep, silver2016mastering}. Interestingly, these impressive accomplishments were mostly achieved by heuristics and tricks, though often plausible, without much principled guidance from a theoretical perspective. On the flip side, however, this reality suggests the great potential a theory could have for advancing the development of deep learning methodologies in the coming decade.


Unfortunately, it is not easy to develop a theoretical foundation for deep learning. Perhaps the most difficult hurdle lies in the nonconvexity of the optimization problem for training neural networks, which, loosely speaking, stems from the interaction between different layers of neural networks. To be more precise, consider a neural network for $K$-class classification (in logits), which in its simplest form reads\footnote{The softmax step is implicitly included in the loss function and we omit other operations such as max-pooling for simplicity.}
\[
\bff(\bx; \wf) = \bm{b}_L + \bW_L \sigma \left( \bm{b}_{L-1} + \bW_{L-1} \sigma(\cdots \sigma(\bm{b}_1 + \bW_1 \bx) \cdots ) \right).
\]
Here, $\wf := \{\bW_1, \bW_2, \ldots, \bW_L\}$ denotes the weights of the $L$ layers, $\{\bm{b}_1, \bm{b}_2, \ldots, \bm{b}_L\}$ denotes the biases, and $\sigma(\cdot)$ is a nonlinear activation function such as the ReLU. Owing to the complex and nonlinear interaction between the $L$ layers, when applying stochastic gradient descent to the optimization problem
\begin{equation}\label{eq:dl_opt}
\min_{\wf} ~ \frac1{N}\sum_{k=1}^K \sum_{i=1}^{n_k} \cL( \bff(\bx_{k,i}; \wf), \by_k  ) + \frac{\lambda}{2} \|\wf\|^2
\end{equation}
with a loss function $\cL$ for training the neural network, it becomes very difficult to pinpoint how a given layer influences the output $\bff$ (above, $\{\bx_{k, i}\}_{i=1}^{n_k}$ denotes the training examples in the $k$-th class, with label $\by_k$, $N = n_1 + \cdots + n_K$ is the total number of training examples, $\lambda > 0$ is the weight decay parameter, and $\|\cdot\|$ throughout the paper is the $\ell_2$ norm). Worse, this difficulty in analyzing deep learning models is compounded by an ever growing number of layers.

Therefore, any attempt to develop a tractable and comprehensive theory for demystifying deep learning would presumably first need to simplify the interaction between a large number of layers. Following this intuition, in this paper we introduce the following optimization program as a \textit{surrogate} model for  \eqref{eq:dl_opt} with the goal of unveiling quantitative patterns of deep neural networks:
\begin{equation}\label{eq:intro_peeled}
\begin{aligned}
\min_{\bW_L, \bTheta} ~ & \frac1N\sum_{k=1}^K \sum_{i=1}^{n_k} \cL( \bW_L \btheta_{k,i}, \by_k )\\
\mathrm{s.t.}~ & \frac{1}{K}\sum_{k=1}^K \left\|\bw_k \right\|^2 \leq E_{W}, \frac{1}{K}\sum_{k=1}^K \frac{1}{n_k}\sum_{i=1}^{n_k}\left\|\btheta_{k,i} \right\|^2 \leq E_{\Theta},
\end{aligned}
\end{equation}
where $\bW_L = \left[ \bw_1, \ldots, \bw_K \right]^\top \in \RR^{K \times p}$ is, as in  \eqref{eq:dl_opt}, comprised of $K$ linear classifiers in the last layer, $\bH = [\bh_{k,i}: 1 \le k \le K, 1 \le i \le n_k] \in \RR^{p \times N}$ corresponds to the $p$-dimensional last-layer activations/features of all $N$ training examples, and $E_{\Theta}$ and $E_W$ are two positive scalars. Note that the bias terms are omitted for simplicity. Although still nonconvex, this new optimization program is presumably much more amenable to analysis than the old one \eqref{eq:dl_opt} as the interaction now is only between two layers.

\begin{figure}[!htp]
		\centering
        \hspace{0.01in}
		\subfigure[1-Layer-Peeled Model]{
			\centering
			\includegraphics[scale=0.22]{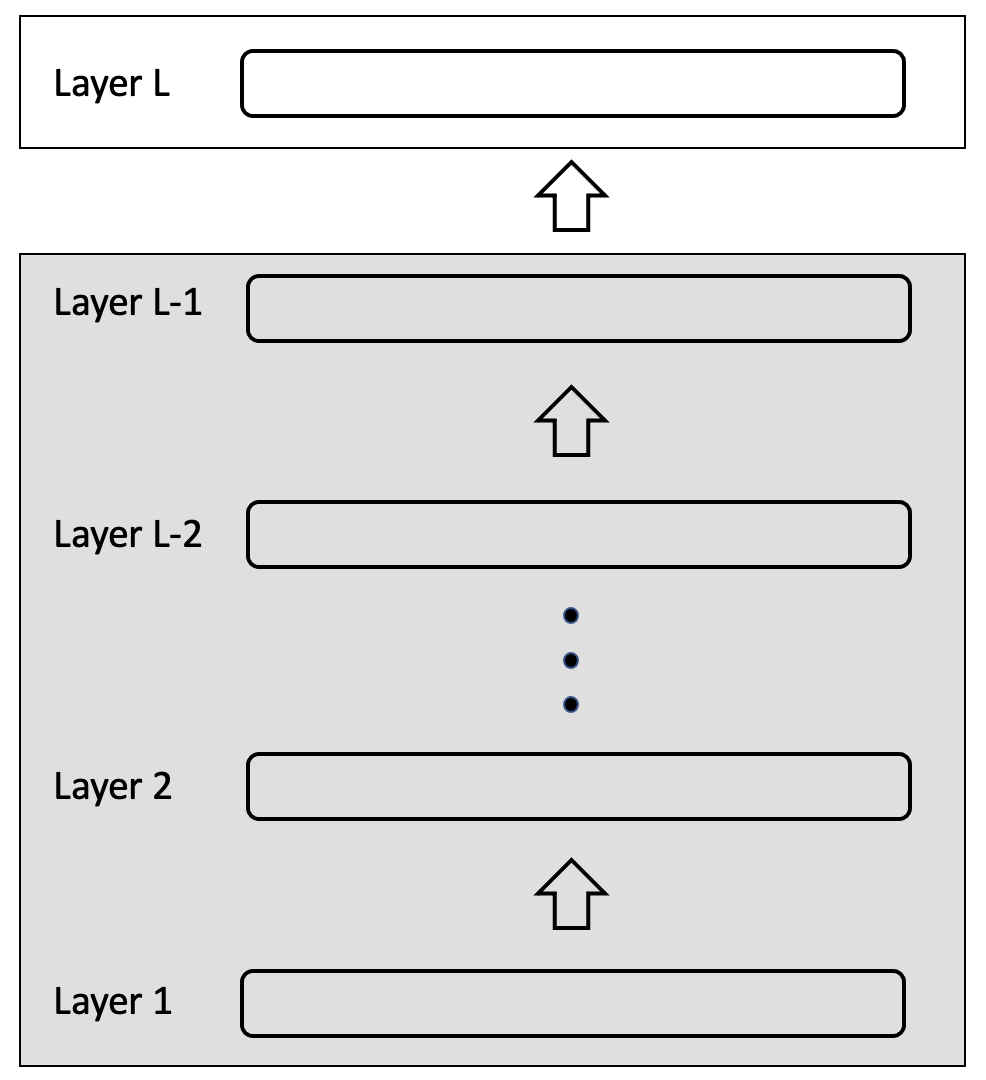}
			\label{fig:1-LPM}}
        \hspace{0.01in} 
        	\subfigure[2-Layer-Peeled Model]{
			\centering
		\includegraphics[scale=0.22]{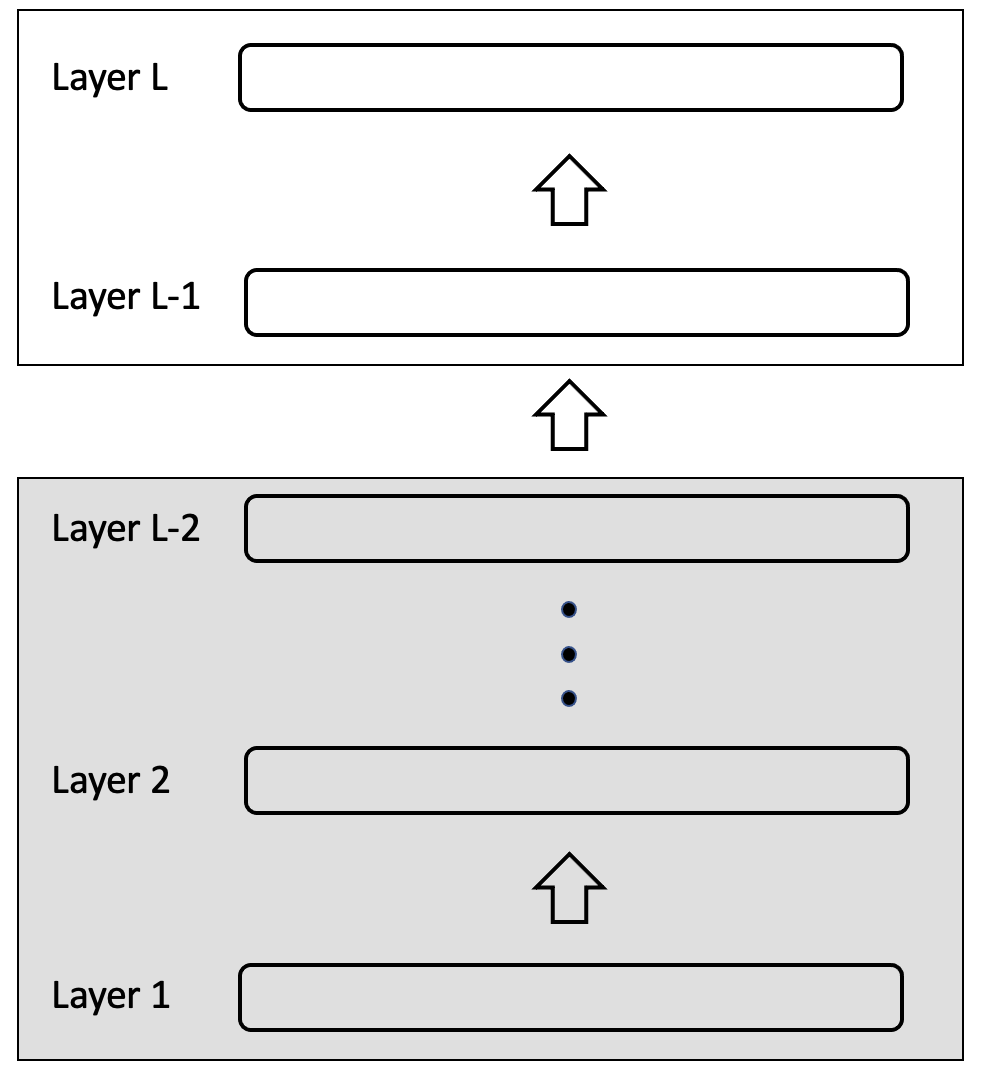}
			\label{2-LPM}}
		\caption{Illustration of Layer-Peeled Models. The right panel represents the 2-Layer-Peeled Model, which is discussed in Section~\ref{sec:discuss}. For each panel, we preserve the details of the white (top) box, whereas the gray (bottom) box is modeled by a simple decision variable for every training example.}
\label{fig:layer-peeled-model}
\end{figure}



In relating  \eqref{eq:intro_peeled} to \eqref{eq:dl_opt}, a first simple observation is that $\bff(\bx_{k,i}; \wf) = \bW_L\sigma(\bW_{L-1} \sigma(\cdots \sigma(\bW_1\bx_{k,i}) \cdots ))$ in \eqref{eq:dl_opt} is replaced by $\bW_L \bh_{k,i}$ in \eqref{eq:intro_peeled}. Put differently, the black-box nature of the last-layer features, namely $\sigma(\bW_{L-1} \sigma(\cdots \sigma(\bW_1\bx_{k,i}) \cdots ))$, is now modeled by a simple decision variable $\bh_{k,i}$ for each training example, with an overall constraint on their $\ell_2$ norm. Intuitively speaking, this simplification is done by \textit{peeling} off the topmost layer from the neural network. Thus, we call the optimization program (\ref{eq:intro_peeled}) the \textit{1-Layer-Peeled Model}, or simply the \textit{Layer-Peeled Model}.


At a high level, the Layer-Peeled Model takes a \textit{top-down} approach to the analysis of deep neural networks. As illustrated in Figure~\ref{fig:layer-peeled-model}, the essence of the modeling strategy is to break down the neural network from top to bottom, specifically singling out the topmost layer and modeling all bottom layers collectively as a single variable. In fact, the top-down perspective that we took in the development of the Layer-Peeled Model was inspired by a recent breakthrough made by Papyan, Han, and Donoho~\cite{papyan2020prevalence}, who discovered a mathematically elegant and pervasive phenomenon termed neural collapse in deep learning training. This top-down approach was also taken in \cite{webb1990optimised,soudry2018implicit,oymak2020toward,yu2020learning,shamir2020gradient} to investigate various aspects of deep learning models.


\subsection{Two Applications}

Despite its plausibility, the ultimate test of the Layer-Peeled Model lies in its ability to faithfully approximate deep learning models through explaining empirical observations and, better, predicting new phenomena. In what follows, we provide convincing evidence that the Layer-Peeled Model is up to this task by presenting two findings. To be concrete, we remark that the results below are concerned with well-trained deep learning models, which correspond to, in rough terms, (near) optimal solutions of  \eqref{eq:dl_opt}.



\paragraph{Balanced Data.}
Roughly speaking, neural collapse~\cite{papyan2020prevalence} refers to the emergence of certain geometric patterns of the last-layer features $\sigma(\bW_{L-1} \sigma(\cdots \sigma(\bW_1\bx_{k,i}) \cdots ))$ and the last-layer classifiers $\bW_L$, when the neural network for \textit{balanced} classification problems is well-trained in the sense that it is toward not only zero misclassification error but also negligible\footnote{Strictly speaking, in the presence of an $\ell_2$ regularization term, which is equivalent to weight decay, the cross-entropy loss evaluated at any global minimizer of  \eqref{eq:dl_opt} is bounded away from $0$.} cross-entropy loss. Specifically, the authors observed the following properties in their massive experiments: the last-layer features from the same class tend to be very close to their class mean; these $K$ class means centered at the global-mean have the same length and form the maximally possible equal-sized angles between any pair; moreover, the last-layer classifiers become dual to the class means in the sense that they are equal to each other for each class up to a scaling factor. See a more precise description in Section~\ref{sec:related_w}.


While it seems hopeless to rigorously prove neural collapse for multiple-layer neural networks~\eqref{eq:dl_opt} at the moment, alternatively, we seek to show that this phenomenon emerges in the surrogate model \eqref{eq:intro_peeled}. More precisely, when the size of each class $n_k = n$ for all $k$, is it true that any global minimizer $\bW_L^\star = \left[ \bw_1^\star, \ldots, \bw_K^\star \right]^\top, \bH^\star = [\bh_{k,i}^\star: 1 \le k \le K, 1 \le i \le n]$ of  \eqref{eq:intro_peeled} exhibits neural collapse? The following result answers this question in the affirmative:

\begin{theorem*}
Neural collapse occurs in the Layer-Peeled Model.
\end{theorem*}

A formal statement of this result and a detailed discussion are given in Section~\ref{sec:results}. 

This result applies to a family of loss functions $\cL$, particularly including the cross-entropy loss and the contrastive loss (see, e.g., \cite{chen2020simple}). As an immediate implication, this result provides evidence of the Layer-Peeled Model's ability to characterize well-trained deep learning models.  



\paragraph{Imbalanced Data.} While a surrogate model would be satisfactory if it explains some already observed phenomenon, we set a \textit{higher} standard for the model, asking whether it can predict a \textit{new} common empirical pattern. Encouragingly, the Layer-Peeled Model happens to meet this standard. Specifically, we consider training deep learning models on imbalanced datasets, where some classes contain many more training examples than others. Despite the pervasiveness of imbalanced classification in many practical applications~\cite{johnson2019survey}, the literature remains scarce on its impact on the trained neural networks from a theoretical standpoint. Here we provide mathematical insights into this problem by using the Layer-Peeled Model. In the following result, we consider optimal solutions to the Layer-Peeled Model on a dataset with two different class sizes: the first $K_A$ majority classes each contain $\barN_A$ training examples ($n_1 = n_2 = \dots = n_{K_A} = \barN_A$), and the remaining $K_B := K - K_A$ minority classes each contain $\barN_B$ examples ($n_{K_A+1} = n_{K_A+2}= \dots = n_{K} = \barN_B$). We call $R := \barN_A/ \barN_B > 1$ the imbalance ratio. 


\begin{theorem*}
In the Layer-Peeled Model, the last-layer classifiers corresponding to the minority classes, namely $\bw^\star_{K_A+1}, \bw^\star_{K_A+2}, \ldots, \bw^\star_K$, collapse to a single vector when $R$ is sufficiently large.
\end{theorem*}

This result is elaborated on in Section~\ref{sec:imbalanced}. The derivation involves some novel elements to tackle the nonconvexity of the Layer-Peeled Model \eqref{eq:intro_peeled} and the asymmetry due to the imbalance in class sizes.


\begin{figure}[!htp]
		\centering
		\includegraphics[scale=0.35]{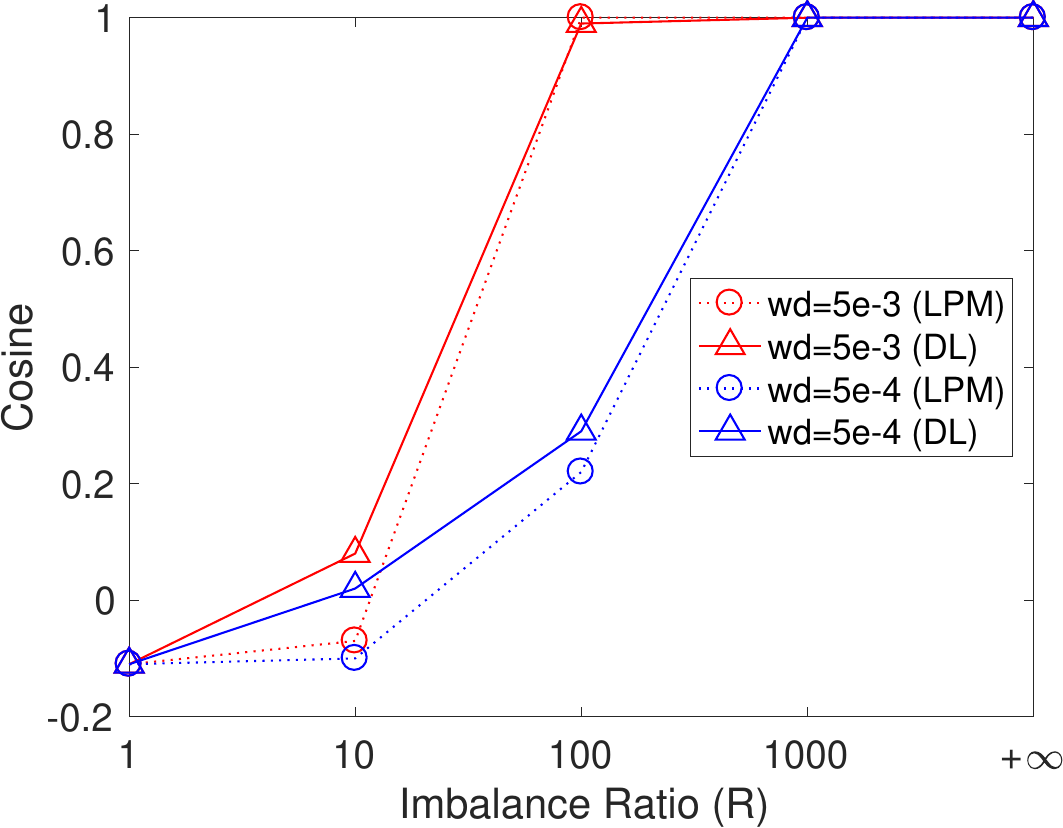}
		\caption{
		Minority Collapse predicted by the Layer-Peeled Model (LPM, in dotted lines) and empirically observed in deep learning (DL, in solid lines) on imbalanced datasets with $K_A=7$ and $K_B = 3$. The $y$-axis denotes the average cosine of the angles between any pair of the minority classifier $\bw_{K_A+1}^\star, \ldots, \bw_K^\star$ for both LPM and DL. The datasets we use are subsets of the CIFAR10 datasets \citep{krizhevsky2009learning} and the size of the majority classes is fixed to $5000$. The experiments use VGG13 \citep{simonyan2014very} as the deep learning architecture, with weight decay (wd) $\lambda = 5 \times 10^{-3}, 5\times 10^{-4}$. The prediction is especially accurate in capturing the phase transition point where the cosine becomes $1$ or, equivalently, the minority classifiers become parallel to each other. More details can be found in Section \ref{subsec:experiments}.}	
\label{fig:simulation-weight-decay}
\end{figure}

In slightly more detail, we identify a phase transition as the imbalance ratio $R$ increases: when $R$ is below a threshold, the minority classes are distinguishable in terms of their last-layer classifiers; when $R$ is above the threshold, they become indistinguishable. While this phenomenon is merely predicted by the simple Layer-Peeled Model \eqref{eq:intro_peeled}, it appears in our computational experiments on deep neural networks. More surprisingly, our prediction of the phase transition point is in excellent agreement with the experiments, as shown in Figure~\ref{fig:simulation-weight-decay}.

This phenomenon, which we refer to as \textit{Minority Collapse}, reveals the fundamental difficulty in using deep learning for classification when the dataset is widely imbalanced, even in terms of optimization, not to mention generalization. This is not a priori evident given that neural networks have a large approximation capacity (see, e.g., \cite{yarotsky2017error}). Importantly, Minority Collapse emerges at a finite value of the imbalance ratio rather than at infinity. Moreover, even below the phase transition point of this ratio, we find that the angles between any pair of the minority classifiers are already smaller than those of the majority classes, both theoretically and empirically.

\subsection{Related Work}\label{sec:related_w}
There is a venerable line of work attempting to gain insights into deep learning from a theoretical point of view~\cite{jacot2018neural,du2018gradient,allen2019convergence,zou2018stochastic,chizat2019lazy,ma2019comparative,bartlett2017spectrally,he2019local,poggio2020theoretical,MeiE7665,sirignano2019mean,rotskoff2018neural,fang2020modeling,kuditipudi2019explaining,shi2020learning}. See also the reviews \cite{9326403,he2020recent,fan2019selective,sun2019optimization} and references therein. 

The work of neural collapse by \cite{papyan2020prevalence} in this body of work is particularly noticeable with its mathematically elegant and convincing insights. In brief, \cite{papyan2020prevalence} observed the following four properties of the last-layer features and classifiers in deep learning training on balanced datasets:\footnote{See the mathematical description of neural collapse in Theorem~\ref{theo: cross-entropy balance}.}
\begin{itemize}
    \item[] \hypertarget{(NC1)}{(NC1)}
     Variability collapse: the within-class variation of the last-layer features becomes $0$,  which means that these features collapse to their class means. 
    \item[] \hypertarget{(NC2)}{(NC2)} The class means centered at their global mean collapse to the vertices of a simplex equiangular tight frame (ETF) up to scaling.\label{nc2}
     \item[] \hypertarget{(NC3)}{(NC3)}  Up to scaling, the last-layer classifiers each collapse to the corresponding class means.\label{nc3}
     \item[] \hypertarget{(NC4)}{(NC4)} The network's decision collapses to simply choosing the class with the closest Euclidean distance between its class mean and the activations of the test example.\label{nc4}
\end{itemize}


 Now we give the formal definition of ETF \cite{Strohmer2003, papyan2020prevalence}. %
\begin{definition}\label{def: ETF}
A $K$-simplex ETF is a collection of points in $\RR^p$ specified by the columns of the matrix
$$  \bM = \sqrt{\frac{K}{K-1}} \bP \left( \bI_K - \frac{1}{K}\cI_K\cI_K^\top \right), $$
where $\bI_K\in \RR^{K\times K}$ is the identity matrix, $\cI_K$ is the ones vector, and
$\bP\in \RR^{p\times K}$  ($p\ge K$)\footnotemark  ~is a partial orthogonal matrix such that $\bP^\top \bP = \bI_K$.
\end{definition}
\footnotetext{To be complete, we only require $p\geq K-1$. When $p=K-1$, we can choose $\bP$ such that $\left[\bP^\top, \cI_K\right]$ is an orthogonal matrix.}

A common setup of the experiments for validating neural collapse is the use of the cross-entropy loss with $\ell_2$ regularization, which corresponds to weight decay in stochastic gradient descent. Based on convincing arguments and numerical evidence, \cite{papyan2020prevalence} demonstrated that the symmetry and stability of neural collapse improve deep learning training in terms of generalization, robustness, and interpretability. Notably, these improvements occur with the benign overfitting phenomenon (see~\cite{ma2018power,belkin2019reconciling,liang2020just,bartlett2020benign,li2020benign}) during the terminal phase of training---when the trained model interpolates the in-sample training data. 

In passing, we remark that concurrent works \cite{mixon2020neural,wojtowytsch2020emergence,lu2020neural,ergen2020convex} produced neural collapse using different surrogate models. In slightly more detail, \cite{mixon2020neural,wojtowytsch2020emergence,lu2020neural} obtained their models by peeling off the topmost layer. The difference, however, is that \cite{wojtowytsch2020emergence,lu2020neural} considered models that impose a norm constraint for each class, as opposed to an overall constraint as employed in the Layer-Peeled Model. Moreover, \cite{mixon2020neural} analyzed gradient flow with an unconstrained features model using the squared loss instead of the cross-entropy loss. The work \cite{ergen2020convex} provided an insightful perspective for the analysis of neural networks using convex duality. Relying on a convex formulation that is in the same spirit as our semidefinite programming relaxation, the authors of \cite{ergen2020convex} observed neural collapse in their ReLU-based model by leveraging strong duality under certain conditions.

%% file: 02-model.tex
\section{Derivation}
\label{sec:derive}


In this section, we heuristically derive the Layer-Peeled Model as an analytical surrogate for well-trained neural networks. Although our derivation lacks rigor, the goal is to reduce the complexity of the optimization problem \eqref{eq:dl_opt} while roughly preserving its structure. Notably, the penalty $\frac{\lambda}{2} \|\wf\|^2$ corresponds to weight decay used in training deep learning models, which is necessary for preventing this optimization program from attaining its minimum at infinity when $\cL$ is the cross-entropy loss. For simplicity, we omit the biases in the neural network $\bff(\bx_{k,i}; \wf)$.


Taking a top-down standpoint, our modeling strategy starts by singling out the weights $\bW_L$ of the topmost layer and rewriting \eqref{eq:dl_opt} as
\begin{equation}\label{eq:dp_separate}
\begin{aligned}
\min_{\bW_L, \bH} ~& \frac{1}{N} \sum_{k=1}^K\sum_{i=1}^{n_k} \cL( \bW_L \bh(\bx_{k,i}; \bW_{-L}), \by_k)\\
&+ \frac{\lambda}{2} \|\bW_L\|^2 + \frac{\lambda}{2} \|\bW_{-L}\|^2,
\end{aligned}
\end{equation}
where the last-layer feature function $\bh(\bx_{k,i}; \bW_{-L}) := \sigma(\bW_{L-1} \sigma(\cdots \sigma(\bW_1\bx_{k,i}) \cdots ))$ and $\bW_{-L}$ denotes the weights from all layers but the last layer. From the Lagrangian dual viewpoint, a minimum of the optimization program above is also an optimal solution to
\begin{equation}\label{eq:lp_der_mid}
\begin{aligned}
\min_{\bW_L, \bW_{-L}} \quad & \frac{1}{N} \sum_{k=1}^K\sum_{i=1}^{n_k} \cL( \bW_L \bh(\bx_{k,i}; \bW_{-L}), \by_k)\\
\mathrm{s.t.} \quad & \|\bW_L\|^2 \le C_1, \|\bW_{-L}\|^2 \le C_2,
\end{aligned}
\end{equation}
for some positive numbers $C_1$ and $C_2$.\footnote{Denoting by $(\bW_L^\star, \bW_{-L}^\star)$ an optimal solution to \eqref{eq:dp_separate}, then we can take $C_1 = \|\bW_L^\star\|^2$ and $C_2 = \|\bW_{-L}^\star\|^2$.} To clear up any confusion, note that due to its nonconvexity, \eqref{eq:dp_separate} may admit multiple global minima and each in general corresponds to different values of $C_1, C_2$. Next, we can equivalently write \eqref{eq:lp_der_mid} as
\begin{equation}\label{eq:lp_der_almost}
\begin{aligned}
\min_{\bW_L, \bH} \quad & \frac{1}{N} \sum_{k=1}^K\sum_{i=1}^{n_k} \cL( \bW_L \bh_{k,i}, \by_k)\\
\mathrm{s.t.} \quad & \|\bW_L\|^2 \le C_1,\\
 \quad & \bH \in \left\{ \bH(\bW_{-L}): \|\bW_{-L}\|^2 \le C_2 \right\},
\end{aligned}
\end{equation}
where $\bH = [\bh_{k,i}: 1 \le k \le K, 1 \le i \le n_k]$ denotes a decision variable and the function $\bH(\bW_{-L})$ is defined as $\bH(\bW_{-L}) := \left[ \bh(\bx_{k,i}; \bW_{-L}): 1 \le k \le K, 1 \le i \le n_k \right]$ for any $\bW_{-L}$. 

To simplify \eqref{eq:lp_der_almost}, we make the \textit{ansatz} that the range of $\bh(\bx_{k,i}; \bW_{-L})$ under the constraint $\|\bW_{-L}\|^2 \le C_2$ is approximately an ellipse in the sense that 
\begin{equation}\label{eq:ansatz}
\begin{aligned}
&\left\{ \bH(\bW_{-L}): \|\bW_{-L}\|^2 \le C_2 \right\} \\
&\quad\quad\quad\quad\quad\quad\quad\quad\approx \left\{\bH: \sum_{k=1}^K \frac1{n_k} \sum_{i=1}^{n_k} \|\bh_{k,i}\|^2 \le C_2' \right\}
\end{aligned}
\end{equation}
for some $C_2' > 0$. Loosely speaking, this ansatz asserts that $\bH$ should be regarded as a variable in an $\ell_2$ space. To shed light on the rationale behind the ansatz, note that $\bh_{k,i}$ intuitively lives in the dual space of $\bW$ in view of the appearance of the product $\bW \bh_{k,i}$ in the objective. Furthermore, $\bW$ is in an $\ell_2$ space for the $\ell_2$ constraint on it. Last, note that $\ell_2$ spaces are self-dual. 


Inserting this approximation into \eqref{eq:lp_der_almost}, we obtain the following optimization program, which we call the Layer-Peeled Model:
\begin{equation}\label{eq: NN simplified model}
\begin{aligned}
\min_{\bW,\bTheta} \quad& \frac{1}{N} \sum_{k=1}^K \sum_{i=1}^{n_k} \cL( \bW\btheta_{k,i}, \by_k )\\
\mathrm{s.t.}\quad &     \frac{1}{K}\sum_{k=1}^K \left\|\bw_k \right\|^2 \leq E_{W},\\
&\frac{1}{K}\sum_{k=1}^K \frac{1}{n_k}\sum_{i=1}^{n_k}\left\|\btheta_{k,i} \right\|^2 \leq E_{\Theta}.
\end{aligned}
\end{equation}
For simplicity, above and henceforth we write $\bW := \bW_L \equiv [\bw_1, \ldots, \bw_K]^\top$ for the last-layer classifiers/weights and the thresholds $E_W = C_1/K$ and $E_{\Theta} = C_2'/K$.

This optimization program is nonconvex but, as we will show soon, is generally mathematically tractable for analysis. On the surface, the Layer-Peeled Model has no dependence on the data $\{\bx_{k,i}\}$, which however is not the correct picture, since the dependence has been implicitly incorporated into the threshold $E_H$. 


In passing, we remark that neural collapse does \textit{not} emerge if the second constraint of \eqref{eq: NN simplified model} uses the $\ell_q$ norm for any $q \ne 2$ (strictly speaking, $\ell_q$ is not a norm when $q < 1$), in place of the $\ell_2$ norm. This fact in turn justifies in part the ansatz \eqref{eq:ansatz}. This result is formally stated in Proposition~\ref{theo:counter example} in Section~\ref{sec:discuss}.

%% file: 03-balanced.tex
\section{Layer-Peeled Model for Explaining Neural Collapse}\label{sec:results}
In this section, we consider training deep neural networks on a balanced dataset---that is, $n_k = n$ for all classes $1 \le k \le K$. Our main finding is that the Layer-Peeled Model displays the neural collapse phenomenon, just as in deep learning training~\cite{papyan2020prevalence}. The proofs are all deferred to SI Appendix. Throughout this section, we assume $p \ge K-1$ unless otherwise specified. This assumption is satisfied in many popular network architectures, where $p$ is usually tens or hundreds of times of $K$.

\subsection{Cross-Entropy Loss}

The cross-entropy loss is perhaps the most popular loss used in training deep learning models for classification tasks. This loss function takes the form
\begin{equation}\label{eq:cross-entropy}
    \cL( \bz, \by_k  ) = -\log\left(    \frac{\exp(\bz(k))}{\sum_{\kk=1}^K \exp(\bz(\kk))} \right),
\end{equation}
where $\bz(k')$ denotes the $k'$-th entry of the logit $\bz$. Recall that $\by_k$ is the label of the $k$-th class and the feature $\bz$ is set to $\bW \bh_{k,i}$ in the Layer-Peeled Model \eqref{eq: NN simplified model}. In contrast to the complex deep neural networks, which are often considered a black-box,  the Layer-Peeled Model is much easier to deal with. As an exemplary use case, the following result shows that any minimizer of the Layer-Peeled Model \eqref{eq: NN simplified model} with the cross-entropy loss admits an almost closed-form expression.


\begin{theorem}\label{theo: cross-entropy balance}
In the balanced case,  any global minimizer $\bW^\star \equiv \left[ \bw_1^\star, \ldots, \bw_K^\star \right]^\top, \bH^\star \equiv [\bh_{k,i}^\star: 1 \le k \le K, 1 \le i \le n]$ of \eqref{eq: NN simplified model} with the cross-entropy loss obeys
\begin{equation}\label{eq:solution}
\btheta_{k,i}^\star = C \bw_k^\star = C' \bmm_k^\star
\end{equation}
for all $1 \le i \le n, 1 \le k \le K$, where the constants $C = \sqrt{\Etheta/\Ew}, C' = \sqrt{\Etheta}$, and the matrix $[\bmm_1^\star, \ldots, \bmm_K^\star]$ forms a $K$-simplex ETF specified in Definition \ref{def: ETF}.

\end{theorem}


\begin{remark}
Note that the minimizers $(\bW^\star, \bH^\star)$'s are equivalent to each other up to rotation. This is because of the rational invariance of simplex ETFs (see $\bP$ in Definition \ref{def: ETF}). 
\end{remark}

This theorem demonstrates the highly symmetric geometry of the last-layer features and weights of the Layer-Peeled Model, which is precisely the phenomenon of neural collapse. Explicitly, \eqref{eq:solution} says that all within-class (last-layer) features are the same: $\bh_{k,i}^\star = \bh_{k, i'}^\star$ for all $1 \le i, i' \le n$; next, it also says that the $K$ class-mean features $\bh_k^\star := \bh_{k,i}^\star$ together exhibit a $K$-simplex ETF up to scaling, from which we immediately conclude that
\begin{equation}\label{eq:cos}
\cos \measuredangle(\bh_k^{\star},\bh_{\kk}^{\star}) = -\frac{1}{K-1} 
\end{equation}
for any $k \ne k'$ by Definition~\ref{def: ETF};\footnote{Note that the cosine value $-\frac{1}{K-1}$ corresponds to the largest possible angle for any $K$ points that have an equal $\ell_2$ norm and equal-sized angles between any pair. As pointed out in \cite{papyan2020prevalence}, the largest angle implies a large-margin solution~\cite{soudry2018implicit}.} in addition, \eqref{eq:solution} also displays the precise duality between the last-layer classifiers and features. Taken together, these facts indicate that the minimizer $\left(\bW^\star, \bH^\star\right)$ satisfies exactly (\hyperlink{(NC1)}{NC1})--(\hyperlink{(NC3)}{NC3}). Last, Property (\hyperlink{(NC4)}{NC4}) is also satisfied by recognizing that, for any given last-layer features $\bh$, the predicted class is $\argmax_{k} \bw_k^\star \bcdot \bh$, where $\bm a \bcdot \bm b$ denotes the inner product of the two vectors. Note that the prediction satisfies
\[
\argmax_{k}\bw_k^\star\bcdot \bh = \argmax_{k}\bh_k^\star\bcdot \bh =\argmin_{k}\|\bh_k^\star- \bh\|^2.     
\]

Conversely, the presence of neural collapse in the Layer-Peeled Model offers evidence of the effectiveness of our model as a tool for analyzing neural networks. To be complete, we remark that other models were very recently proposed to justify the neural collapse phenomenon \cite{mixon2020neural,wojtowytsch2020emergence,lu2020neural} (see also \cite{poggio2020explicit}). 


\subsection{Extensions to Other Loss Functions}
In the modern practice of deep learning, various loss functions are employed to take into account the problem characteristics. Here we show that the Layer-Peeled Model continues to exhibit the phenomenon of neural collapse for some popular loss functions.



\paragraph{Contrastive Loss.} 



Contrastive losses have been extensively used recently in both supervised and unsupervised deep learning~\cite{pennington2014glove,arora2019theoretical, chen2020simple,baevski2020wav2vec}. These losses pull similar training examples together in their embedding space while pushing apart dissimilar examples. Here we consider the supervised contrastive loss~\cite{khosla2020supervised}, which (in the balanced case) is defined through the last-layer features by introducing $\cL_c$ as 
\begin{equation}\label{eq:contrastive loss}
\frac{1}{\barN}\sum_{j=1}^{\barN} -\log\left( \frac{\exp(\btheta_{k,i}\bcdot \btheta_{k,j} /\tau)}{ \sum_{\kk=1}^K\sum_{\ell=1}^{\barN} \exp (\btheta_{k,i}\bcdot \btheta_{\kk, \ell} /\tau)}    \right),
\end{equation}
where $\tau > 0$ is a parameter. Note that this loss function uses the label information implicitly. As the loss does not involve the last-layer classifiers explicitly, the Layer-Peeled Model in this case takes the form\footnote{In \eqref{eq:contrastive loss}, $\bh_{k,i} \equiv \bh(\bx_{k,i}, \bW_{-L})$ depends on the data, whereas in \eqref{eq:contrastive loss1} $\bh_{k,i}$'s form the decision variable $\bH$.} 
\begin{equation}\label{eq:contrastive loss1}
    \begin{aligned}
    \min_{\bTheta}\quad &  \frac1N \sum_{k=1}^K \sum_{i=1}^n \cL_c (\bh_{k,i}, \by_k)\\
    \mathrm{s.t.}\quad&\frac{1}{K}\sum_{k=1}^K\frac{1}{\barN}\sum_{i=1}^{\barN}\left\|\btheta_{k,i} \right\|^2 \leq \Etheta.
    \end{aligned}
\end{equation}


We show that this Layer-Peeled Model also exhibits neural collapse in its last-layer features, even though the label information is not explicitly  explored in the loss.  


\begin{theorem}\label{proposition: contrastive loss balance}
Any global minimizer of \eqref{eq:contrastive loss1} satisfies
\begin{equation}\label{eq:solution2}
  \bh_{k,i}^\star = \sqrt{E_H}\bmm_k^\star  
\end{equation}
for all $1 \le k \le K$ and $1 \le i \le n$, where $[\bmm_1^\star, \ldots, \bmm_K^\star]$ forms a $K$-simplex ETF.

\end{theorem}


Theorem~\ref{proposition: contrastive loss balance} shows that the contrastive loss in the associated Layer-Peeled Model does a perfect job in pulling together training examples from the same class. Moreover, as seen from the denominator in \eqref{eq:contrastive loss}, minimizing this loss would intuitively render the between-class inner products of last-layer features as small as possible, thereby pushing the features to form the vertices of a $K$-simplex ETF up to scaling.

\paragraph{Softmax-Based Loss.} 
The cross-entropy loss can be thought of as a softmax-based loss. To see this, define the softmax transform as
\begin{equation}\nonumber
   \bS(\bz) =   \left[\frac{\exp(\bz(1))}{\sum_{k=1}^K \exp(\bz(k))}, \ldots, \frac{\exp(\bz(K))}{\sum_{k=1}^K \exp(\bz(k))}                        \right]^\top
\end{equation}
for $\bz \in \RR^K$. Let $g_1$ be any nonincreasing convex function and $g_2$ be any nondecreasing convex function, both defined on $(0, 1)$. We consider a softmax-based loss function that takes the form
\begin{equation}\label{eqn:general loss}
    \cL( \bz, \by_k  ) =  
    g_1\left(\bS(\bz)(k)\right)+ \sum_{\kk=1, ~\kk\neq k}^K g_2\left( \bS(\bz)(\kk) \right).
\end{equation}
Here, $\bS(\bz)(k)$ denotes the $k$-th element of $\bS(\bz)$. Taking $g_1(x) = -\log x$ and $g_2 \equiv 0$, we recover the cross-entropy loss. Another example is to take $g_1(x) = (1-x)^q$ and $g_2(x) =x^q$ for $q > 1$, which can be implemented in most deep learning libraries such as PyTorch~\citep{paszke2019pytorch}.


We have the following theorem regarding the softmax-based loss functions in the balanced case.


\begin{theorem}\label{proposition: general loss balance}
Assume $\sqrt{\Etheta\Ew} > \frac{K-1}{K}\log\left(K^2\sqrt{\Etheta\Ew} +(2K-1)(K-1)\right)$. For any loss function defined in \eqref{eqn:general loss}, $(\bW^\star, \bH^\star)$ given by \eqref{eq:solution} is a global minimizer of  \eqref{eq: NN simplified model}.  Moreover, if $g_2$ is strictly convex and at least one of $g_1, g_2$ is strictly monotone, then any global minimizer must be given by \eqref{eq:solution}.
\end{theorem}

In other words, neural collapse continues to emerge with softmax-based losses under mild regularity conditions. The first part of this theorem does not preclude the possibility that the Layer-Peeled Model admits solutions other than \eqref{eq:solution}. When applied to the cross-entropy loss, it is worth pointing out that this theorem is a weak version of Theorem~\ref{theo: cross-entropy balance}, albeit more general. Regarding the first assumption in Theorem~\ref{proposition: general loss balance}, note that $\Etheta$ and $\Ew$ would be arbitrarily large if the weight decay $\lambda$ in \eqref{eq:dl_opt} is sufficiently small, thereby meeting the assumption concerning $\sqrt{\Etheta\Ew}$ in this theorem.


We remark that Theorem \ref{proposition: general loss balance} does not require the convexity of the loss $\cL$. To circumvent the hurdle of nonconvexity, our proof in SI Appendix presents several novel elements.



In passing, we leave the experimental confirmation of neural collapse with these loss functions for future work.

%% file: 04-imbalanced.tex
\section{Layer-Peeled Model for Predicting Minority Collapse}\label{sec:imbalanced}


Deep learning models are often trained on datasets where there is a disproportionate ratio of observations in each class~\cite{wang2016training, huang2016learning, madasamy2017data}. For example, in the Places2 challenge dataset \cite{zhou2016places}, the number of images in its majority scene categories is about eight times that in its minority classes. Another example is the Ontonotes dataset for part-of-speech tagging~\citep{hovy2006ontonotes}, where the number of words in its majority classes can be more than one hundred times that in its minority classes. While empirically the imbalance in class sizes often leads to inferior model performance of deep learning (see, e.g., \cite{johnson2019survey}), there remains a lack of a solid theoretical footing for understanding its effect, perhaps due to the complex details of deep learning training.

In this section, we use the Layer-Peeled Model to seek a fine-grained characterization of how class imbalance impacts neural networks that are trained for a sufficiently long time. In particular, neural collapse no longer emerges in the presence of class imbalance (see numerical evidence in Figure S2 in SI Appendix). Instead, our analysis predicts a phenomenon we term \textit{Minority Collapse}, which fundamentally limits the performance of deep learning especially on the minority classes, both theoretically and empirically. All omitted proofs are relegated to SI Appendix.




\subsection{Technique: Convex Relaxation}
When it comes to imbalanced datasets, the Layer-Peeled Model no longer admits a simple expression for its minimizers as in the balanced case, due to the lack of symmetry between classes. This fact results in, among others, an added burden on numerically computing the solutions of the Layer-Peeled Model.


To overcome this difficulty, we introduce a convex optimization program as a relaxation of the nonconvex Layer-Peeled Model \eqref{eq: NN simplified model}, relying on the well-known result for relaxing a quadratically constrained quadratic program as a semidefinite program (see, e.g., \cite{sturm2003cones}). To begin with, defining $\bh_k$ as the feature mean of the $k$-th class (i.e., $\bh_k := \frac{1}{n_k}\sum_{i=1}^{n_k} \bh_{k,i}$), we introduce a new decision variable $\bX :=  \left[\btheta_{1},\btheta_{2},\dots,\btheta_{K},  \bW^\top  \right]^\top \left[\btheta_{1},\btheta_{2},\dots,\btheta_{K},  \bW^\top  \right]\in \RR^{2K\times 2K}$. By definition, $\bX$ is positive semidefinite and satisfies
$$
 \frac{1}{K}\sum_{k=1}^K \bX(k,k)\notag\\
 = \frac{1}{K} \sum_{k=1}^K\|\bh_k \|^2
 \overset{a}{\leq}  \frac{1}{K}\sum_{k=1}^K \frac{1}{n_k}\sum_{i=1}^{n_k}\left\|\btheta_{k,i} \right\|^2\notag\\
 \leq E_{\Theta} 
$$
and
$$ 
 \frac{1}{K}\sum_{k=K+1}^{2K} \bX(k,k)
 = \frac{1}{K} \sum_{k=1}^K\|\bw_k \|^2\leq \Ew, \notag
 $$
where $\overset{a}\leq$ follows from the Cauchy--Schwarz inequality. Thus, we consider the following semidefinite programming problem:\footnote{Although 
 \eqref{eq:convex sdp problem} involves a semidefinite  constraint, it is not a semidefinite  program in the strict sense because a semidefinite program uses a linear objective function.}
\begin{align}\label{eq:convex sdp problem}
     \min_{\bX\in \RR^{2K\times 2K}}~& \sum_{k=1}^K \frac{n_k}{N}  \cL( \bz_k, \by_k  )\notag\\
\mathrm{s.t.}~&  \bX \succeq0,\\
&\frac{1}{K}\sum_{k=1}^K \bX(k,k) \leq \Etheta,~ \frac{1}{K}\sum_{k=K+1}^{2K} \bX(k,k) \leq \Ew,\notag\\
&\text{ for all }  1 \le k \le K,\notag\\
&\bz_k = \left[\bX(k,K+1), \bX(k,K+2),\dots,\bX(k,2K)    ~\right]^\top.\notag
\end{align}

Lemma \ref{theo:to convex} below relates the solutions of \eqref{eq:convex sdp problem} to that of \eqref{eq: NN simplified model}.

\begin{lemma}\label{theo:to convex}
Assume $p\geq 2K$ and the loss function $\cL$ is convex in its first argument.   Let $\bX^{\star}$ be a minimizer of the convex program (\ref{eq:convex sdp problem}).  Define $\left(\bTheta^{\star}, \bW^{\star}\right)$ as
\begin{equation}
    \begin{aligned}
     &\left[\btheta_{1}^{\star},\btheta_{2}^{\star},\dots,\btheta_{K}^{\star}, ~ (\bW^{\star})^\top \right] =\bP (\bX^{\star})^{1/2} ,\\
&\btheta_{k,i}^{\star} = \btheta_{k}^{\star}, ~\text{ for all } 1 \le i \le n, 1 \le k \le K,\label{eq:general solution} 
    \end{aligned}
\end{equation}
where $(\bX^{\star})^{1/2}$ denotes the positive square root of $\bX^{\star}$ and $\bP\in \RR^{p\times 2K}$ is any partial orthogonal matrix such that $\bP^\top \bP = \bI_{2K}$.  Then $(\bTheta^{\star}, \bW^{\star})$ is a minimizer of \eqref{eq: NN simplified model}.  Moreover, if all  $\bX^{\star}$'s satisfy $\frac{1}{K}\sum_{k=1}^K \bX^{\star}(k,k) = \Etheta$, then all the solutions of   \eqref{eq: NN simplified model} are in the form of \eqref{eq:general solution}.
\end{lemma}

This lemma in effect says that the relaxation does \textit{not} lead to any loss of information when we study the Layer-Peeled Model through a convex program, thereby offering a computationally efficient tool for gaining insights into the terminal phase of training deep neural networks on imbalanced datasets. An appealing feature is that the size of the program (\ref{eq:convex sdp problem}) is independent of the number of training examples. Besides, this lemma predicts that even in the imbalanced case the last-layer features collapse to their class means under mild conditions. Therefore, Property (\hyperlink{(NC1)}{NC1}) is satisfied (see more discussion about the condition in SI Appendix). 


The assumption of the convexity of $\cL$ in the first argument is satisfied by a large class of loss functions. The condition that the first $K$ diagonal elements of any $\bX^{\star}$ make the associated constraint saturated is also not restrictive. For example, we prove in SI Appendix that this condition is satisfied for the cross-entropy loss. We also remark that  \eqref{eq:convex sdp problem} is not the unique convex relaxation. An alternative is to relax \eqref{eq: NN simplified model} via a nuclear norm-constrained convex program~\cite{bach2008convex, haeffele2019structured} (see more details in SI Appendix).

\subsection{Minority Collapse}
\label{subsec:minority-collapse}

With the technique of convex relaxation in place, now we numerically solve the Layer-Peeled Model on imbalanced datasets, with the goal of identifying possible nontrivial patterns. As a worthwhile starting point, we consider a dataset that has $K_A$ majority classes each containing $\barN_A$ training examples and $K_B$ minority classes each containing $\barN_B$ training examples. That is, assume $n_1 = n_2 = \dots = n_{K_A} = \barN_A$ and $n_{K_A+1} = n_{K_A+2}= \dots = n_{K} = \barN_B$. For convenience, call $R := \barN_A/\barN_B > 1$ the imbalance ratio. Note that the case $R = 1$ reduces to the balanced setting.


\begin{figure}[!htp]
		\centering
        \hspace{0.01in}
		\subfigure[$\Ew=1$, $\Etheta=5$]{
			\centering
			\includegraphics[scale=0.21]{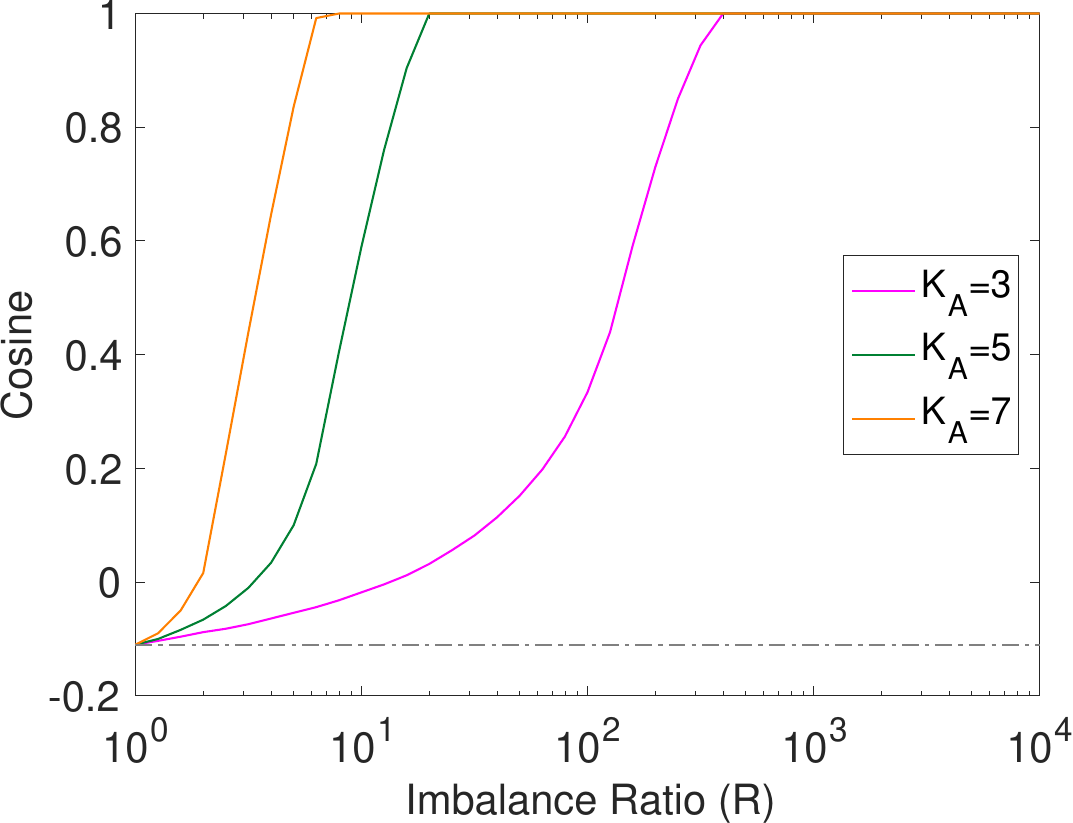}
			\label{fig:simulation-I}}
        \hspace{0.01in}
        \subfigure[$\Ew=1$, $\Etheta=10$]{
			\centering
			\includegraphics[scale=0.21]{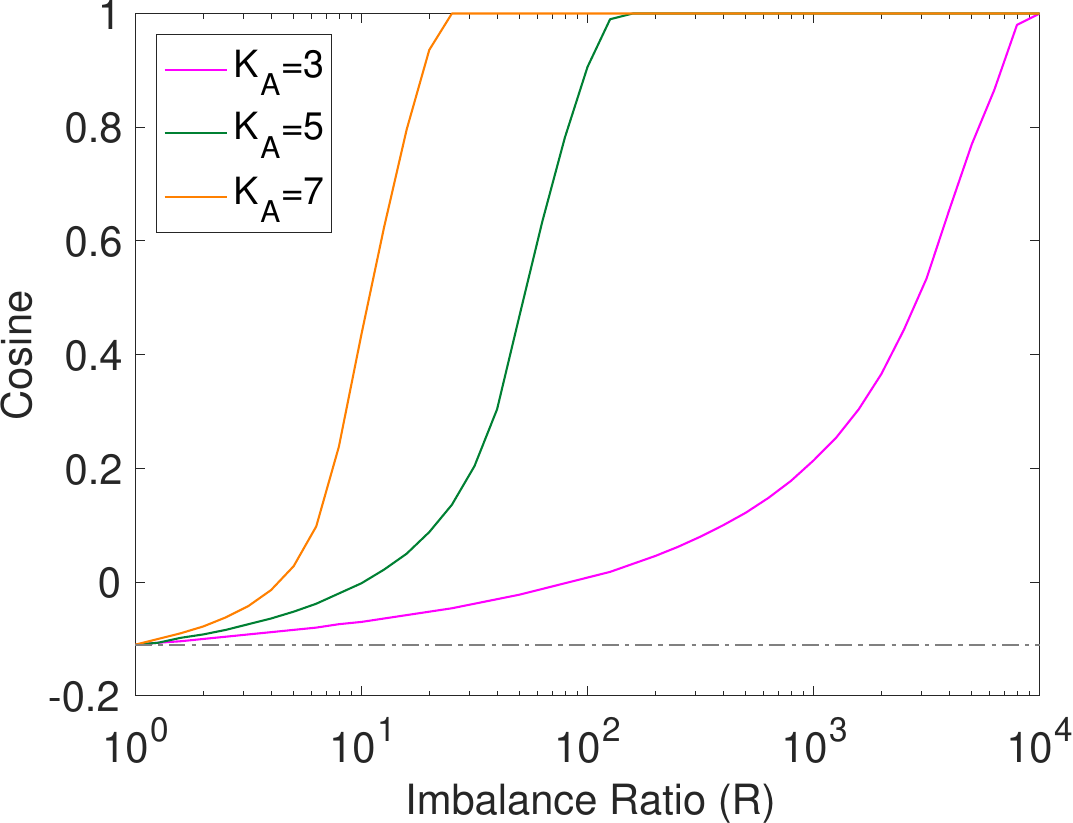}
			\label{fig:simulation-II}}
		\caption{
		 The average cosine of the angles between any pair of the minority classifier solved from the Layer-Peeled Model. The average cosine reaches $1$ once $R$ is above some threshold. The total number of classes $K_A + K_B$ is fixed to $10$. The gray dash-dotted line indicates the value of $-\frac{1}{K-1}$, which is given by \eqref{eq:cos}. The between-majority-class angles can still be large even when Minority Collapse emerges. Notably, our simulation suggests that the minority classifiers exhibit an equiangular frame and so do the majority classifiers.  
		}	
\label{fig:LPM-simulations}
\end{figure}


An important question is to understand how the $K_B$ last-layer minority classifiers behave as the imbalance ratio $R$ increases, as this is directly related to the model performance on the minority classes. To address this question, we show that the average cosine of the angles between any pair of the $K_B$ minority classifiers in Figure~\ref{fig:LPM-simulations} by solving the simple convex program (\ref{eq:convex sdp problem}). This figure reveals a two-phase behavior of the minority classifiers $\bw^\star_{K_A+1}, \bw^\star_{K_A+2}, \ldots, \bw^\star_K$ as $R$ increases:
\begin{itemize}
    \item[(1)]  When $R < R_0$ for some $R_0 > 0$, the average between-minority-class angle becomes smaller as $R$ increases.
    \item[(2)] Once $R \geq R_0$, the average between-minority-class angle become zero and, in addition, the minority classifiers have about the same length. This implies that all the minority classifiers collapse to a single vector.
\end{itemize}
Above, the phase transition point $R_0$ depends on the class sizes $K_A, K_B$ and the thresholds $\Etheta, \Ew$. This value becomes smaller when $E_W, E_H$, or the numer of minority classes $K_B$ is smaller while fixing the other parameters (see more numerical examples in Figure S2 in SI Appendix).

We refer to the phenomenon that appears in the second phase as Minority Collapse. While it can be expected that the minority classifiers become closer to each other as the level of imbalance increases, surprisingly, these classifiers become completely indistinguishable once $R$ hits a \textit{finite} value. Once Minority Collapse takes place, the neural network would predict equal probabilities for all the minority classes regardless of the input. As such, its predictive ability is by no means better than a coin toss when conditioned on the minority classes. This situation would only get worse in the presence of adversarial perturbations. This phenomenon is especially detrimental when the minority classes are more frequent in the application domains than in the training data. Even outside the regime of Minority Collapse, the classification might still be unreliable if the imbalance ratio is large as the softmax predictions for the minority classes can be close to each other.

To put the observations in Figure~\ref{fig:LPM-simulations} on a firm footing, we prove in the theorem below that Minority Collapse indeed emerges in the Layer-Peeled Model as $R$ tends to infinity.
\begin{theorem}\label{theo:imbalance limit} 
Assume $p\geq K$ and $\barN_A/\barN_B \to \infty$, and fix $K_A$ and $K_B$. Let $\left(\bTheta^\star, \bW^\star \right)$ be any global minimizer of the Layer-Peeled Model \eqref{eq: NN simplified model} with the cross-entropy loss. As $R \equiv \barN_A/\barN_B \to \infty$, we have
$$  \lim \bw^\star_{k} -\bw^\star_{\kk} =   \bm{0}_p,~\text{ for all } K_A < k < \kk \le K. $$


\end{theorem}

To intuitively see why Minority Collapse occurs, first note that the majority classes become the predominant part of the risk function as the level of imbalance increases. The minimization of the objective, therefore, pays too much emphasis on the majority classifiers, encouraging the between-majority-class angles to grow and meanwhile shrinking the between-minority-class angles to zero. As an aside, an interesting question for future work is to prove that $\bw^\star_{k}$ and $\bw^\star_{\kk}$ are exactly equal for sufficiently large $R$.


\subsection{Experiments}
\label{subsec:experiments}


At the moment, Minority Collapse is merely a prediction of the Layer-Peeled Model. An immediate question thus is: does this phenomenon really occur in real-world neural networks? At first glance, it does not necessarily have to be the case since the Layer-Peeled Model is a dramatic simplification of deep neural networks.


To address this question, we resort to computational experiments.\footnote{Our code is publicly available at \url{https://github.com/HornHehhf/LPM}.} Explicitly, we consider training two network architectures, VGG and ResNet \citep{he2016deep}, on the FashionMNIST \citep{xiao2017fashion} and CIFAR10 datasets, and in particular, replace the dropout layers in VGG with batch normalization \citep{ioffe2015batch}. As both datasets have 10 classes, we use three combinations of $(K_A, K_B) = (3, 7), (5, 5), (7, 3)$ to split the data into majority classes and minority classes. In the case of FashionMNIST (CIFAR10), we let the $K_A$ majority classes each contain all the $\barN_A = 6000$ ($\barN_A = 5000$) training examples from the corresponding class of FashionMNIST (CIFAR10), and the $K_B$ minority classes each have $\barN_B = 6000/R$ ($\barN_B = 5000/R$) examples randomly sampled from the corresponding class. The rest experiment setup is basically the same as \cite{papyan2020prevalence}. In detail, we
use the cross-entropy loss and stochastic gradient descent with momentum $0.9$ and weight decay $\lambda = 5 \times 10^{-4}$. The networks are trained for $350$ epochs with a batch size of $128$. The initial learning is annealed by a factor of $10$ at $1/3$ and $2/3$ of the $350$ epochs. The only difference from \cite{papyan2020prevalence} is that we simply set the learning rate to $0.1$ instead of sweeping over $25$ learning rates between $0.0001$ and $0.25$. This is because the test performance of our trained models is already comparable with their best reported test accuracy. Detailed training and test performance is displayed in Tables S1 and S2 in SI Appendix.



The results of the experiments above are displayed in Figure~\ref{fig:collapse}. This figure clearly indicates that the angles between the minority classifiers collapse to zero as soon as $R$ is large enough. Moreover, the numerical examination in Table~\ref{table:norm} shows that the norm of the classifier is constant across the minority classes. Taken together, these two pieces clearly give evidence for the emergence of Minority Collapse in these neural networks, thereby further demonstrating the effectiveness of our Layer-Peeled Model. Besides, Figure~\ref{fig:collapse} also shows that the issue of Minority Collapse is compounded when there are more majority classes, which is consistent with Figure~\ref{fig:LPM-simulations}.

\begin{figure}[!htp]
		\centering
        \hspace{0.01in}
		\subfigure[VGG11 on FashionMNIST]{
			\centering
			\includegraphics[scale=0.21]{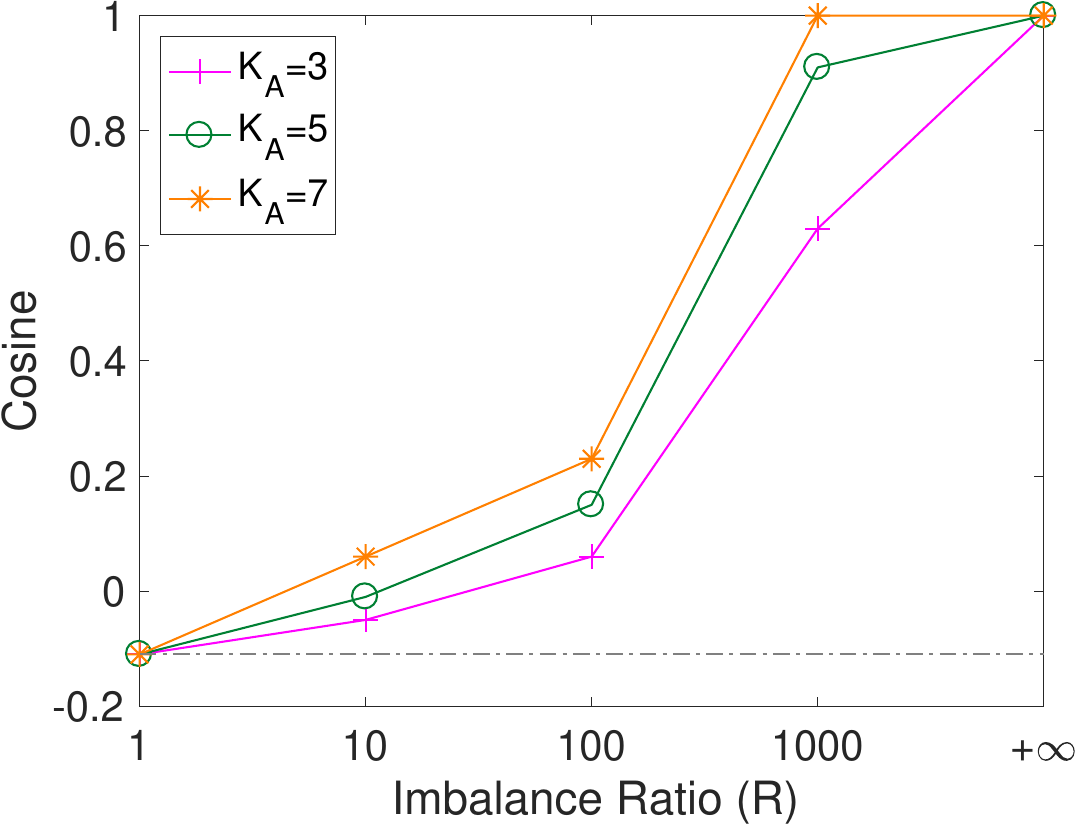}
			\label{fig:vgg-fashionmnist}} 
        \hspace{0.01in}
        	\subfigure[VGG13 on CIFAR10]{
			\centering
			\includegraphics[scale=0.21]{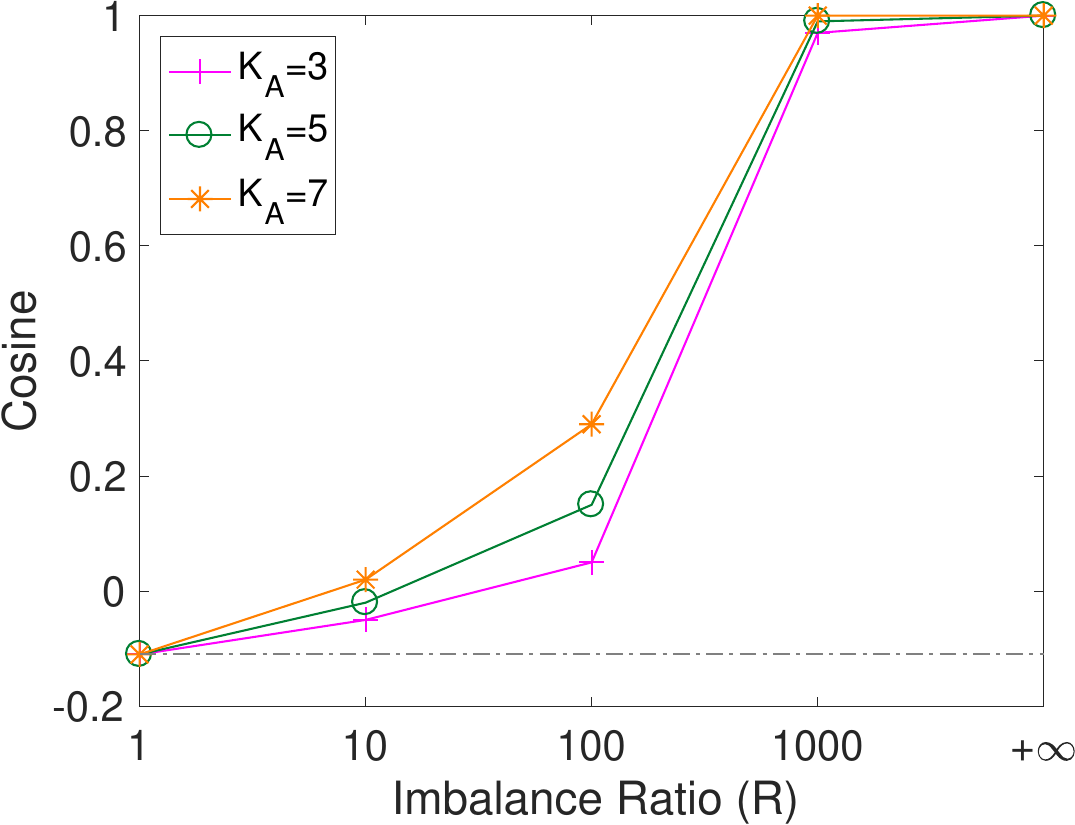}
			\label{fig:vgg-cifar10}}
        \hspace{0.01in}
     	\subfigure[ResNet18 on FashionMNIST]{
			\centering
			\includegraphics[scale=0.21]{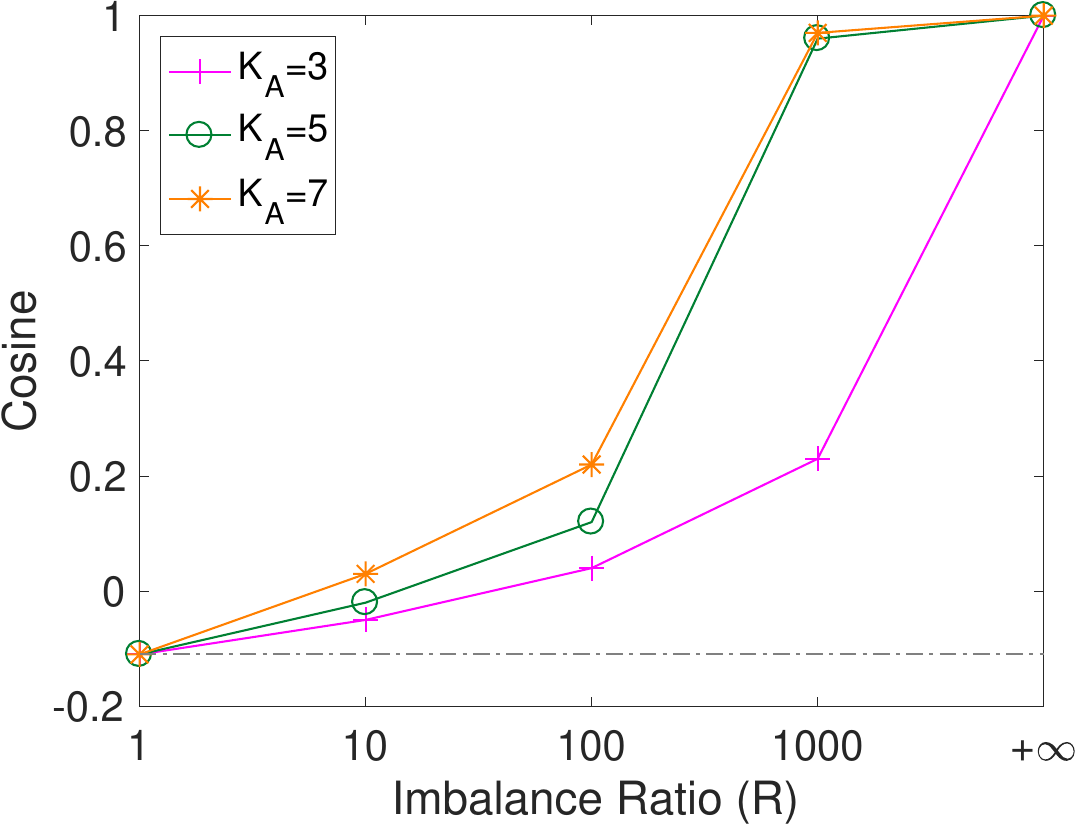}
			\label{fig:resnet-fashionmnist}}
		\subfigure[ResNet18 on CIFAR10]{
			\centering
			\includegraphics[scale=0.21]{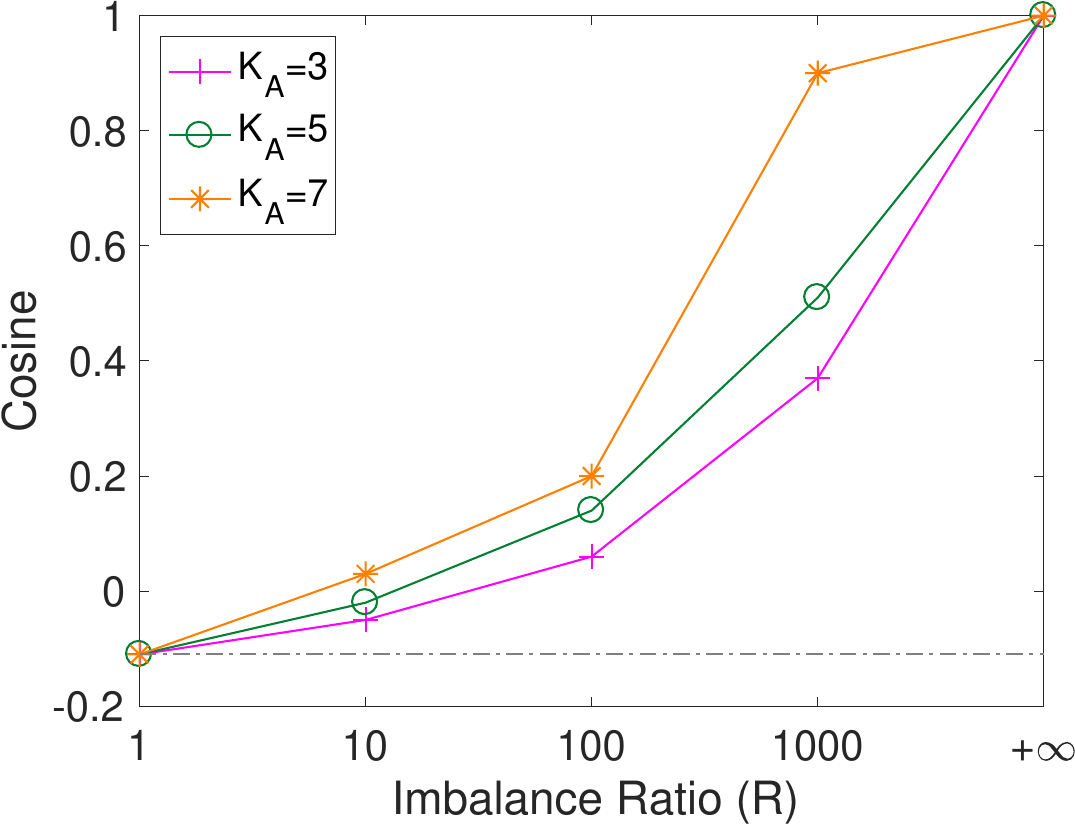}
			\label{fig:resnet-cifar10}}
		\caption{Occurrence of Minority Collapse in deep neural networks. Each curve denotes the average between-minority-class cosine. We fix $K_A + K_B = 10$. In particular, Figure~\ref{fig:collapse}(b) shares the same setting with Figure~\ref{fig:simulation-weight-decay} in Section~\ref{sec:intro}, where the LPM-based predictions are given by $(E_W, E_{\Theta})$ such that the two constraints in the Layer-Peeled Model become active for the weights of the trained networks. For ResNet 18, Minority Collapse also occurs as long as $R$ is sufficiently large. Specifically, the average cosine would hit 1 for $K_A = 7$ when $R = 5000$ on CIFAR10, and when $R = 3000$ on FashionMNIST.
}		
\label{fig:collapse}
\end{figure}

\begin{table*}[!htp]
\centering
\scalebox{0.735}{
\begin{tabular}{c||c|c|c||c|c|c||c|c|c||c|c|c}
\hline
Dataset &\multicolumn{6}{c||}{FashionMNIST} & \multicolumn{6}{c}{CIFAR10} \\ \hline
Network architecture & \multicolumn{3}{c||}{VGG11} & \multicolumn{3}{c||}{ResNet18} & \multicolumn{3}{c||}{VGG13} & \multicolumn{3}{c}{ResNet18} \\ \hline
No.~of majority classes & $K_A=3$ & $K_A=5$ & $K_A=7$ & $K_A=3$ & $K_A=5$ & $K_A=7$  & $K_A=3$ & $K_A=5$ & $K_A=7$ & $K_A=3$ & $K_A=5$ & $K_A=7$ \\ \hline
Norm variation & $2.7 \times 10^{-5}$ & $4.4 \times 10^{-8}$ & $6.0 \times 10^{-8}$ & $1.4 \times 10^{-5}$ & $5.0 \time 10^{-8}$ & $6.3 \times 10^{-8}$ & $1.4 \times 10^{-4}$ & $9.0 \times 10^{-7}$ & $5.2 \times 10^{-8}$ & $5.4 \times 10^{-5}$ & $3.5 \times 10^{-7}$ & $5.4 \times 10^{-8}$ \\ \hline 
\end{tabular}
}
\caption{Variability of the lengths of the minority classifiers when $R = \infty$.  Each number in the row of ``norm variation'' is $\mathrm{Std}(\|\bw_B^\star\|) / \mathrm{Avg}(\|\bw_B^\star\|)$, where $\mathrm{Std}(\|\bw_B^\star\|)$ denotes the standard deviation of the lengths of the $K_B$ classifiers and the denominator denotes the average. The results indicate that the classifiers of the minority classes have almost the same length.}
\label{table:norm}
\end{table*}


\begin{figure}[!htp]
		\centering
        \hspace{0.01in}
		\subfigure[VGG11 on FashionMNIST]{
			\centering
			\includegraphics[scale=0.21]{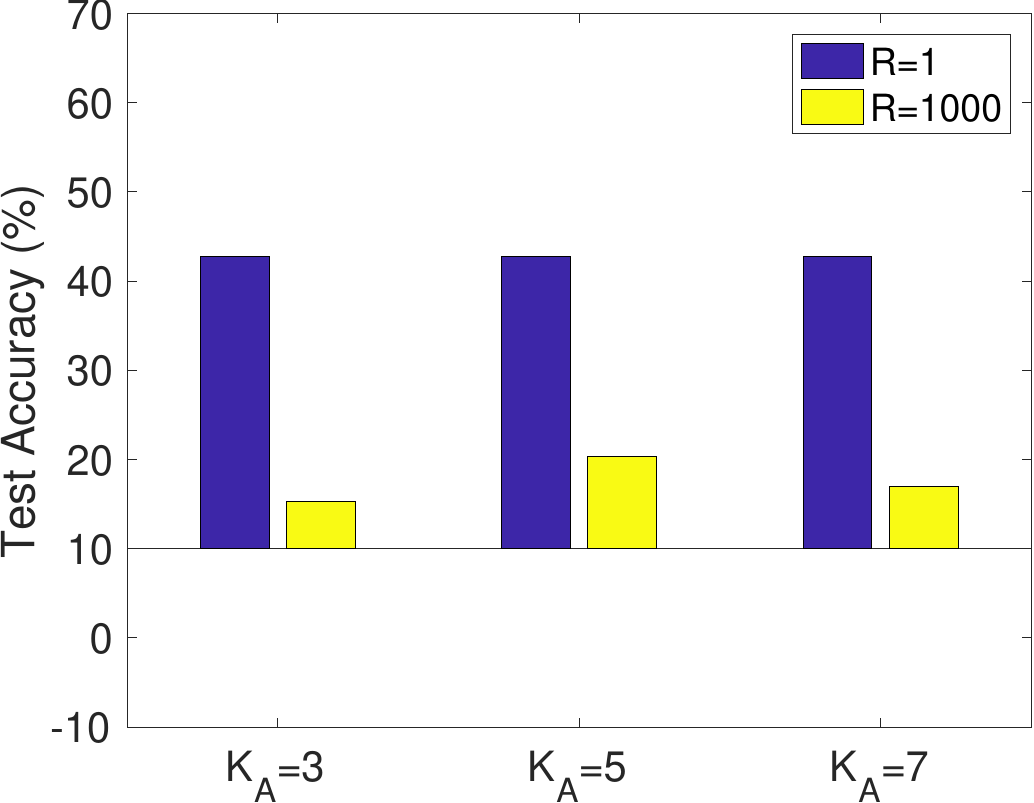}
			\label{fig:vgg-fashionmnist-comparison}} 
        \hspace{0.01in}
        	\subfigure[VGG13 on CIFAR10]{
			\centering
			\includegraphics[scale=0.21]{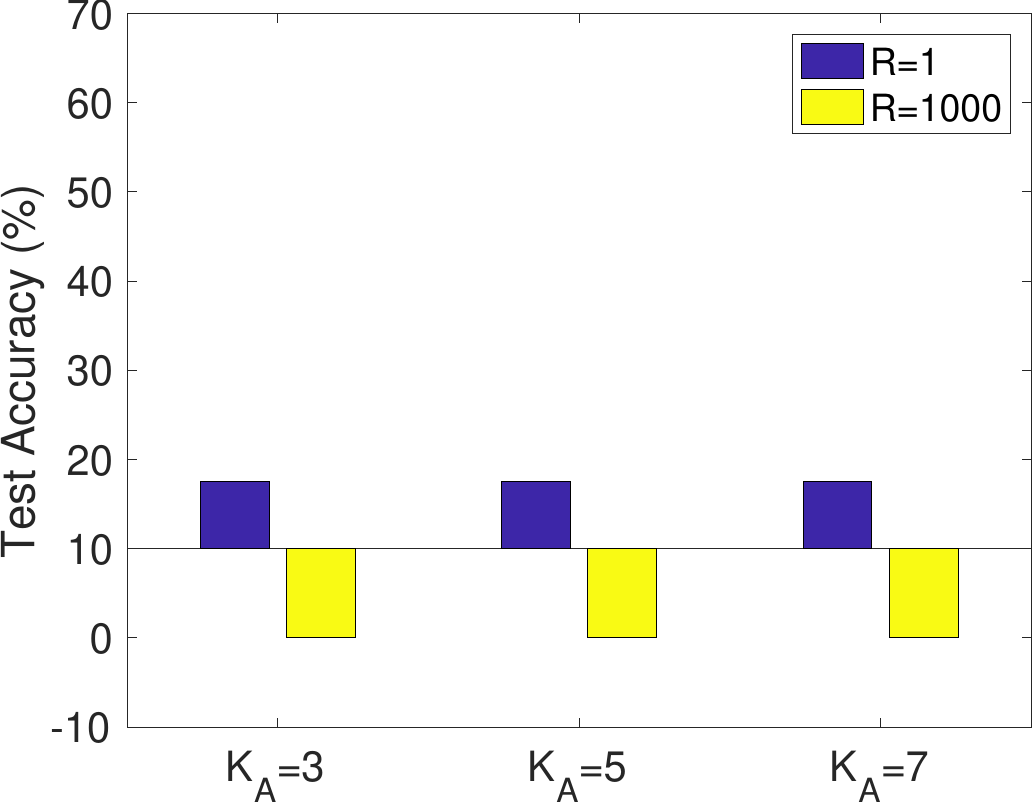}
			\label{fig:vgg-cifar10-comparison}} 
        \hspace{0.01in}
     	\subfigure[ResNet18 on FashionMNIST]{
			\centering
			\includegraphics[scale=0.21]{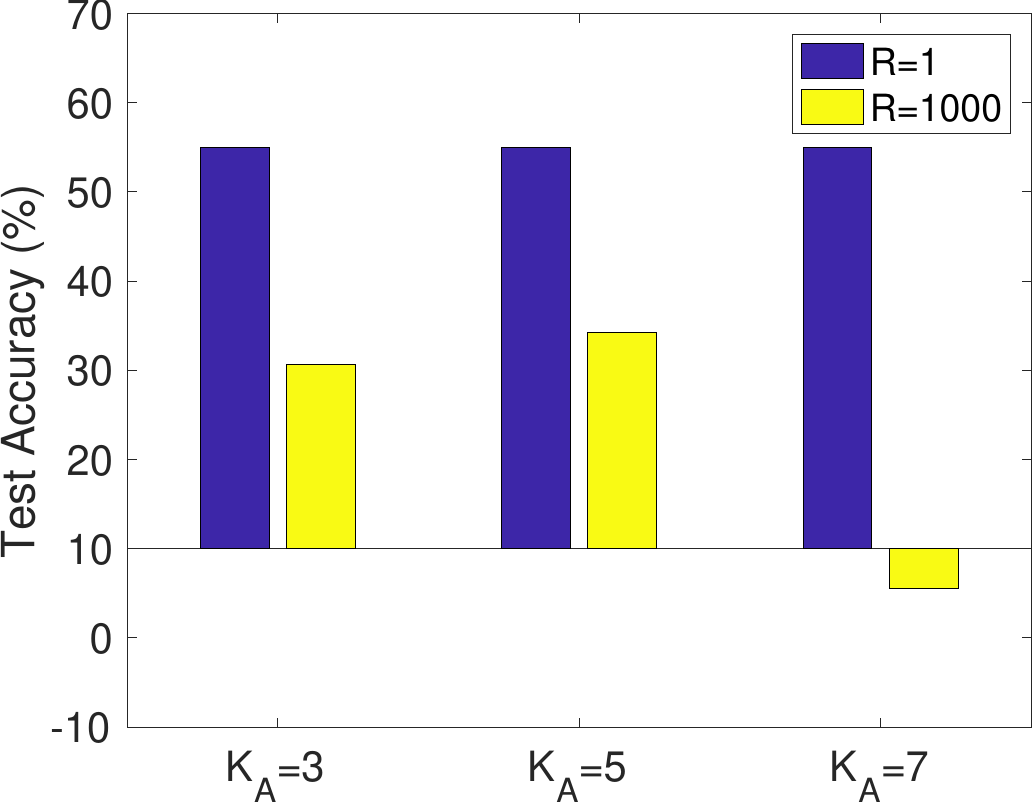}
			\label{fig:resnet-fashionmnist-comparison}} 
		\subfigure[ResNet18 on CIFAR10]{
			\centering
			\includegraphics[scale=0.21]{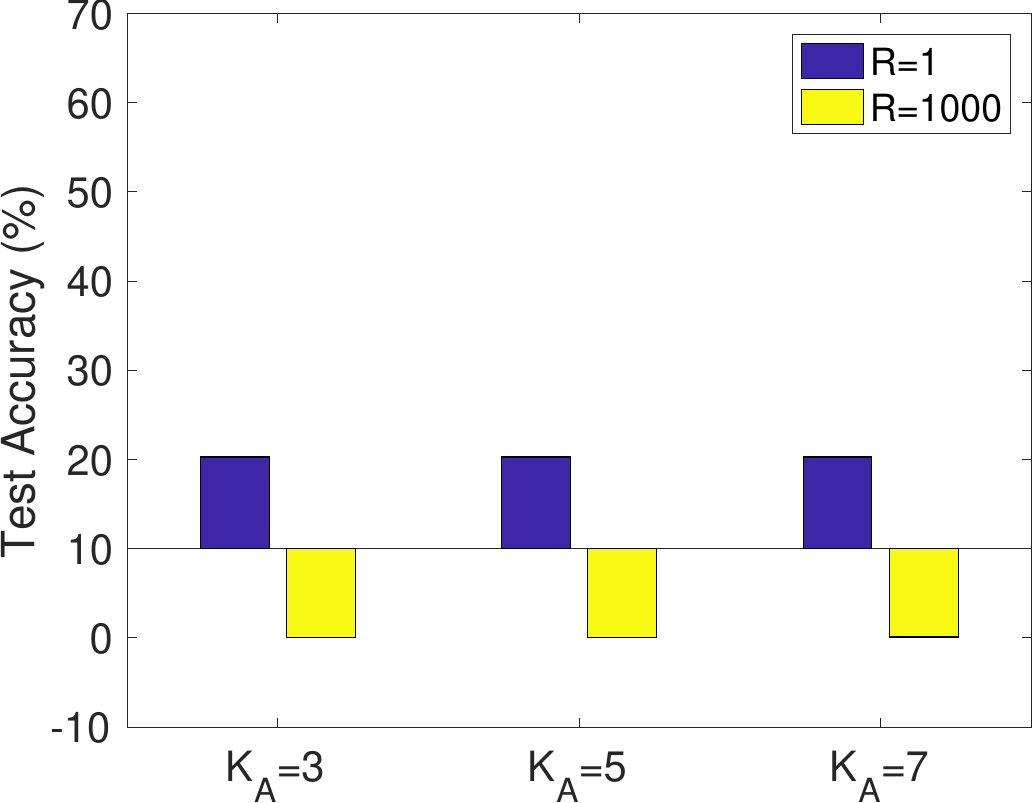}
			\label{fig:resnet-cifar10-comparison}}
		\caption{Comparison of the test accuracy on the minority classes between $R=1$ and $R=1000$. We fix $K_A + K_B = 10$ and use $n_B = 6$ ($n_B = 5$) training examples from each minority class and $n_A = 6R$ ($n_A = 5R$) training examples from each majority class in FashionMNIST (CIFAR10). Note that when $R=1000$, the test accuracy on the minority classes can be lower than $10\%$ because the trained neural networks misclassify many examples in the minority classes as some majority classes. }		
\label{fig:comparison}
\end{figure}

Next, in order to get a handle on how Minority Collapse impacts the test accuracy, we plot the results of another numerical study in Figure~\ref{fig:comparison}. The setting is the same as Figure~\ref{fig:collapse}, except that now we randomly sample $6$ or $5$ examples per class for the minority classes depending on whether the dataset is FashionMNIST or CIFAR10. The results show that the performance of the trained model deteriorates in the test data when the imbalance ratio $R = 1000$, when Minority Collapse has occurred or is about to occur. This is by no means intuitive a priori as the test performance is only restricted to the minority classes and a large value of $R$ only leads to more training data in the majority classes without affecting the minority classes at all.


It is worthwhile to mention that the emergence of Minority Collapse would prevent the model from achieving zero training error. This is because its prediction is uniform over the minority classes and, therefore, the ``argmax'' rule does not give the correct label for a training example from a minority class. As such, the occurrence of Minority Collapse is a departure from the terminal phase of deep learning training. While this fact seems to contradict conventional wisdom on the approximation power of deep learning, it is important to note that the constraints in the Layer-Peeled Model or, equivalently, weight decay in neural networks limits the expressive power of deep learning models. Besides, it is equally important to recognize that the training error, which mostly occurs in the minority classes, is actually very small when Minority Collapse emerges since the minority examples only account for a small portion of the entire training set. In this spirit, the aforementioned departure is not as significant as it appears at first glance since the training error is generally, if not always, not exactly zero
(see, e.g., \cite{papyan2020prevalence}). From an optimization point of view, a careful examination indicates that Minority Collapse can be attributed to the two constraints in the Layer-Peeled Model or the $\ell_2$ regularization in \eqref{eq:dl_opt}. For example, Figure~\ref{fig:simulation-weight-decay} shows that Minority Collapse occurs earlier with a larger value of $\lambda$. However, this issue does not disappear by simply setting a small penalty coefficient $\lambda$ as the imbalance ratio can be arbitrarily large.

%% file: 05-upsampling.tex
\section{How to Mitigate Minority Collapse?}\label{sec:oversampling}


In this section, we further exploit the use of the Layer-Peeled Model in an attempt to lessen the detrimental effect of Minority Collapse. Instead of aiming to develop a full set of methodologies to overcome this issue, which is beyond the scope of the paper, our aim is to evaluate some simple techniques used for imbalanced datasets.

Among many approaches to handling class imbalance in deep learning (see the review \cite{johnson2019survey}), perhaps the most popular one is to oversample training examples from the minority classes~\cite{buda2018systematic, shu2019meta, cui2019class,NEURIPS2019_621461af}. In its simplest form, this sampling scheme retains all majority training examples while duplicating each training example from the minority classes for $w_r$ times, where the oversampling rate $w_r$ is a positive integer. Oversampling in effect transforms the original problem to the minimization of a new optimization problem by replacing the risk term in  \eqref{eq:dl_opt} with
\begin{equation}\label{eq: loss:adjust weight}
\begin{aligned}
&\frac{1}{\barN_A K_A+w_r\barN_BK_B}\Bigg[ \sum_{k=1}^{K_A}\sum_{i=1}^{\barN_A} \cL( \bff(\bx_{k,i}; \wf), \by_k  )\\
&\quad\quad\quad\quad\quad\quad\quad+  w_r\!\!\!\sum_{k=K_A+1}^{K}\!\sum_{i=1}^{\barN_B} \cL( \bff(\bx_{k,i}; \wf), \by_k  )\Bigg]
\end{aligned}
\end{equation}
while keeping the penalty term $\frac{\lambda}{2} \|\wf\|^2$. Note that oversampling is closely related to weight adjusting (see more discussion in SI Appendix).

A close look at \eqref{eq: loss:adjust weight} suggests that the neural network obtained by minimizing this new program might behave as if it were trained on a (larger) dataset with $n_A$ and $w_r n_B$ examples in each majority class and minority class, respectively. To formalize this intuition, as earlier, we start by considering the Layer-Peeled Model in the case of oversampling:
\begin{align}\label{eq: imbalance weight}
  \min_{\bTheta, \bW} ~&  \frac{1}{N'}\!\!\left[ \sum_{k=1}^{K_A}\sum_{i=1}^{\barN_A} \cL( \bW\btheta_{k,i}, \by_k  ) +  w_r\!\!\!\!\sum_{k=K_A+1}^{K}\!\!\sum_{i=1}^{\barN_B} \cL( \bW\btheta_{k,i}, \by_k  )\right]\notag \\
\mathrm{s.t.}~& \frac{1}{K}\sum_{k=1}^K \left\|\bw_k \right\|^2 \leq E_{W},\\
&\!\!\!\!\!\!\!\frac{1}{K}\sum_{k=1}^{K_A} \frac{1}{n_A}\sum_{i=1}^{n_A}\left\|\btheta_{k,i} \right\|^2 + \frac{1}{K}\sum_{k=K_A+1}^K \frac{1}{n_B}\sum_{i=1}^{n_B}\left\|\btheta_{k,i} \right\|^2\leq E_{\Theta}\notag,
\end{align}
where $N' :=\barN_A K_A+w_r\barN_BK_B $.

\begin{figure}[!htp]
		\centering
        \hspace{0.01in}
		\subfigure[VGG11 on FashionMNIST]{
			\centering
			\includegraphics[scale=0.21]{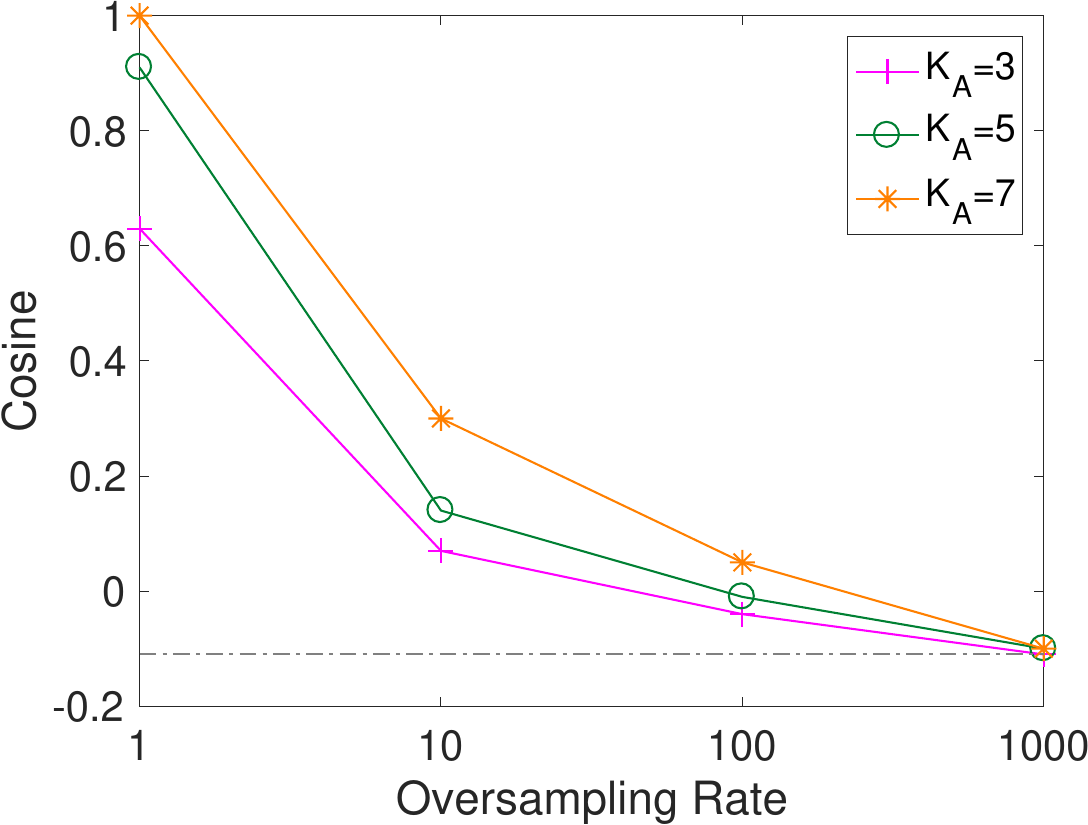}
			\label{fig:vgg-fashionmnist-weights}}
        \hspace{0.01in} 
        	\subfigure[VGG13 on CIFAR10]{
			\centering
			\includegraphics[scale=0.21]{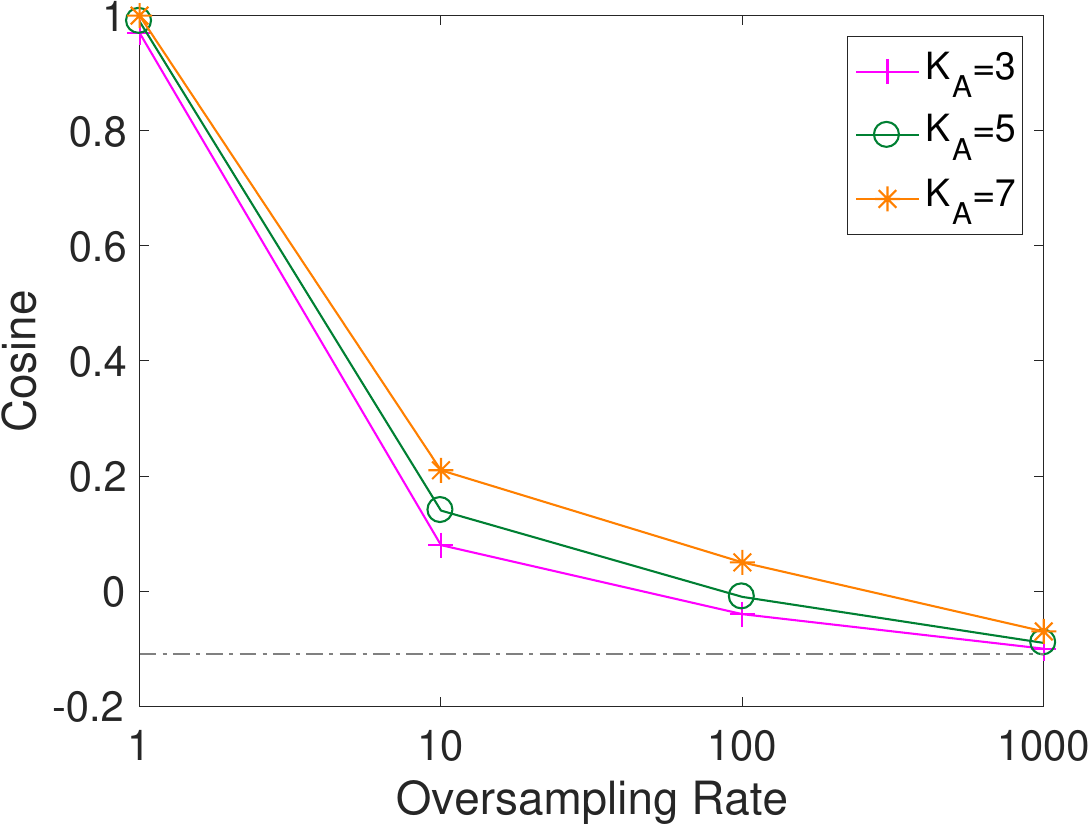}
			\label{fig:vgg-cifar10-weights}}
        \hspace{0.01in}
     	\subfigure[ResNet18 on FashionMNIST]{
			\centering
			\includegraphics[scale=0.21]{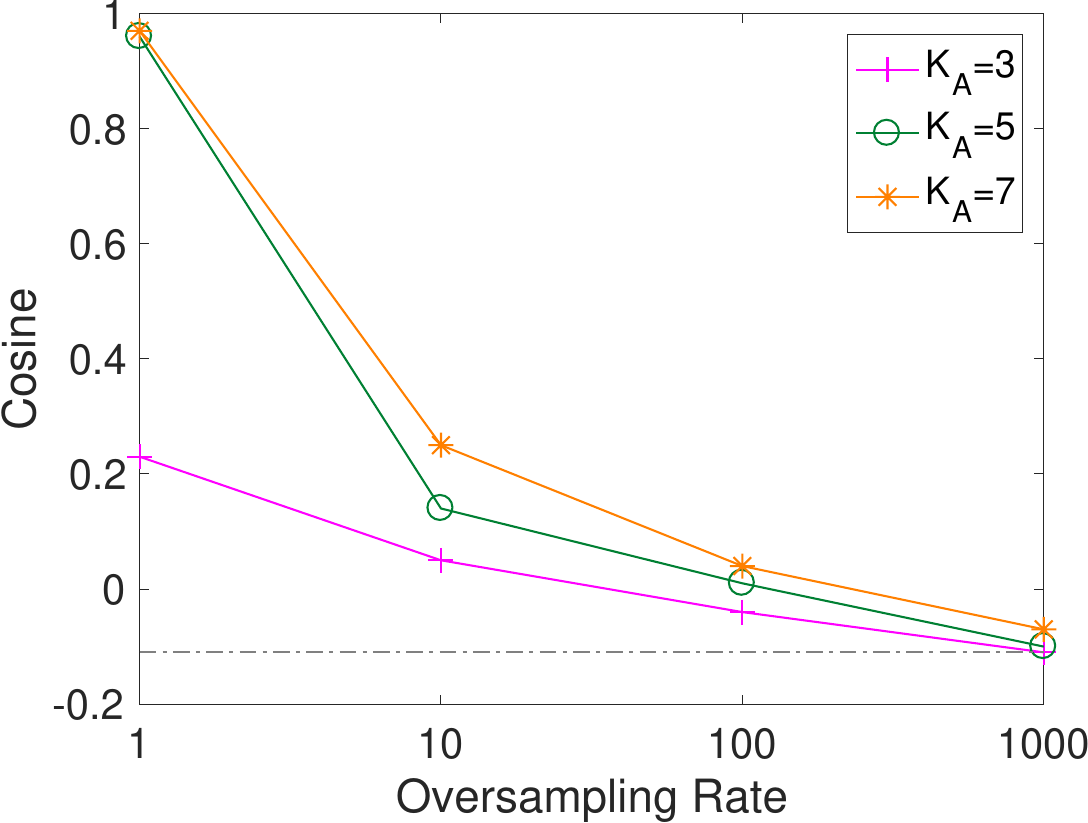}
			\label{fig:resnet-fashionmnist-weights}} 
		\subfigure[ResNet18 on CIFAR10]{
			\centering
			\includegraphics[scale=0.21]{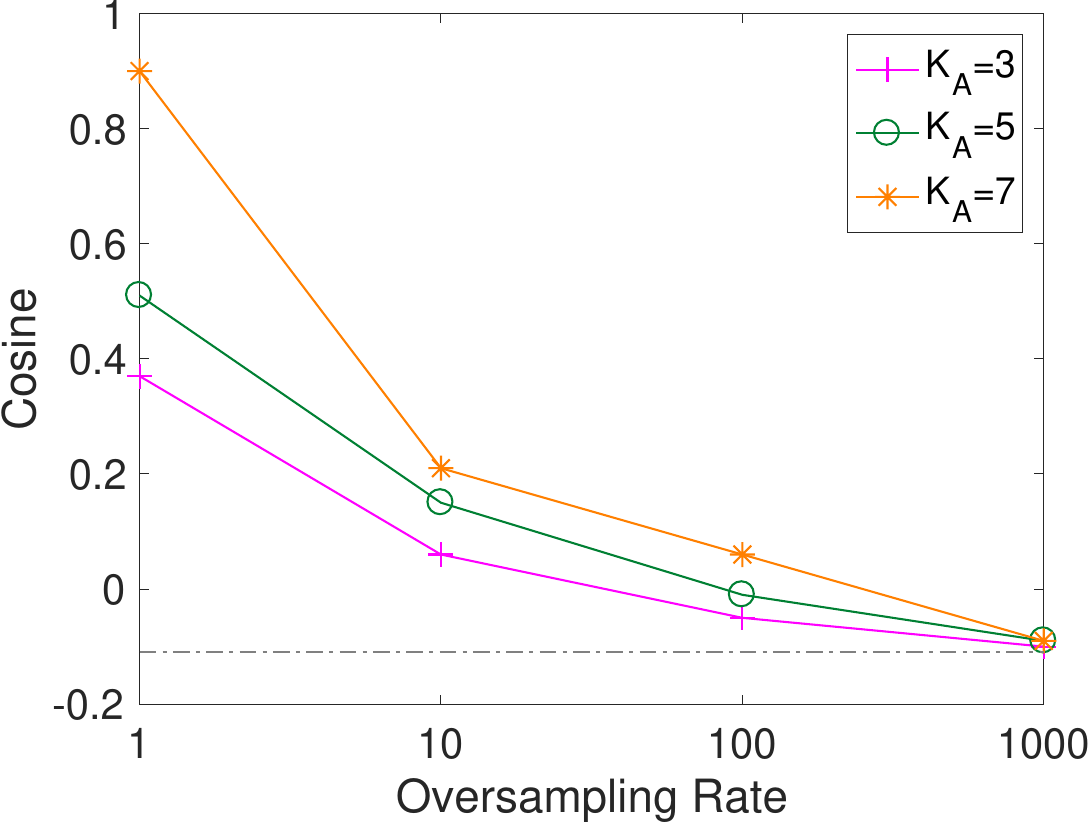}
			\label{fig:resnet-cifar10-weights}}
		\caption{Effect of oversampling when the imbalance ratio is $R=1000$. Each plot shows the average cosine of the between-minority-class angles. The results indicate that increasing the oversampling rate would enlarge the between-minority-class angles. }		
\label{fig:weights}
\end{figure}

The following result confirms our intuition that oversampling indeed boosts the size of the minority classes for the Layer-Peeled Model.
\begin{proposition}\label{proposition: imbalance weight}
Assume $p\geq 2K$ and the loss function $\cL$ is convex in the first argument. Let $\bX^{\star}$ be any minimizer of the convex program (\ref{eq:convex sdp problem}) with $n_1 = n_2 = \dots = n_{K_A} = \barN_A$ and  $n_{K_A+1} = n_{K_A+2}= \dots = n_{K} = w_r\barN_B$.
 Define $\left(\bTheta^{\star}, \bW^{\star}\right)$ as
\begin{equation}\label{eq: imblance solution}
    \begin{aligned}
     &\left[\btheta_{1}^{\star},\btheta_{2}^{\star},\dots,\btheta_{K}^{\star},  (\bW^{\star})^\top\right] =\bP (\bX^{\star})^{1/2} ,\\
&~~\btheta_{k,i}^{\star} = \btheta_{k}^{\star},~\text{ for all } 1 \le i \le n_A, 1 \le k \le K_A,\\
&~~\btheta_{k,i}^{\star} = \btheta_{k}^{\star}, ~\text{ for all } 1 \le i \le n_B, K_A < k \le K,
    \end{aligned}
\end{equation}
where $\bP\in \RR^{p\times 2K}$ is any partial orthogonal matrix such that $\bP^\top \bP = \bI_{2K}$.  Then $(\bTheta^{\star}, \bW^{\star})$ is a global minimizer of the oversampling-adjusted Layer-Peeled Model~\eqref{eq: imbalance weight}. Moreover, if all $\bX^{\star}$'s satisfy $\frac{1}{K}\sum_{k=1}^K \bX^{\star}(k,k) = \Etheta$, then all the solutions of    \eqref{eq: imbalance weight} are in the form of \eqref{eq: imblance solution}.
\end{proposition}
Together with Lemma \ref{theo:to convex}, Proposition \ref{proposition: imbalance weight} shows that the number of training examples in each minority class is now in effect $w_r\barN_B$ instead of $\barN_B$ in the Layer-Peeled Model. In the special case $w_r = \barN_A/\barN_B \equiv R$, the results show that all the angles are equal between any given pair of the last-layer classifiers, no matter if they fall in the majority or minority classes.




We turn to Figure~\ref{fig:weights} for an illustration of the effects of oversampling on real-world deep learning models, using the same experimental setup as in Figure~\ref{fig:comparison}. From Figure~\ref{fig:weights}, we see that the angles between pairs of the minority classifiers become larger as the oversampling rate $w_r$ increases. Consequently, the issue of Minority Collapse becomes less detrimental in terms of training accuracy as $w_r$ increases. This again corroborates the predictive ability of the Layer-Peeled Model.

\begin{table}[!htp]
\centering
\scalebox{0.735}{
\begin{tabular}{c||c|c|c||c|c|c}
\hline
Network architecture & \multicolumn{3}{|c||}{VGG11} & \multicolumn{3}{|c}{ResNet18} \\ \hline \hline
No.~of majority classes & $K_A=3$ & $K_A=5$ & $K_A=7$ & $K_A=3$ & $K_A=5$ & $K_A=7$  \\ \hline \hline
Original (minority) & 15.29 & 20.30 & 17.00 & 30.66 & 34.26 & 5.53 \\ \hline
Oversampling (minority) & {\bf 41.13} & {\bf 57.22} &  {\bf 30.50} & {\bf 37.86} & {\bf 53.46} & {\bf 8.13} \\ \hline 
Improvement (minority) & 25.84& 36.92 & 13.50 & 7.20 & 19.20 & 2.60 \\ \hline \hline
Original (overall) & 40.10 & 57.61 & 69.09 & 50.88 & 64.89 & 66.13 \\ \hline
Oversampling (overall) & {\bf 58.25} & {\bf 76.17} & {\bf  73.37} & {\bf 55.91} & {\bf 74.56} & {\bf 67.10} \\ \hline
Improvement (overall) & 18.15& 18.56& 4.28 & 5.03 & 9.67 & 0.97 \\ \hline
\end{tabular}
}
\caption{
Test accuracy (\%) on FashionMNIST when $R = 1000$. For example, ``Original (minority)'' means that the test accuracy is evaluated only on the minority classes and oversampling is not used. When oversampling is used, we report the best test accuracy among four oversampling rates: $1$, $10$, $100$, and $1000$. The best test accuracy is never achieved at $w_r = 1000$, indicating that oversampling with a large $w_r$ would impair the test performance. 
}
\label{table:weighted-algorithm}
\end{table} 

Next, we refer to Table~\ref{table:weighted-algorithm} for effect on the test performance. The results clearly demonstrate the improvement in test accuracy using oversampling, with certain choices of the oversampling rate. The improvement is noticeable on both the minority classes and all classes.

Behind the results of Table~\ref{table:weighted-algorithm}, however, it reveals an issue when addressing Minority Collapse by oversampling. Specifically, this technique might lead to degradation of test performance using a very large oversampling rate $w_r$, which though can mitigate Minority Collapse. How can we efficiently select an oversampling rate for optimal test performance? More broadly, Minority Collapse does not seem likely to be fully resolved by sampling-based approaches alone, and the doors are widely open for future investigation.

%% file: 06-discussion.tex
\section{Discussion}
\label{sec:discuss}

In this paper, we have developed the Layer-Peeled Model as a simple yet effective modeling strategy toward understanding well-trained deep neural networks. The derivation of this model follows a top-down strategy by isolating the last layer from the remaining layers. Owing to the analytical and numerical tractability of the Layer-Peeled Model, we provide some explanation of a recently observed phenomenon called neural collapse in deep neural networks trained on balanced datasets~\cite{papyan2020prevalence}. Moving to imbalanced datasets, an analysis of this model suggests that the last-layer classifiers corresponding to the minority classes would collapse to a single vector once the imbalance level is above a certain threshold. This new phenomenon, which we refer to as Minority Collapse, occurs consistently in our computational experiments.

The efficacy of the Layer-Peeled Model in analyzing well-trained deep learning models implies that the ansatz \eqref{eq:ansatz}---a crucial step in the derivation of this model---is at least a useful approximation. Moreover, this ansatz can be further justified by the following result in an indirect manner, which, together with Theorem~\ref{theo: cross-entropy balance}, shows that the $\ell_2$ norm suggested by the ansatz happens to be the only choice among all the $\ell_q$ norms that is consistent with empirical observations. Its proof is given in SI Appendix.

\begin{proposition}\label{theo:counter example}
Assume $K\geq3$ and $p\geq K$.\footnote{See discussion in the case $K = 2$ in SI Appendix.} For any  $q\in (0,2)\cup (2, \infty)$, consider the optimization problem
\begin{equation}\label{eq:counterproblem}
    \begin{aligned}
\min_{\bW, \bH} \quad& \frac{1}{N} \sum_{k=1}^K\sum_{i=1}^{\barN} \cL( \bW\btheta_{k,i}, \by_k  )\\
\mathrm{s.t.}\quad& \frac{1}{K}\sum_{k=1}^K\left\|\bw_k \right\|^2 \leq \Ew,\\
&  \frac{1}{K}\sum_{k=1}^K\frac{1}{\barN}\sum_{i=1}^{\barN}\left\|\btheta_{k,i} \right\|^q_q \leq \Etheta,   
    \end{aligned}
\end{equation}
where $\cL$ is the cross-entropy loss. Then, any global minimizer of this program does not satisfy \eqref{eq:solution} for any positive numbers $C$ and $C'$. That is, neural collapse does not emerge in this model.

\end{proposition}

While the Layer-Peeled Model has demonstrated its noticeable effectiveness, it requires future investigation for consolidation and extension. First, an analysis of the gap between the Layer-Peeled Model and well-trained deep learning models would be a welcome advance. For example, how does the gap depend on the neural network architectures? How to take into account the sparsity of the last-layer features when using the ReLU activation function? From a different angle, a possible extension is to retain multiple layers following the top-down viewpoint. Explicitly, letting $1 \le m < L$ be the number of the top layers we wish to retain in the model, we can represent the prediction of the neural network as $\bff(\bx, \wf) = \bff(\bh(\bx; \bW_{1:(L-m)}), \bW_{(L-m+1):L})$ by letting $\bW_{1:(L-m)}$ and $\bW_{(L-m+1):L}$ be the first $L-m$ layers and the last $m$ layers, respectively. Consider the $m$-Layer-Peeled Model: 
\begin{equation}\nonumber
\begin{aligned}
\min_{\bW,\bTheta} \quad& \frac{1}{N} \sum_{k=1}^K \sum_{i=1}^{n_k} \cL( \bff(\bh_{k,i}, \bW_{(L-m+1):L}), \by_k )\\
\mathrm{s.t.}\quad &     \frac{1}{K} \|\bW_{(L-m+1):L}\|^2 \leq E_{W},\\
&\frac{1}{K}\sum_{k=1}^K \frac{1}{n_k}\sum_{i=1}^{n_k}\left\|\btheta_{k,i} \right\|^2 \leq E_{\Theta}.
\end{aligned}
\end{equation}
The two constraints might be modified to take into account the network architectures. An immediate question is whether this model with $m = 2$ is capable of capturing new patterns of deep learning training. 


From a practical standpoint, the Layer-Peeled Model together with its convex relaxation \eqref{eq:convex sdp problem} offers an analytical and computationally efficient technique to identify and mitigate bias induced by class imbalance. An interesting question is to extend Minority Collapse from the case of two-valued class sizes to general imbalanced datasets. Next, as suggested by our findings in Section~\ref{sec:oversampling}, how should we choose loss functions in order to mitigate Minority Collapse~\cite{NEURIPS2019_621461af}? Last, a possible use case of the Layer-Peeled Model is to design more efficient sampling schemes to take into account fairness considerations~\citep{buolamwini2018gender,zou2018ai,mehrabi2019survey}. 


Broadly speaking, insights can be gained not only from the Layer-Peeled Model but also from its modeling strategy. The details of empirical deep learning models, though formidable, can often be simplified by rendering a certain part of the network \textit{modular}. When the interest is about the top few layers, for example, this paper clearly demonstrates the benefits of taking a top-down strategy for modeling neural networks especially in consolidating our understanding of previous results and in discovering new patterns. Owing to its mathematical convenience, the Layer-Peeled Model shall open the door for future research extending these benefits.

%% file: 07-proofs.tex
For simplicity,   we define  $[m_1:m_2] := \{m_1, m_1+1, \dots,m_2 \}$ for $m_1,m_2\in\mathbb{N}$ with $m_1\leq m_2$ and $[m_2] := [1:m_2]$ for $m_2\geq1$.  

\subsection{Balanced Case}
\label{app:balanced}

\subsubsection{Proofs of Theorem \ref{theo: cross-entropy balance} and Proposition \ref{theo:counter example} }\label{sec:balance proof 1}

 Because there are multiplications of variables in the objective function,  \eqref{eq: NN simplified model} is nonconvex. Thus the KKT condition is not sufficient for  optimality. To prove Theorem \ref{theo: cross-entropy balance}, we  directly determine the global minimum  of  \eqref{eq: NN simplified model}. During this procedure,  one key  step  is to show that minimizing   \eqref{eq: NN simplified model} is equivalent to minimize a  symmetric quadratic function: $$\sum_{i=1}^{\barN}\left[ \left(\sum_{k=1}^K \btheta_{k,i}\right)^{\top}\left(\sum_{k=1}^K \bw_k\right)  -K\sum_{k=1}^{K}\btheta_{k,i}^\top\bw_k \right]$$ under  suitable conditions. The detail is shown below. 
\begin{proof}[Proof of Theorem \ref{theo: cross-entropy balance}]
By the concavity of $\log(\cdot)$,  for any $\bz\in \RR^K$, $k\in[K]$, constants $C_a,C_b>0$, letting $C_c = \frac{C_b}{(C_a+C_b)(K-1)}$,  we have
\begin{align}\label{eq:logsim}
 - \log\left(  \frac{\bz(k)}{\sum_{\kk=1}^{K}\bz(\kk) }   \right) =& - \log (\bz(k)) + \log \left( \sum_{\kk=1}^{K}\bz(\kk) \right)\notag\\
 =&   - \log(\bz(k))+ \log\left( \frac{C_a}{C_a+C_b} \left(\frac{(C_a+C_b)~\bz(k)}{C_a}\right) + C_c\sum_{\kk=1, \kk\neq k}^K  \frac{\bz(\kk) }{C_c}           \right).
\end{align}
Recognizing the equality
\[
\frac{C_a}{C_a+C_b} + \underbrace{C_c + \cdots + C_c}_{K-1} = \frac{C_a}{C_a+C_b} + (K-1) \frac{C_b}{(C_a+C_b)(K-1)} = 1
\]
and the concavity of $\log(\cdot)$, we see that the Jensen inequality gives
\begin{equation}\label{eq:log_jensen}
\log\left( \frac{C_a}{C_a+C_b} \left(\frac{(C_a+C_b)~\bz(k)}{C_a}\right) + C_c\sum_{\kk=1, \kk\neq k}^K  \frac{\bz(\kk) }{C_c} \right) \ge \frac{C_a}{C_a+C_b} \log\left( \frac{(C_a+C_b)~\bz(k)}{C_a}\right) + C_c\sum_{\kk=1, \kk\neq k}^K \log\left(  \frac{\bz(\kk) }{C_c} \right).
\end{equation}
Plugging this inequality into \eqref{eq:logsim}, we get
\begin{align*}
 - \log\left(  \frac{\bz(k)}{\sum_{\kk=1}^{K}\bz(\kk) }   \right) \ge&  - \log(\bz(k)) + \frac{C_a}{C_a+C_b} \log\left( \frac{(C_a+C_b)~\bz(k)}{C_a}\right) + C_c\sum_{\kk=1, \kk\neq k}^K \log\left(  \frac{\bz(\kk) }{C_c} \right)\\
=& - \frac{C_b}{C_a+C_b}\left[ \log(\bz(k)) - \frac{1}{K-1} \sum_{\kk=1, \kk\neq k}^K \log(\bz(\kk))  \right] +C_d,
\end{align*}
where the constant $C_d :=  \frac{C_a}{C_a+C_b}\log(\frac{C_a+C_b}{C_a})+ \frac{C_b}{C_a+C_b}\log(1/C_c)$. Note that in \eqref{eq:logsim},  $C_a$ and $C_b$ can be any positive numbers.  To prove Theorem \ref{theo: cross-entropy balance}, we set  $C_a:= \exp\left(\sqrt{\Etheta\Ew} \right)$ and $C_b:=  \exp\left(-\sqrt{\Etheta\Ew}/(K-1) \right)$, which shall lead to the tightest lower bound for the objective of \eqref{eq: NN simplified model}.  Applying \eqref{eq:logsim} to the objective, we have
\begin{align}\label{eq: lower cross}
& \frac{1}{N} \sum_{k=1}^K\sum_{i=1}^{\barN} \cL( \bW\btheta_{k,i}, \by_k  )\\
 \geq&\frac{C_b}{(C_a+C_b)N(K-1)}\sum_{i=1}^{\barN}\left[ \left(\sum_{k=1}^K \btheta_{k,i}\right)^{\top}\left(\sum_{k=1}^K \bw_k\right)  -K\sum_{k=1}^{K}\btheta_{k,i}^\top\bw_k \right]+ C_d. \notag
\end{align}
 Defining $\bartheta_i := \frac{1}{K}\sum_{k=1}^K \btheta_{k,i}$ for $i\in[\barN]$, it follows from the simple inequality $2ab \le a^2 + b^2$ that
 \begin{align}\label{eq:padd}
   &\sum_{i=1}^{\barN}\left[ \left(\sum_{k=1}^K \btheta_{k,i}\right)^{\top}\left(\sum_{k=1}^K \bw_k\right)  -K\sum_{k=1}^{K}\btheta_{k,i}^\top\bw_k \right]\notag\\
 =&K\sum_{i=1}^{\barN} \sum_{k=1}^K(\bartheta_i  - \btheta_{k,i})^\top \bw_k \notag\\
 \geq&  - \frac{K}{2}\sum_{k=1}^{K} \sum_{i=1}^{\barN}  \|\bartheta_i  - \btheta_{k,i} \|^2 /C_e-  \frac{C_eN}{2}\sum_{k=1}^{K}\| \bw_k\|^2,   
 \end{align}
where we pick $C_e := \sqrt{\Etheta/\Ew}$. The two terms in the right hand side of \eqref{eq:padd} can be bounded  via the constraints of   \eqref{eq: NN simplified model}. Specifically,  we have
\begin{equation}\label{eq:bounw}
    \frac{C_e N}{2}\sum_{k=1}^K\| \bw_k\|^2 \leq \frac{KN\sqrt{\Etheta\Ew}}{2},
\end{equation}
and 
\begin{align}\label{eq:boundtheta}
  \frac{K}{2}\sum_{k=1}^K\sum_{i=1}^{\barN}   \|\bartheta_i  - \btheta_{k,i} \|^2 /C_e 
  &\overset{a}=   \frac{K^2}{2C_e}\sum_{i=1}^{\barN}\left(    \frac{1}{K} \sum_{k=1}^K \|\btheta_{k,i}\|^2 - \|\bartheta_i\|^2   \right) \notag\\
  &\leq  \frac{K}{2C_e}  \sum_{k=1}^{K}\sum_{i=1}^{\barN} \|\btheta_{k,i}\|^2\leq \frac{KN\sqrt{\Etheta\Ew}}{2},
\end{align}
where $\overset{a}=$ uses the fact that $\E \|\ba-\E [\ba]  \|^2 =  \E \|\ba  \|^2  - \|\E[\ba]\|^2$. Thus plugging \eqref{eq:padd},  \eqref{eq:bounw}, and \eqref{eq:boundtheta}  into  \eqref{eq: lower cross}, we have
\begin{equation}\label{eq:lowerboundff}
    \frac{1}{N} \sum_{k=1}^K\sum_{i=1}^{\barN} \cL( \bW\btheta_{k,i}, \by_k  ) \geq -\frac{C_b}{C_a+C_b}\frac{K\sqrt{\Etheta\Ew}}{K-1} +C_d:=L_0.
\end{equation}
Now we check the conditions that reduce \eqref{eq:lowerboundff} to an equality.

By the strict concavity of $\log(\cdot)$, \eqref{eq:log_jensen} reduces to an equality only if $$\frac{(C_a+C_b)~\bz(k)}{C_a} = \frac{\bz(\kk) }{C_c}$$
for $\kk\neq k$. 
Therefore, \eqref{eq: lower cross} reduces to an equality only if 
$$ \frac{(C_a+C_b) \btheta_{k,i}^\top \bw_k}{C_a} =  \frac{ \btheta_{k,i}^\top \bw_{\kk} }{C_c}. $$
Recognizing $C_c = \frac{C_b}{(C_a+C_b)(K-1)}$ and  taking the logarithm of both sides of the above equation, we obtain 
\begin{equation}
\btheta_{k,i}\bw_k = \btheta_{k,i}\bw_{\kk} + \log\left(\frac{C_a(K-1)}{C_b}\right),  \notag
\end{equation}
for all $(k,i,\kk)\in \{(k,i,\kk): k\in[K], \kk\in[K], \kk\neq k, i\in[\barN]    \}
$.
  \eqref{eq:padd} becomes equality if and only if
\begin{equation}
  \bartheta_i  - \btheta_{k,i}  = - C_e \bw_k, \quad k\in[K], ~ i\in[\barN]. \notag
\end{equation} 
 \eqref{eq:bounw} and \eqref{eq:boundtheta} become equalities if and only if:
\begin{equation}
  \frac{1}{K}\sum_{k=1}^K\frac{1}{\barN}\sum_{i=1}^{\barN}\left\|\btheta_{k,i} \right\|^2= \Etheta,\quad \frac{1}{K}\sum_{k=1}^K\left\|\bw_k \right\|^2= \Ew,\quad
  \bartheta_{i} =\mathbf{0}_p, \notag ~i\in[\barN].
\end{equation}
Applying Lemma \ref{lemma: to ETF} shown in the end of the section, we have $\left(\bTheta, \bW\right)$ satisfies \eqref{eq:solution}. 

Reversely, it is easy to verify that  \eqref{eq:lowerboundff} reduces to equality  when $(\bTheta, \bW)$  admits \eqref{eq:solution}. So $L_0$ is the global minimum  of   \eqref{eq: NN simplified model} and  \eqref{eq:solution} is the unique form for the minimizers. We complete the proof of Theorem \ref{theo: cross-entropy balance}.
\end{proof}

\begin{lemma}\label{lemma: to ETF}
Suppose $\left(\bTheta, \bW\right)$  satisfies 
\begin{equation}\label{eq:paddcon}
  \bartheta_i  - \btheta_{k,i}  = -\sqrt{\frac{\Etheta}{\Ew}} \bw_k, \quad k\in[K], \quad i\in[\barN],
\end{equation} 
and
\begin{equation}\label{eq:conadd}
  \frac{1}{K}\sum_{k=1}^K\frac{1}{\barN}\sum_{i=1}^{\barN}\left\|\btheta_{k,i} \right\|^2= \Etheta,\quad \frac{1}{K}\sum_{k=1}^K\left\|\bw_k \right\|^2= \Ew,\quad
  \bartheta_{i} =\mathbf{0}_p, ~i\in[\barN],
\end{equation}
where $  \bartheta_i := \frac{1}{K}\sum_{k=1}^K \btheta_{k,i} $ with $i\in[n]$.  Moreover,  there exists a constant $C$ such that  for all $(k,i,\kk)\in \{(k,i,\kk): k\in[K], \kk\in[K], \kk\neq k, i\in[\barN]    \}$, we have
\begin{equation}\label{lemma:eqcondition}
\btheta_{k,i}\cdot\bw_k = \btheta_{k,i}\cdot\bw_{\kk} + C.
\end{equation}
Then $\left(\bTheta, \bW\right)$ satisfies \eqref{eq:solution}.
\end{lemma}

\begin{proof}
Combining \eqref{eq:paddcon} with the last equality in  \eqref{eq:conadd}, we have 
\begin{equation}
   \bW =   \sqrt{\frac{\Ew}{\Etheta}} ~ \bigg[\btheta_{1},\ldots,\btheta_{K}\bigg]^\top,\quad\quad
\btheta_{k,i} = \btheta_{k}, ~ k\in[K],~ i\in[\barN]. \notag
\end{equation}
Thus it remains to show  
\begin{eqnarray}\label{eq:wetf}
   \bW = \sqrt{\Ew}~ \left(\bM\right)^\top,
\end{eqnarray}
where $\bM$ is a $K$-simplex ETF.  

Plugging $\btheta_k = \btheta_{k,i} = \sqrt{\frac{\Ew}{\Etheta}} \bw_k$ into \eqref{lemma:eqcondition}, we have, for all $(k,\kk)\in \{(k,\kk):k\in[K], \kk\in[K], \kk\neq k \}$,
  $$ \sqrt{\frac{\Etheta}{\Ew}}\|\bw_k\|^2 = \btheta_{k}\cdot\bw_k = \btheta_{k}\cdot\bw_{\kk} + C =   \sqrt{\frac{\Etheta}{\Ew}}\bw_k\cdot\bw_{\kk} + C,$$
    and 
    $$\sqrt{\frac{\Etheta}{\Ew}}\|\bw_{\kk}\|^2 =  \btheta_{\kk}\cdot\bw_{\kk}= \btheta_{\kk}\cdot\bw_{k} + C  = \sqrt{\frac{\Etheta}{\Ew}}\bw_{\kk}\cdot\bw_{k} + C. $$
   Therefore, from  $\frac{1}{K}\sum_{k=1}^K\left\|\bw_k \right\|^2= \Ew$, we have  $\| \bw_k\| = \sqrt{\Ew}$  and  $\btheta_{k}\bw_{\kk} = C':=  \sqrt{\Etheta\Ew}- C$.
  
  Furthermore, recalling that    $ \bartheta_{i} =\mathbf{0}_p$ for $i\in[\barN]$, we have $\sum_{k=1}^K\btheta_{k} = \mathbf{0}_p$, which further yields  $\sum_{k=1}^K\btheta_{k} \cdot \bw_\kk = 0$ for $\kk\in[K]$. Then it follows from $\btheta_{k}\bw_{\kk} = C'$ and $\btheta_{k}\bw_{k} = \sqrt{\Etheta\Ew}$ that $\btheta_{k}\bw_{\kk} = - \sqrt{\Etheta\Ew}/(K-1)$.  Thus   we obtain 
$$ \bW \bW^\top  =  \sqrt{\frac{\Ew}{\Etheta}} \bW [\btheta_{1},\ldots,\btheta_{K}] = \Ew\left[\frac{K}{K-1}  \left( \bI_K - \frac{1}{K}\cI_K\cI_K^\top \right)\right],  $$
which implies \eqref{eq:wetf}.  We complete the proof.
\end{proof}

\begin{proof}[Proof of Proposition \ref{theo:counter example}]
We introduce the set  $\cS_R$ as  
\begin{equation}\label{eq:couter solution}
 \cS_R := \left\{ \left(\bTheta, \bW\right): \begin{matrix}
  [\btheta_{1},\ldots,\btheta_{K}]  = B_1 b\bP\left[  (a+1)\bI_K -  \cI_K \cI_K^\top\right],\\
   \bW =B_2 B_3  b\left[  (a+1)\bI_K -  \cI_K \cI_K^\top\right]^\top\bP^\top,\\
   \btheta_{k,i} = \btheta_{k}, \quad k\in[K],~ i\in[\barN]  \notag, \\
 b\geq 0,~a\geq 0,~  b^q[a^q +(K-1)] =1,\\
  |B_1|\leq \sqrt{\Etheta}, ~|B_2| \leq \sqrt{\Ew}, ~ B_3\geq 0,~ B_3^2b^2[a^2+(K-1)] =1,\\
  \bP\in \RR^{p\times K}, ~ \bP^\top \bP = \bI_K.
  \end{matrix}
 \right\}
\end{equation}
We can examine that $\cS_R$ admits the constraints of   \eqref{eq: NN simplified model}. So  any  $\left(\bTheta, \bW\right)\in \cS_R$ is a feasible solution. Moreover, one can observe that  this feasible solution has a special symmetry structure: for each $k\in[K]$,   the features in  class $k$ collapse to their mean $\btheta_k$, i.e., (\hyperlink{(NC1)}{NC1}), and  $\bw_k$ is parallel to  $\btheta_k$, i.e., (\hyperlink{(NC3)}{NC3}).  However,  weights do not form  the vertices of ETF unless $a=K-1$.   Therefore, it  suffices to show that the minimizer of  $\frac{1}{N} \sum_{k=1}^K\sum_{i=1}^{\barN} \cL( \bW\btheta_{k,i}, \by_k  )$  in the set $\cS_R$ do not satisfy $a=K-1$.

In fact,  for any $\left(\bTheta, \bW\right)\in \cS_R$,  the  objective function value can be written as a function of $B_1$, $B_2$, $B_3$, $a$, and $b$. We have
\begin{align}
&\frac{1}{N} \sum_{k=1}^K\sum_{i=1}^{\barN} \cL( \bW\btheta_{k,i}, \by_k  )\notag\\
=& -\log\left(\frac{ \exp (B_1B_2B_3 b^2 [a^2+(K-1)]  ) }{\exp(B_1B_2B_3 b^2 [a^2+K-1])+ (K-1)\exp(B_1B_2B_3 b^2[K-2-2a])}\right)\notag\\
 =& -\log\left(\frac{1}{1+ (K-1)\exp( - B_1B_2B_3 b^2(a+1)^2  )}\right). \notag
\end{align}
Then it follows to  maximize    $B_1B_2B_3 b^2(a+1)^2$ or equivalently  $\left[B_1B_2B_3b^2(a+1)^2\right]^2$.  By $B_3^2b^2[a^2+(K-1)] =1$ and $b^q[a^q +(K-1)]=1$, we have
\begin{align}
 \left[B_1B_2B_3 b^2 (a+1)^2\right]^2&\overset{a}{\leq}  \Etheta\Ew \left[B_3^2b^2 (a+1)^2\right]  \left[ b^2 (a+1)^2\right]\notag\\
 &=\Etheta\Ew \left[ \frac{(a+1)^2}{a^2+(K-1)} \right] \left[ \frac{(a+1)^q}{a^q +K-1}\right]^{2/q}, \label{eq:a+1} 
\end{align}
where $\overset{a}{\leq}$ picks $B_1 = \sqrt{\Etheta}$ and $B_2 = \sqrt{\Ew}$.
Let us consider function $g:[0,+\infty) \to \RR: g(x) = \left[ \frac{(x+1)^2}{x^2+(K-1)} \right] \left[ \frac{(x+1)^q}{x^q +K-1}\right]^{2/q}$.   Note that by the first-order optimality, once if $g'(K-1)\neq 0$,  then \eqref{eq:a+1} cannot achieve the maximum at $a=K-1$, which is our desired result. Indeed, we have 
$$g'(K-1) = \frac{2K^4}{\left[(K-1)^2+(K-1)\right]\left[(K-1)^{q} +K-1   \right]^{2/q+1}} \left[(K-1)-(K-1)^{q-1}\right].    $$
Therefore, $a= K-1 \ge 2$ is not the maximizer of \eqref{eq:a+1}, unless $q=2$.
We complete the proof.
\end{proof}

Following the proof of Proposition \ref{theo:counter example}, for completeness we discuss the structure of the global minimizers of Program (\ref{eq:counterproblem}) in the case $K = 2$. In short, we show that when $q\in (1,2)\cup (2, \infty)$, although the global minimizers of \eqref{eq:counterproblem} remain in the form of \eqref{eq:solution}, they are no longer rotationally invariant due to certain constraints on the solutions. This is in contrast to a $K$-simplex ETF, which is rotationally invariant (see Definition~\ref{def: ETF}).




For simplicity of notation, we assume that there is one training example in each class (the case of multiple training examples can be directly extended). Program (\ref{eq:counterproblem}) in the case $K = 2$ takes the following form:
\begin{equation}\label{eq:keq2}
    \begin{aligned}
\min_{\bW, \bH} \quad& -\log\left( \frac{\exp(\bw_1^\top \btheta_1)}{ \exp(\bw_1^\top \btheta_1) +  \exp(\bw_2^\top \btheta_1) }\right) -\log\left( \frac{\exp(\bw_2^\top \btheta_2)}{ \exp(\bw_1^\top \btheta_2) +  \exp(\bw_2^\top \btheta_2) }\right)\\
\mathrm{s.t.}\quad&\quad\quad\quad \left\|\bw_1 \right\|^2 +  \left\|\bw_2 \right\|^2 \leq 2\Ew,\\
& \quad\quad\quad \left\|\btheta_1 \right\|_q^q +   \left\|\btheta_2 \right\|_q^q \leq 2\Etheta.
    \end{aligned}
\end{equation}
We show that the optimal solution to (\ref{eq:keq2}) satisfies some specific ETF structures. In brief, when $q>2$, both the features and weights are parallel to a certain vector, and when $1<q<2$, the solution is sparse in the sense that only one entry is nonzero for both the features and the weights. 
\begin{lemma}
\label{lemma:binary}
For $q>2$, any global minimizer of \eqref{eq:keq2} satisfies
\begin{align}\label{eq:keq2solutionbig}
    \btheta_1^\star = - \btheta_2^\star =C_1 \bw_1^\star = -C_1\bw_2^\star =C_2 (\pm\mathbf{1}_p),
\end{align}
where the constants $C_1 = \left(\frac{\Etheta}{p}\right)^{1/q}\left(\frac{\Ew}{p}\right)^{-1/2}, C_2 = \left(\frac{\Etheta}{p}\right)^{1/q} $, and $\pm\mathbf{1}_p$ denotes a $p$-dimensional vector such that each entry is either $1$ or $-1$ (there are in total $2^p$ such vectors). For $1<q<2$,  any global minimizer of \eqref{eq:keq2} satisfies
\begin{align}\label{eq:keq2solutionsmall}
    \btheta_1^\star = - \btheta_2^\star =C_3 \bw_1^\star = -C_3\bw_2^\star, \quad  \|\btheta_1^\star \|_0 = 1, \quad  \|\btheta_1^\star \| = C_4,
\end{align}
where the constants $C_3 = \Etheta^{1/q}\Ew^{-1/2}$ and $C_4 = \Etheta^{1/q}$.
\end{lemma}

\begin{proof}
For any constants $C_a, C_b>0$, letting $C_c = \frac{C_b}{C_a+C_b}$, using the same arguments as \eqref{eq:logsim} and \eqref{eq: lower cross}, we have
\begin{align}\label{eq:keq2lowebound1}
    &-\log\left( \frac{\exp(\bw_1^\top \btheta_1)}{ \exp(\bw_1^\top \btheta_1) +  \exp(\bw_2^\top \btheta_1) }\right) -\log\left( \frac{\exp(\bw_2^\top \btheta_2)}{ \exp(\bw_1^\top \btheta_2) +  \exp(\bw_2^\top \btheta_2) }\right)\notag\\
    &\geq\frac{C_b}{C_a+C_b}\left[(\btheta_1+\btheta_2)^\top (\bw_1+\bw_2) - 2(\btheta_1^\top\bw_1+ \bh_2^\top \bw_2) \right]+C_d.
\end{align}
Then it follows that
\begin{align}\label{eq:k2sumsplit}
    (\btheta_1+\btheta_2)^\top (\bw_1+\bw_2) - 2(\btheta_1^\top\bw_1+ \bh_2^\top \bw_2)
    =-(\btheta_1 - \btheta_2)^\top(\bw_1 -\bw_2)
    \geq - \left\|\btheta_1 - \btheta_2 \right\|  \left\| \bw_1 -\bw_2\right\|.
\end{align}
We have
\begin{align}\label{eq:boundwk2}
   \left\| \bw_1 -\bw_2\right\|^2  =\left\| \bw_1\right\|^2+\left\| \bw_2\right\|^2 - 2\bw_1^\top \bw_2\leq  2\left\| \bw_1\right\|^2+2\left\| \bw_2\right\|^2 \leq 4\Ew
\end{align}
and 
\begin{align}\label{eq:boundhk2}
  \left\| \btheta_1 -\btheta_2\right\|^2 = \sum_{i=1}^p \left|\btheta_1(i) - \btheta_2(i)  \right|^2 \leq \sum_{i=1}^p \left(|\btheta_1(i)| + |\btheta_2(i)|  \right)^2  \overset{a}\leq 2^{2-2/q}\left[ \sum_{i=1}^p\left( |\btheta_1(i)|^q + |\btheta_2(i)|^q  \right)^{\frac{2}{q}} \right], 
\end{align}
where $\btheta_1(i)$ and $\btheta_2(i)$  denotes the $i$-th entry of $\btheta_1$ and $\btheta_2$, respectively. In $\overset{a}\leq$, we use Jensen's inequality that $\left( \frac{|\btheta_1(i)|+| \btheta_2(i)|}{2}\right)^q\leq \frac{|\btheta_1(i)|^q+|\btheta_2(i)|^q}{2}$ 
since $|x|^q$ is strictly convex.
\begin{itemize}
    \item When $1<q<2$, we pick $C_a =\exp\left( \Etheta^{1/q}\Ew^{1/2}\right)$ and $C_b=1/C_a$.  We  have
    \begin{align}\label{eq:qfrom1to2}
    \sum_{i=1}^p\left( |\btheta_1(i)|^q + |\btheta_2(i)|^q  \right)^{\frac{2}{q}}  \leq     \left( \sum_{i=1}^p |\btheta_1(i)|^q + |\btheta_2(i)|^q   \right)^{2/q} \leq 2^{2/q}\Etheta^{2/q}, 
    \end{align}
\end{itemize}
where the first inequality  uses that $\sum_{i=1}^p|x_i|^o  \leq \left(\sum_{i=1}^{p}|x_i|\right)^o  $ for $o>1$ and the equality holds if and only if the non-zero elements of $\{x_i\}_{i=1}^p$ is at most $1$.
Then by plugging \eqref{eq:boundwk2}, \eqref{eq:boundhk2}, and  \eqref{eq:qfrom1to2} into \eqref{eq:keq2lowebound1}, using \eqref{eq:k2sumsplit}, we have \begin{equation}\label{eq:qfrom1to2ff}
  -\log\left( \frac{\exp(\bw_1^\top \btheta_1)}{ \exp(\bw_1^\top \btheta_1) +  \exp(\bw_2^\top \btheta_1) }\right) -\log\left( \frac{\exp(\bw_1^\top \btheta_2)}{ \exp(\bw_1^\top \btheta_2) +  \exp(\bw_2^\top \btheta_2) }\right) \geq -\frac{C_b}{C_a+C_b}\sqrt{2^{4-4/q}\Ew \Etheta^{2/q} }  +C_d. 
\end{equation}
Now we check the conditions that reduce \eqref{eq:qfrom1to2ff} to an equality. \eqref{eq:k2sumsplit} reduces to an equality if and only if there exists a constant $C_5>0$ such that $ \btheta_1 -\btheta_2 = C_5(\bw_1 -\bw_2)$. \eqref{eq:boundwk2} reduces to an equality if and only if 
$\bw_1 =-\bw_2$ and $\|\bw_1\|^2 = \Ew$. \eqref{eq:boundhk2} reduces to an equality if and only if $\btheta_1 = -\btheta_2$.  Finally, \eqref{eq:qfrom1to2} reduces to an equality if and only if there is only one non-zero entry $i$ such that we exactly have $|\btheta_1(i)|^q + |\btheta_2(i)|^q =2\Etheta $. We can obtain \eqref{eq:keq2solutionbig}.

\item When $q>2$, we pick $C_a =\exp\left(p \left(\frac{\Etheta}{p}\right)^{1/q}\left(\frac{\Ew}{p}\right)^{1/2}\right)$ and $C_b=1/C_a$.  We have
  \begin{align}\label{eq:qlarge2}
    \sum_{i=1}^p\left( |\btheta_1(i)|^q + |\btheta_2(i)|^q  \right)^{\frac{2}{q}}  \leq  p^{1 - 2/q}   \left( \sum_{i=1}^p |\btheta_1(i)|^q + |\btheta_2(i)|^q   \right)^{2/q} \leq p^{1 - 2/q}  2^{2/q} \Etheta^{2/q}, 
    \end{align}
where the first inequality uses Jensen's inequality that $\left(\frac{1}{p}\sum_{i=1}^p |x_i|\right)^a \leq \frac{1}{p}\sum_{i=1}^p|x_i|^{a}$ for $a>1$ since $|x|^{a}$ is strictly convex with respect to $x$, and   let $a = q/2$, and  $x_i = \left(|\btheta_1(i)|^q + |\btheta_2(i)|^q  |\right)^{2/q}\geq0$.
Then by plugging \eqref{eq:boundwk2}, \eqref{eq:boundhk2}, and  \eqref{eq:qlarge2} into \eqref{eq:keq2lowebound1}, using \eqref{eq:k2sumsplit}, we have \begin{equation}\label{eq:qlarge2ff}
  -\log\left( \frac{\exp(\bw_1^\top \btheta_1)}{ \exp(\bw_1^\top \btheta_1) +  \exp(\bw_2^\top \btheta_1) }\right) -\log\left( \frac{\exp(\bw_1^\top \btheta_2)}{ \exp(\bw_1^\top \btheta_2) +  \exp(\bw_2^\top \btheta_2) }\right) \geq -\frac{C_b}{C_a+C_b}\sqrt{2^{4-4/q}p^{1 - 2/q}  \Ew \Etheta^{2/q} }  +C_d. 
\end{equation}
Now we check the conditions that reduce \eqref{eq:qfrom1to2ff} to an equality. In fact,  by the strict convexity, \eqref{eq:qlarge2} reduces to an equality if and only if $|\btheta_1(i)|^q + |\btheta_2(i)|^q =  |\btheta_1(j)|^q + |\btheta_2(j)|^q  $ for all $i\neq j$ and $  \sum_{i=1}^p |\btheta_1(i)|^q + |\btheta_2(i)|^q  = 2\Etheta$. Then by combining the conditions to reduce  \eqref{eq:k2sumsplit}, \eqref{eq:boundwk2},  and  \eqref{eq:boundhk2} to equalities, we can obtain \eqref{eq:keq2solutionbig}.

\end{proof}

\begin{figure}[!htp]
		\centering
		\includegraphics[scale=0.5]{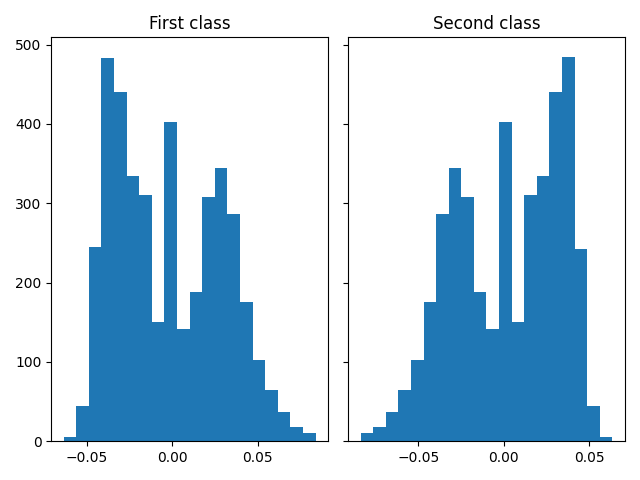}
		\caption{Histograms of last-layer weights of VGG11 trained on the first two classes in FashionMNIST. Each histogram shows the empirical distribution of all the entries of $\bw_1$ or $\bw_2$. If the prediction of Lemma~\ref{lemma:binary} applied to real neural networks for binary classification, then we would observe a mixture of one or two point masses in the histograms, which however is not the case. There are $6000$ examples in each class, and we use the same experimental settings as in Section \ref{subsec:experiments}. The training and test accuracies are $100\%$ and $99.75\%$, respectively. 
}	
\label{fig:binary-weight-distribution}
\end{figure}

Figure~\ref{fig:binary-weight-distribution} displays simulation results concerning the last-layer weights for binary classification using deep neural networks. The results show that last-layer weights exhibit neither the all-ones nor the sparse pattern as in Lemma~\ref{lemma:binary}, thereby implying that the $\ell_2$ norm is the best choice among all $\ell_q$ norms for modeling deep neural networks using the Layer-Peeled Model.


In the case where $q\leq 1$, we conjecture that the $\ell_q$ norm regularizer would also render the solution to \eqref{eq:keq2} sparse. We leave this for future work.

\subsubsection{Proofs of Theorems \ref{proposition: contrastive loss balance} and \ref{proposition: general loss balance}}
The proofs of  Theorems  \ref{proposition: contrastive loss balance} and \ref{proposition: general loss balance}  follow from the similar argument of Theorem \ref{theo: cross-entropy balance}.

\begin{proof}[Proof of Theorem \ref{proposition: contrastive loss balance}]
For $k\in[K]$, $i\in[\barN]$, and $\kk\in[K]$,  define
$$ E_{k, i, \kk}  :=  \frac{1}{\barN}\sum_{j=1}^{\barN}\exp(\btheta_{k,i} \cdot \btheta_{\kk,  j}/\tau). $$
For constants $C_a: =\exp\left(\sqrt{\Etheta\Ew} \right)$  and $ C_b:= (K-1)\exp\left(-\sqrt{\Etheta\Ew} \right)$, let $C_c := \frac{C_b}{(C_a+C_b)(K-1)}$. Using a similar argument as \eqref{eq:logsim}, we have for $j\in[\barN]$, 
\begin{align}\label{eq:contral1}
& -\log\left( \frac{\exp(\btheta_{k,i}\cdot \btheta_{k,j}/\tau )}{ \sum_{\kk=1}^K E_{k,i, \kk}}    \right)\\
=& - \btheta_{k,i}\cdot \btheta_{k,j}/\tau + \log\left( \frac{C_a}{C_a+C_b} \left(\frac{(C_a+C_b)~E_{k,i,k}}{C_a}\right) + C_c\sum_{\kk=1,~\kk\neq k}^{K}  \frac{E_{k,i, \kk} }{C_c}           \right)\notag\\
 \overset{a}{\geq}& - \btheta_{k,i}\cdot \btheta_{k,j}/\tau + \frac{C_a}{C_a+C_b} \log\left( \frac{(C_a+C_b)~E_{k,i,k}}{C_a}\right) + C_c\sum_{\kk=1,~\kk\neq k}^{K} \log\left(  \frac{E_{k,i, \kk} }{C_c} \right)   \notag\\
 \overset{b}=& - \btheta_{k,i}\cdot \btheta_{k,j}/\tau + \frac{C_a}{C_a+C_b} \log\left( E_{k,i,k}\right) + C_c\sum_{\kk=1,~\kk\neq k}^{K} \log\left(E_{k,i, \kk} \right) +C_d\notag\\
\overset{c}\geq& - \btheta_{k,i}\cdot \btheta_{k,j} /\tau+ \frac{C_a}{(C_a+C_b)\barN} \sum_{\ell=1}^{\barN} \btheta_{k,i} \cdot \btheta_{k, \ell}/\tau + \frac{C_c}{\barN}\sum_{\kk=1,~\kk\neq k}^{K} \sum_{\ell=1}^{\barN}\btheta_{k,i} \cdot \btheta_{\kk, \ell} /\tau +C_d.\notag
\end{align}
where $\overset{a}\geq$ and $\overset{c}\geq$  apply the concavity of $\log(\cdot)$ and in   $\overset{b}=$ we define $C_d :=  \frac{C_a}{C_a+C_b}\log(\frac{C_a+C_b}{C_a})+ \frac{C_b}{C_a+C_b}\log(1/C_c)$. Then plugging \eqref{eq:contral1} into the objective function, we have
\begin{align}\label{eq: contral2}
& \frac{1}{N}\sum_{k=1}^K \sum_{i=1}^{\barN} \frac{1}{\barN}\sum_{j=1}^{\barN} -\log\left( \frac{\exp(\btheta_{k,i}\cdot \btheta_{k,j}/\tau )}{\sum_{\kk=1}^K\sum_{\ell=1}^{\barN} \exp (\btheta_{k,i}\cdot \btheta_{\kk, \ell} )}    \right)\\
=&\frac{1}{N}\sum_{k=1}^K \sum_{i=1}^{\barN} \frac{1}{\barN}\sum_{j=1}^{\barN} -\log\left( \frac{\exp(\btheta_{k,i}\cdot \btheta_{k,j}/\tau )}{ \sum_{\kk=1}^K E_{k,i, \kk}}    \right) + \log(n)\notag\\
\overset{\eqref{eq:contral1}}\geq&
\frac{C_b K}{(C_a+C_b)N (K-1)\tau }\sum_{k=1}^K\sum_{i=1}^{\barN} \left(    - \frac{1}{\barN}\sum_{j=1}^{\barN}\left(\btheta_{k,i}\cdot \btheta_{k,j} - \frac{1}{K}\sum_{\kk=1}^K \btheta_{k,i}\cdot \btheta_{\kk, j} \right) \right)+C_d+ \log(n). \notag
\end{align}
Now defining $\bartheta_i := \frac{1}{K}\sum_{k=1}^K \btheta_{k,i}$ for $i\in[\barN]$, a similar argument as \eqref{eq:padd} and \eqref{eq:boundtheta} gives that
 \begin{align}\label{eq:contral3}
 &\sum_{k=1}^K\sum_{i=1}^{\barN} \left(    - \frac{1}{\barN}\sum_{j=1}^{\barN}\left(\btheta_{k,i}\cdot \btheta_{k,j} - \frac{1}{K}\sum_{\kk=1}^K \btheta_{k,i}\cdot \btheta_{\kk, j} \right) \right)\notag\\
 =&\sum_{k=1}^K\sum_{i=1}^{\barN}\left(    - \frac{1}{\barN}\sum_{j=1}^{\barN}\btheta_{k,i}\cdot (\btheta_{k,j} -\bartheta_{j})  \right)\notag\\
 \overset{a}\geq& -\frac{1}{2}\sum_{k=1}^K\sum_{i=1}^{\barN} \left\| \btheta_{k,i}\right\|^2 - \frac{1}{2} \sum_{k=1}^K \sum_{i=1}^{\barN}\left\|\btheta_{k,i} -\bartheta_{i}\right\|^2 \notag\\
 \overset{b}\geq&-\frac{1}{2}\sum_{k=1}^K\sum_{i=1}^{\barN} \left\| \btheta_{k,i}\right\|^2 - \frac{K}{2} \sum_{i=1}^{\barN}\left( \frac{1}{K} \sum_{k=1}^K\left\|\btheta_{k,i}\right\|^2 - \left\|\bartheta_{i}\right\|^2 \right)\notag\\
 \geq& -\sum_{k=1}^K\sum_{i=1}^{\barN} \left\| \btheta_{k,i}\right\|^2\overset{c}{\geq} -N\Etheta,
  \end{align}
 where $\overset{a}\geq$ follows from  $2ab \le a^2 + b^2$, $\overset{b}\geq$ follows from  $\E \|\ba-\E [\ba]  \|^2 =  \E \|\ba  \|^2  - \|\E[\ba]\|^2$, and $\overset{c}\geq$ uses the constraint of  \eqref{eq:contrastive loss}. Therefore, plugging \eqref{eq:contral3} into \eqref{eq: contral2} yields that
 \begin{align}
 &\frac{1}{N}\sum_{k=1}^K \sum_{i=1}^{\barN} \frac{1}{\barN}\sum_{j=1}^{\barN} -\log\left( \frac{\exp(\btheta_{k,i}\cdot \btheta_{k,j}/\tau )}{ \sum_{\kk=1}^K\sum_{\ell=1}^{\barN} \exp (\btheta_{k,i}\cdot \btheta_{\kk, \ell}/\tau )}    \right) \notag\\
 \geq& -\frac{C_bK \Etheta }{(C_a+C_b) (K-1)\tau }+ C_d +\log(n).\label{eq:contrafl}
 \end{align}

 Now we check the conditions that can make  \eqref{eq:contrafl} reduce equality.  By the strictly concavity of $\log(\cdot)$,    \eqref{eq:contral1} reduce to equalities only if for all $(k,i,\kk)\in \{(k,i,\kk): k\in[K], \kk\in[K], \kk\neq k, i\in[\barN]    \}$, 
\begin{equation}\label{eq:Eeqconvex}
\frac{E_{k,i,k}}{C_a(K-1)}  = \frac{E_{k,i,\kk}}{C_b}.
\end{equation}
\eqref{eq:contral3} reduce to equalities  if and only if: 
\begin{equation}
  \btheta_{k,i} = \btheta_{k},~i\in[\barN],~k\in[K],  \quad  \frac{1}{K}\sum_{k=1}^{K}\left\|\btheta_{k}\right\|^2= \Etheta, \quad 
 \sum_{k=1}^K\btheta_{k} =\mathbf{0}_p. 
\end{equation}
Plugging $\btheta_{k,i} = \btheta_{k}$ into \eqref{eq:Eeqconvex}, we have  for $(k,\kk)\in \{k,\kk: k\in[K], \kk\in[K], \kk\neq k    \}$,
$$ \frac{\exp(\|\btheta_k\|^2)}{C_a(K-1)} = \frac{\exp(\btheta_k\cdot \btheta_{\kk})}{C_b} =   \frac{\exp(\|\btheta_\kk\|^2)}{C_a(K-1)}.   $$
Then it follows from  $\frac{1}{K}\sum_{k=1}^{K}\left\|\btheta_{k}\right\|^2= \Etheta$ that  $\| \btheta_{k} \|^2= \Etheta$ for $k\in[K]$.  Moreover,  since    $\sum_{k=1}^K\btheta_{k} = \mathbf{0}_p$, we obtain 
$$ \btheta_k \cdot \btheta_{\kk}  = -\frac{\Etheta}{K-1} $$
for $(k,\kk)\in \{k,\kk: k\in[K], \kk\in[K], \kk\neq k    \}$.
Therefore, 
$$  [\btheta_{1},\ldots,\btheta_{K}]^\top[\btheta_{1},\ldots,\btheta_{K}]  =\Etheta \left[\frac{ K}{K-1}  \left( \bI_K - \cI_K\cI_K^\top \right)\right],  $$
which implies \eqref{eq:solution2}. 

Reversely, it is easy to verify that the equality for \eqref{eq:contrafl} is reachable when $\bTheta$ admits  \eqref{eq:solution2}. We complete the proof of Theorem \ref{proposition: contrastive loss balance}. 
\end{proof}

\begin{proof}[Proof of Theorem \ref{proposition: general loss balance}]
We first determine the minimum  of \eqref{eq: NN simplified model}. For the simplicity of our expressions, we introduce $\bz_{k,i} := \bW\btheta_{k,i}$ for $k\in[K]$ and $i\in[\barN]$.   By the  convexity of $g_2$, for any $k\in[K]$ and  $i\in[\barN]$,  we have 
\begin{align}\label{eq:szz}
   \sum_{\kk=1, ~\kk\neq k}^K g_2\left( \bS(\bz_{k,i})(\kk)\right) &\geq (K-1) g_2\left(  \frac{1}{K-1}\sum_{\kk=1, ~\kk\neq k}^K\bS(\bz_{k,i})(\kk) \right)\notag\\
   &\overset{a}= (K-1)g_2\left(1 - \frac{1}{K-1} \bS(\bz_{k,i})(k) \right),
\end{align}
where $\overset{a}=$ uses $\sum_{k=1}^K \bS(\ba)(k)=1 $ for any $\ba\in\RR^K$. Then it follows by the convexity of $g_1$ and $g_2$ that
\begin{align}
  & \frac{1}{N} \sum_{k=1}^K\sum_{i=1}^{\barN} \cL( \bW\btheta_{k,i}, \by_k  )\label{eq:lnonconvex}\\
   =& \frac{1}{N} \sum_{i=1}^{\barN}  \sum_{k=1}^K\left[ g_1\left(\bS(\bz_{k,i})(k)\right)+ \sum_{\kk=1, ~\kk\neq k}^K g_2\left( \bS(\bz_{\kk,i})(\kk) \right)\right]\notag\\
   \overset{\eqref{eq:szz}}{\geq}&   \frac{1}{N} \sum_{i=1}^{\barN}  \sum_{k=1}^K\left[ g_1\left(\bS(\bz_{k,i})(k)\right)+ (K-1) g_2\left(1 - \frac{1}{K-1} \bS(\bz_{k,i})(k) \right)\right]\notag\\
   \geq&g_1\left(\frac{1}{N} \sum_{i=1}^{\barN}  \sum_{k=1}^K\bS(\bz_{k,i})(k)\right)+ (K-1) g_2\left(1- \frac{1}{N(K-1)} \sum_{i=1}^{\barN}  \sum_{k=1}^K   \bS(\bz_{k,i})(k) \right) .\notag
\end{align}

Because $g_1(x)+(K-1) g_2(1-\frac{x}{K-1})$ is monotonously deceasing, it suffices to maximize  $$\frac{1}{N} \sum_{i=1}^{\barN}  \sum_{k=1}^K\bS(\bz_{k,i})(k).$$ To begin with, 
for any $\bz_{k,i}$ with $k\in[K]$ and $i\in[\barN]$,  by convexity of exponential function and the monotonicity of $q(x) =\frac{a}{a+x}$ for  $x>0$ if $a>0$, we have
\begin{align}
 \bS(\bz_{k,i})(k) &=  \frac{\exp(\bz_{k,i}(k))}{\sum_{\kk=1}^K\exp(\bz_{k,i}(\kk)) }\notag\\&\leq \frac{ \exp(\bz_{k,i}(k))  }{ \exp(\bz_{k,i}(k)) + (K-1) \exp\left(  \frac{1}{K-1}\sum_{\kk=1,~\kk\neq k}^K \bz_{k,i}(\kk)\right) }\notag\\
   &= \frac{1  }{ 1 + (K-1) \exp\left(  \frac{1}{K-1}\sum_{\kk=1,~\kk\neq k}^K \bz_{k,i}(\kk) - \bz_{k,i}(k)\right) }. \label{eq:expb}
\end{align}
Consider function  $g_0:\RR\to \RR$ as $g_0(x) = \frac{1}{1+C\exp(x)}$ with $C := (K-1)\geq1$. We have 
\begin{equation}\label{eq:ghessian}
g_0''(x) = -  \frac{\exp(x)(1+C\exp(x))( 1-C\exp(x))}{(1+C\exp(x))^4}.
\end{equation}
For any feasible solution   $\left(\bTheta, \bW\right)$ of   \eqref{eq: NN simplified model}, 
we divide the index set $[\barN]$ into two subsets $\cS_1$ and $\cS_2$ defined below:
\begin{itemize}
    \item[] \hypertarget{(bad case)}{(i)}     $i\in\cS_1$ if  there exists at least one $k\in[K]$  such that   $$\frac{1}{K-1}\sum_{\kk=1,~\kk\neq k}^K \bz_{k,i}(\kk) - \bz_{k,i}(k) \geq \log\left(\frac{1}{K-1}\right).$$
    \item[]  \hypertarget{(good case)}{(ii)}    $i\in\cS_2$ if for all  $k\in[K]$, 
    $\frac{1}{K-1}\sum_{\kk=1,~\kk\neq k}^K \bz_{k,i}(\kk) - \bz_{k,i}(k) < \log\left(\frac{1}{K-1}\right).$
\end{itemize}
Clearly,  $\cS_1 \cap \cS_2 = \varnothing$.
Let $| \cS_1 | =t$, then $| \cS_2| = \barN-t$.  Define function $L: [\barN]\to \RR$ as
\begin{equation}\label{eq:defL}
 L(t): =
\begin{cases}
 N- \left(  \frac{1}{2} t  +  \frac{K (\barN-t)}{1+ \exp\left( \frac{K}{K-1}\sqrt{\barN/(\barN-t)}\sqrt{\Etheta\Ew} - \log( K-1   )\right)} \right), &\quad t\in[0:\barN-1],\\
   N - \frac{\barN}{2}, &\quad t=\barN.
\end{cases}
\end{equation}
We show in Lemma \ref{lemma:general s} (see the end of the proof) that
\begin{equation}\label{eq:SSlower}
 \frac{1}{N} \sum_{i=1}^{\barN}  \sum_{k=1}^K\bS(\bz_{k,i})(i)  \leq \frac{1}{N} L(0).
\end{equation}
Plugging \eqref{eq:SSlower}  into \eqref{eq:lnonconvex},   the objective function can be lower bounded as:
\begin{equation}\label{eq: general llowerbound}
 \frac{1}{N} \sum_{k=1}^K\sum_{i=1}^{\barN} \cL( \bW\btheta_{k,i}, \by_k  )\geq g_1\left(\frac{1}{N}L(0)\right)+ (K-1) g_2\left(1- \frac{1}{N(K-1)}L(0)\right):=L_0 .
\end{equation}
On the other hand, one can directly verify that the equality for \eqref{eq: general llowerbound} is reachable when $(\bTheta, \bW)$ satisfies  \eqref{eq:solution}. 
So $L_0$ is the global minimum  of   \eqref{eq: NN simplified model} and
 \eqref{eq:solution} is a minimizer of   \eqref{eq: NN simplified model}.

Now we show all the solutions are in the form of \eqref{eq:solution} under the assumption that  $g_2$ is strictly convex and   $g_1$ (or $g_2$) is strictly monotone.

By the strict convexity of $g_2$,
the equality in \eqref{eq:szz} holds if and only if for any   $k\in[K]$ and $i\in[\barN]$ and  $\kk\in[K]$, $\kkk\in[K]$ such that for all $\kk\neq k$ and $\kkk\neq k$, we have 
  \begin{equation}
    \bS(\bz_{i,j})(\kk) = \bS(\bz_{i,j})(\kkk),   \notag
  \end{equation}
  which indicates that 
 \begin{equation} \label{eq:z2}
 \btheta_{k,i}\cdot \bw_{\kk} =   \btheta_{k,i}\cdot \bw_{\kkk}. 
 \end{equation} 
  Again,   by the  strict convexity of $g_2$, \eqref{eq:lnonconvex} holds  if and only if for all $k\in[K]$, $i\in[\barN]$, and a suitable number $C'\in(0,1)$, we have
  \begin{equation}
   \cS(\bz_{k,i})(k) : = C'. \label{eq:zij}
  \end{equation}
 Combining \eqref{eq:z2} with \eqref{eq:zij}, we have for all $(k,i,\kk)\in \{(k,i,\kk): k\in[K], \kk\in[K], \kk\neq k, i\in[\barN]    \}$,
\begin{equation}
\frac{\exp( \btheta_{k,i} \cdot \bw_k  )}{\exp( \btheta_{k,i} \cdot \bw_{\kk}  )} = \frac{C'(K-1)}{1-C'}, \notag
\end{equation}
which implies that 
$$ \btheta_{k,i}\cdot\bw_k = \btheta_{k,i}\cdot\bw_{\kk} + \log\left( \frac{C'(K-1)}{1-C'}  \right). $$
On the other hand, by the strict  monotonicity of $g_1(x)+(K-1) g_2(1-\frac{x}{K-1})$,   the equality in \eqref{eq: general llowerbound} holds if and only if  $\frac{1}{N} \sum_{i=1}^{\barN}  \sum_{k=1}^K\bS(\bz_{k,i})(k) = L(0)$. Thus  Lemma \ref{lemma:general s} reads
\begin{equation}
  \bartheta_i  - \btheta_{k,i}  = -\sqrt{\frac{\Etheta}{\Ew}} \bw_k, \quad k\in[K], \quad i\in[\barN], \notag
\end{equation}
and
\begin{equation}
  \frac{1}{K}\sum_{k=1}^K\frac{1}{\barN}\sum_{i=1}^{\barN}\left\|\btheta_{k,i} \right\|^2= \Etheta,\quad \frac{1}{K}\sum_{k=1}^K\left\|\bw_k \right\|^2= \Ew,\quad
  \bartheta_{i} =\mathbf{0}_p, \notag ~i\in[\barN],
\end{equation}
where $  \bartheta_i := \frac{1}{K}\sum_{k=1}^K \btheta_{k,i} $ with $i\in[n]$. Putting the pieces together, from Lemma \ref{lemma: to ETF}, we have $\left(\bTheta, \bW \right)$ satisfies \eqref{eq:solution}, achieving the uniqueness argument.  We complete the proof of Theorem \ref{proposition: general loss balance}.
\end{proof}

\begin{lemma}\label{lemma:general s}For any  feasible solution   $\left(\bTheta, \bW\right)$,     we have
\begin{equation}\label{lemma:eq1}
 \sum_{i=1}^{\barN}  \sum_{k=1}^K\bS( \bW\btheta_{k,i})(k) \leq L(0),
\end{equation}
with $L$ defined in \eqref{eq:defL}. Moreover, recalling the definition of $\cS_1$ and $\cS_2$ in  (\hyperlink{(bad case)}{i})  and  (\hyperlink{(good case)}{ii}), respectively,   the  equality in \eqref{lemma:eq1} holds if and only if $|\cS_1|=0$,   
\begin{equation}
  \bartheta_i  - \btheta_{k,i}  = -\sqrt{\frac{\Etheta}{\Ew}} \bw_k, \quad k\in[K], \quad i\in[\barN], \notag
\end{equation}
and 
\begin{equation}
 \frac{1}{K}\sum_{k=1}^K\frac{1}{\barN}\sum_{i=1}^{\barN}\left\|\btheta_{k,i} \right\|^2= \Etheta,\quad \frac{1}{K}\sum_{k=1}^K\left\|\bw_k \right\|^2= \Ew,\quad
  \bartheta_{i} =\mathbf{0}_p, \notag ~i\in[\barN],
\end{equation}
where $  \bartheta_i := \frac{1}{K}\sum_{k=1}^K \btheta_{k,i} $ with $i\in[n]$. 
\end{lemma}

\begin{proof}[Proof of Lemma \ref{lemma:general s}]
For any  feasible solution   $\left(\bTheta, \bW\right)$, we separately consider $\cS_1$ and $\cS_2$ defined  in (\hyperlink{(bad case)}{i})  and  (\hyperlink{(good case)}{ii}), respectively. Let $t:=|\cS_1 |$.  
\begin{itemize}
    \item For $i\in\cS_1$, let $k\in[K]$ be any index such that $\frac{1}{K-1}\sum_{\kk\neq k} \bz_{k,i}(\kk) - \bz_{k,i}(k) \geq \log\left(\frac{1}{K-1}\right)$, where $\bz_{k,i} := \bW\btheta_{k,i}$.  By the monotonicity of $g_0(x)$, it follows from \eqref{eq:expb} that  $S(\bz_{k,i})(k) \leq 1/2$. Furthermore,  for any other index $\kk\in[K]$ such that $\kk\neq k$, using that  $\frac{\exp(\bz_{\kk,i}(\kk))}{\sum_{\kkk=1}^K \exp(\bz_{\kk,i})(\kkk)} \leq 1$, we have 
\begin{equation}\label{eq:cass11}
 \sum_{i\in\cS_1} \sum_{k=1}^K\bS(\bz_{k,i})(k) \leq t(1/2+K-1).
\end{equation}
 \item For $i\in\cS_2$,  by the concavity of $g_0(x)$ when $x< \log\left(\frac{1}{K-1}\right)$ from \eqref{eq:ghessian}, we have, for $\cS_2\neq \varnothing$,
\begin{align}\label{eq:case21}
 & \sum_{i\in\cS_2} \sum_{k=1}^K\bS(\bz_{k,i})(k)\\ \overset{\eqref{eq:expb}}{\leq}& \sum_{i\in\cS_2} \sum_{k=1}^K \frac{1  }{ 1 + (K-1) \exp\left(  \frac{1}{K-1}\sum_{\kk=1,~\kk\neq k}^K  \bz_{k,i}(\kk) - \bz_{k,i}(k)\right) }\notag\\
 \leq&  \frac{(\barN-t)K }{ 1 + (K-1) \exp\left( \frac{1}{(\barN-t)K}\sum_{i\in\cS_2} \sum_{k=1}^K \left( \frac{1}{K-1}\sum_{\kk=1,~\kk\neq k}^K  \bz_{k,i}(\kk) - \bz_{k,i}(k)\right)\right) }.\notag
\end{align}
 We can  bound  $  \sum_{i\in\cS_2} \sum_{k=1}^K \left( \frac{1}{K-1}\sum_{\kk=1,~\kk\neq k}^K  \bz_{k,i}(\kk) - \bz_{k,i}(k)\right)$ using the similar arguments as \eqref{eq:padd} and \eqref{eq:boundtheta}. Specifically,  recalling $\bartheta_i = \frac{1}{K}\sum_{k=1}^K \btheta_{k,i}$ for $i\in[\barN]$,  we have
\begin{align}\label{eq:case22}
 &\sum_{i\in\cS_2} \sum_{k=1}^K \left(\frac{1}{K-1}\sum_{\kk=1,~\kk\neq k}^K \bz_{k,i}(\kk) - \bz_{k,i}(k)\right)\\
 = &\frac{1}{K-1}\sum_{i\in\cS_2}\left[ \left(\sum_{k=1}^K \btheta_{k,i}\right)^{\top}\left(\sum_{k=1}^K \bw_k\right)  -K\sum_{K=1}^{K}\btheta_{k,i}^\top\bw_k \right]\notag\\
 \overset{\eqref{eq:padd}}\geq&  - \frac{K}{2(K-1)}\sum_{k=1}^{K} \sum_{i\in \cS_2}  \|\bartheta_i  - \btheta_{k,i} \|^2 /C''-  \frac{C''K(\barN-t)}{2(K-1)}\sum_{k=1}^{K}\| \bw_k\|^2\notag\\
 \overset{\eqref{eq:boundtheta}}\geq& -\frac{K}{2(K-1)}\sum_{k=1}^{K} \sum_{i\in \cS_2}  \| \btheta_{k,i} \|^2 /C''-  \frac{C''K(\barN-t)}{2(K-1)}\sum_{k=1}^{K}\| \bw_k\|^2\notag\\
 \geq&-\frac{K}{2(K-1)}\sum_{k=1}^{K} \sum_{i=1}^{\barN}  \| \btheta_{k,i} \|^2 /C''-  \frac{C''K(\barN-t)}{2(K-1)}\sum_{k=1}^{K}\| \bw_k\|^2 \notag\\
 \geq& -\frac{K^2}{(K-1)}\sqrt{\Etheta\Ew(\barN-t)\barN},\notag
\end{align}
where in the last inequality we follow from the constraints of   \eqref{eq: NN simplified model} and set  $C'': = \sqrt{\frac{\barN\Etheta}{(\barN-t)\Ew}}$.
\end{itemize}
We  combine the above  two cases.   When $t\in[0, \barN-1]$, by plugging \eqref{eq:case22} into \eqref{eq:case21}, using the  monotonicity of $g_0(x)$, and  adding \eqref{eq:cass11}, we have 
\begin{align}
  \sum_{k=1}^{\barN}  \sum_{i=1}^K\bS(\bz_{k,i})(k)&\leq N- \left(  \frac{1}{2} t  +  \frac{K}{1+ \exp\left( \frac{K}{K-1}\sqrt{\barN/(\barN-t)}\sqrt{\Etheta\Ew} - \log( K-1   )\right)} (n-t) \right)\notag\\
&=L(t). \label{eq:sfinal}
\end{align}
 And when $t = \barN$, it directly follows from \eqref{eq:case21} that  
 $$\sum_{k=1}^{\barN}  \sum_{i=1}^K\bS(\bz_{k,i})(k)\leq N-\frac{\barN}{2}=L(\barN).$$ 
 Therefore, it suffices to show $L(t)\leq L(0)$ for all $t\in[0:\barN]$. We first consider the case when $t\in[0:N-1]$. We  show that $L(t)$ is monotonously decreasing. Indeed, define
\begin{equation}
 q(t) :=  \frac{K}{1+ \exp\left( \frac{K}{K-1}\sqrt{\barN/(\barN-t)}\sqrt{\Etheta\Ew} - \log( K-1   )\right)}. \notag
 \end{equation}
We have
\begin{align}
    q'(t) &= \frac{-\frac{1}{2}K\exp\left( \frac{K}{K-1}\sqrt{\barN/(\barN-t)}\sqrt{\Etheta\Ew} - \log( K-1   )\right)\frac{K}{K-1}\sqrt{\Etheta\Ew\barN} (\barN-t)^{-3/2} }{\left[1+ \exp\left( \frac{K}{K-1}\sqrt{\barN/(\barN-t)}\sqrt{\Etheta\Ew} - \log( K-1   )\right)\right]^2}\notag\\
    &\geq \frac{-\frac{1}{2}\frac{K^2}{K-1}\sqrt{\Etheta\Ew\barN} (\barN-t)^{-3/2} }{1+ \exp\left( \frac{K}{K-1}\sqrt{\barN/(\barN-t)}\sqrt{\Etheta\Ew} - \log( K-1   )\right)}\notag,
\end{align}
which implies that 
\begin{align}
&L'(t)= -\left[ \frac{1}{2} - q(t) +q'(t)(\barN-t)\right]\notag\\
\leq&   \frac{\frac{1}{2}\frac{K^2}{K-1}\sqrt{\Etheta\Ew\barN} (\barN-t)^{-1/2} +K }{1+ \exp\left( \frac{K}{K-1}\sqrt{\barN/(\barN-t)}\sqrt{\Etheta\Ew} - \log( K-1   )\right)}              -\frac{1}{2}\notag\\
=&   \frac{K\left(\frac{K}{K-1}\sqrt{\barN/(\barN-t)}\sqrt{\Etheta\Ew}\right) +2K-1-\exp\left( \frac{K}{K-1}\sqrt{\barN/(\barN-t)}\sqrt{\Etheta\Ew} - \log( K-1   )\right) }{2\left[1+ \exp\left( \frac{K}{K-1}\sqrt{\barN/(\barN-t)}\sqrt{\Etheta\Ew} - \log( K-1   )\right)\right]}.\notag
\end{align}
Consider function $f(x):\left[ \frac{K}{K-1}\sqrt{\Etheta\Ew},\frac{K}{K-1}\sqrt{\Etheta\Ew\barN}\right]\to R$ as:
$$ f(x) = K x +2K-1 - \exp(x-\log(K-1)).  $$
We have 
$$ f'(x) = K - \exp(x)/(K-1)< 0 $$
when $x\in \left[ \frac{K}{K-1}\sqrt{\Etheta\Ew},\frac{K}{K-1}\sqrt{\Etheta\Ew\barN}\right]$, where we use the assumption that
 \begin{equation}
     \sqrt{\Etheta\Ew} > \frac{K-1}{K}\log\left(K^2\sqrt{\Etheta\Ew} +(2K-1)(K-1)\right)\geq \frac{K-1}{K}\log\left(K(K-1)\right).\notag
 \end{equation}
 Therefore, for all $x\in\left[ \frac{K}{K-1}\sqrt{\Etheta\Ew},\frac{K}{K-1}\sqrt{\Etheta\Ew\barN}\right] $, we have \begin{equation}
     f(x) \leq f\left(\frac{K}{K-1}\sqrt{\Etheta\Ew} \right) = \frac{K^2}{K-1} \sqrt{\Etheta\Ew} +2K-1 -\frac{1}{K-1}\exp\left(\frac{K}{K-1}\sqrt{\Etheta\Ew} \right) \overset{a}<0,\notag
 \end{equation}
 where $\overset{a}<$ use our assumption again. 
 We obtain $L'(t)<0$ for all $t\in[0:N-1]$. So  $L(t)$ reaches the maximum if and only if  $t=0$ when $t\in[0:N-1]$.  Moreover, under our assumption, one can verify that $L(N)<L(0)$.   We obtain \eqref{lemma:eq1} from  \eqref{eq:sfinal} with $t=0$. 

 When $t=0$, the  first inequality of \eqref{eq:case22} reduces to equality if and only if:
$$   \bartheta_i  - \btheta_{k,i}  = -\sqrt{\frac{\Etheta}{\Ew}} \bw_k, \quad k\in[K], \quad i\in[\barN]. $$
 The  second and third inequalities of \eqref{eq:case22} reduce to equalities  if and only if: 
$$  \frac{1}{K}\sum_{k=1}^K\frac{1}{\barN}\sum_{i=1}^{\barN}\left\|\btheta_{k,i} \right\|^2= \Etheta,\quad \frac{1}{K}\sum_{k=1}^K\left\|\bw_k \right\|^2= \Ew,\quad
  \bartheta_{i} =\mathbf{0}_p, \notag ~i\in[\barN]. $$
We obtain Lemma \ref{lemma:general s}.
\end{proof}


\subsection{Imbalanced Case}
\label{sec:appimbalanced}

\subsubsection{Proofs of Lemma \ref{theo:to convex} and Proposition \ref{proposition: imbalance weight}}
\begin{proof}[Proof of Lemma \ref{theo:to convex}]
 For any feasible solution $\left(\bH, \bW\right)$ for the original program  \eqref{eq: NN simplified model},  we define 
\begin{equation}
    \bh_k := \frac{1}{n_k}\sum_{i=1}^{n_k} \bh_{k,i},~k\in[K], \quad \text{and}\quad
    \bX :=  \left[\btheta_{1},\btheta_{2},\dots,\btheta_{K},  \bW^\top  \right]^\top \left[\btheta_{1},\btheta_{2},\dots,\btheta_{K},  \bW^\top  \right]. \notag
\end{equation}
 Clearly, $\bX\succeq0$. For the other two constraints of \eqref{eq:convex sdp problem}, we have 
$$
 \frac{1}{K}\sum_{k=1}^K \bX(k,k)\notag\\
 = \frac{1}{K} \sum_{k=1}^K\|\bh_k \|^2
 \overset{a}\leq  \frac{1}{K}\sum_{k=1}^K \frac{1}{n_k}\sum_{i=1}^{n_k}\left\|\btheta_{k,i} \right\|^2\notag\\
 \overset{b}\leq E_{\Theta}, 
$$
and
$$ 
 \frac{1}{K}\sum_{k=K+1}^{2K} \bX(k,k)
 = \frac{1}{K} \sum_{k=1}^K\|\bw_k \|^2\overset{c}\leq \Ew, \notag
 $$
 where $\overset{a}\leq$ applies  Jensen's inequality and $\overset{b}\leq$ and $\overset{c}\leq$ use that $\left(\bH, \bW\right)$ is a feasible solution. So $\bX$ is a feasible solution for the convex program  \eqref{eq:convex sdp problem}.       Letting $L_0$ be the global minimum  of  \eqref{eq:convex sdp problem}, for any feasible solution $\left(\bH, \bW\right)$,  we obtain
 \begin{align}\label{eq:convex fea1}
    \frac{1}{N} \sum_{k=1}^K \sum_{i=1}^{n_k} \cL( \bW\btheta_{k,i}, \by_k ) &=  \sum_{k=1}^K\frac{\barN_k}{N}  \left[\frac{1}{\barN_k}\sum_{k=1}^{n_k} \cL( \bW\btheta_{k,i}, \by_k )\right]\notag\\
    &\overset{a}\geq   \sum_{k=1}^{K} \frac{\barN_k}{N} \cL( \bW\btheta_{k}, \by_k )
    =  \sum_{k=1}^{K}\frac{n_k}{N}  \cL( \bz_k, \by_k  ) \geq L_0, 
 \end{align}
 where in $\overset{a}\geq$, we use $\cL$ is convex on the first argument, and so  $\cL(\bW\bh, \by_k)$ is convex on $\bh$ given $\bW$ and $k\in[K]$.  
 
 On the other hand,  considering the solution $\left(\bH^{\star}, \bW^{\star}\right)$ defined in \eqref{eq:general solution} with $\bX^{\star}$ being a minimizer of  \eqref{eq:convex sdp problem}, we have  $\left[\btheta_{1}^{\star},\btheta_{2}^{\star},\dots,\btheta_{K}^{\star},  (\bW^{\star})^\top  \right]^\top \left[\btheta_{1}^{\star},\btheta_{2}^{\star},\dots,\btheta_{K}^{\star},  (\bW^{\star})^\top  \right] = \bX^{\star}$  ($p\geq 2K$ guarantees the existence of  $\left[\btheta_{1}^{\star},\btheta_{2}^{\star},\dots,\btheta_{K}^{\star},  (\bW^{\star})^\top  \right]$). We can verify that $\left(\bH^{\star}, \bW^{\star}\right)$ is a feasible solution for   \eqref{eq: NN simplified model} and  have
\begin{equation}\label{eq:convex fea2}
     \frac{1}{N} \sum_{k=1}^K \sum_{i=1}^{n_k} \cL( \bW^{\star}\btheta_{k,i}^{\star}, \by_k ) =  \sum_{k=1}^K  \frac{n_k}{N}\cL( \bz_k^{\star}, \by_k  ) = L_0, 
\end{equation}
 where $\bz_k^{\star} =  \left[\bX^{\star}(k,1+K), \bX^{\star}(k,2+K),\dots,\bX^{\star}(k,2K)    ~\right]^\top$ for $k\in[K]$.

 Combining \eqref{eq:convex fea1}  and \eqref{eq:convex fea2}, we conclude that  $L_0$ is the global minimum of   \eqref{eq: NN simplified model} and $(\bTheta^{\star}, \bW^{\star})$ is a minimizer.
 
 Suppose there is a minimizer $\left(\bH', \bW'\right)$  that cannot be written as    \eqref{eq:general solution}. Let  
 $$  \bh_k' = \frac{1}{n_k}\sum_{i=1}^{n_k} \bh_{k,i}',~k\in[K], \quad \text{and}\quad
    \bX' =  \left[\btheta_{1}',\btheta_{2}',\dots,\btheta_{K}',  (\bW')^\top  \right]^\top \left[\btheta_{1}',\btheta_{2}',\dots,\btheta_{K}',  (\bW')^\top  \right]. \notag $$
\eqref{eq:convex fea1} implies that $\bX'$ is a minimizer of  \eqref{eq:convex sdp problem}. As  $\left(\bH', \bW'\right)$ cannot be written as   \eqref{eq:general solution} with $\bX^{\star} = \bX'$, then there is a $\kk\in[K]$, $i,j\in[n_\kk]$ with $i\neq j$ such that $\bh_{\kk,i}\neq\bh_{\kk,j}$. We have 
\begin{align}
   &\frac{1}{K}\sum_{k=1}^K \bX'(k,k)
 = \frac{1}{K} \sum_{k=1}^K\|\bh_k' \|^2\notag\\
 =&\frac{1}{K}\sum_{k=1}^K \frac{1}{n_k}\sum_{i=1}^{n_k}\left\|\btheta_{k,i}' \right\|^2  -  \frac{1}{K}\sum_{k=1}^K 
    \frac{1}{n_k}\sum_{k=1}^K\| \bh_{k,i}' -\bh_{k}' \|^2 \notag\\ 
\leq& \frac{1}{K}\sum_{k=1}^K \frac{1}{n_k}\sum_{i=1}^{n_k}\left\|\btheta_{k,i}' \right\|^2  - \frac{1}{K} \frac{1}{n_\kk} (\| \bh_{\kk,i}' -\bh_{\kk}' \|^2+\| \bh_{\kk,j}' -\bh_{\kk}' \|^2) \notag\\
\leq& \frac{1}{K}\sum_{k=1}^K \frac{1}{n_k}\sum_{i=1}^{n_k}\left\|\btheta_{k,i}' \right\|^2  - \frac{1}{K} \frac{1}{2n_\kk}\| \bh_{\kk,i}' -\bh_{\kk,j}' \|^2 \notag\\
<& \Etheta.\notag
\end{align}
By contraposition,   if all  $\bX^{\star}$ satisfy that $\frac{1}{K}\sum_{k=1}^K \bX^{\star}(k,k) = \Etheta$, then all the solutions of    \eqref{eq: NN simplified model} are in the form of \eqref{eq:general solution}.  We complete the proof.
\end{proof}

Proposition \ref{proposition: imbalance weight} can be obtained by  the same argument. We omit the proof here.

\subsubsection{Proof of Theorem \ref{theo:imbalance limit}}
\label{app:proof_imbalance_limit}
 To prove Theorem \ref{theo:imbalance limit}, we first study a limit case where we only learn the classification for a partial classes.   We  solve  the optimization program:
\begin{equation}\label{eq: NN unblance limit}
\begin{aligned}
\min_{\bTheta, \bW} \quad& \frac{1}{K_A\barN_A} \sum_{k=1}^{K_A} \sum_{i=1}^{\barN_A} \cL( \bW\btheta_{k,i}, \by_k )\\
\mathrm{s.t.}\quad&\frac{1}{K}\sum_{k=1}^K \frac{1}{n_k}\sum_{i=1}^{n_k}\left\|\btheta_{k,i} \right\|^2 \leq E_{\Theta},\\
&     \frac{1}{K}\sum_{k=1}^K \left\|\bw_k \right\|^2 \leq E_{W},
\end{aligned}
\end{equation}
where $n_1 = n_2 = \dots = n_{K_A} = \barN_A$ and $n_{K_A+1} = n_{K_A+2}= \dots = n_{K} = \barN_B$.  Lemma \ref{lemma:mini unblance limit}  characterizes  useful properties for the  minimizer of  \eqref{eq: NN unblance limit}.
\begin{lemma}\label{lemma:mini unblance limit}
Let $(\bH, \bW)$ be a minimzer of  \eqref{eq: NN unblance limit}. We have $\bh_{k,i} = \mathbf{0}_p$ for all $k\in[K_A+1:K]$ and $i\in[\barN_B]$. Let $L_0$ be the global minimum of  \eqref{eq: NN unblance limit}. We have
$$L_0 =  \frac{1}{K_A\barN_A} \sum_{k=1}^{K_A} \sum_{i=1}^{\barN_A} \cL( \bW\btheta_{k,i}, \by_k ).$$
Then $L_0$ only depends on $K_A$, $K_B$, $\Etheta$, and $\Ew$. Moreover, for any feasible solution  $\left(\bH', \bW'\right)$, if there exist $k,\kk \in[K_A+1:K]$ such that  $\left\|\bw_k - \bw_\kk\right\| = \epsilon>0 $, we have 
$$ \frac{1}{K_A\barN_A} \sum_{k=1}^{K_A} \sum_{i=1}^{\barN_A} \cL\left( \bW'\btheta_{k,i}', \by_k \right) \geq L_0+\epsilon', $$
where $\epsilon'>0$ depends on $\epsilon$, $K_A$, $K_B$, $\Etheta$, and $\Ew$.
\end{lemma}

Now we are ready to prove Theorem \ref{theo:imbalance limit}.
The proof is based on the  contradiction.   

\begin{proof}[Proof of Theorem \ref{theo:imbalance limit}]
Consider sequences $n_A^{\ell}$ and $n_B^{\ell}$ with $R^{\ell} := n_A^{\ell}/n^{\ell}_B$ for $\ell = 1,2,\dots$. We have $R^{\ell}\to  \infty$.  For each optimization program indexed by $\ell\in\mathbb{N}_+$,  we introduce $(\bH^{\ell, \star}, \bW^{\ell,\star})$ as a minimizer and  separate the objective function  into two parts.  We consider
 $$\cL^{\ell}\left(\bTheta^{\ell}, \bW^{\ell}\right) = \frac{K_A\barN_A^{\ell}}{K_A\barN_A^{\ell}+K_B \barN_B^{\ell}} \cL^{\ell}_A\left(\bTheta^{\ell}, \bW^{\ell}\right) +  \frac{K_B\barN_B^{\ell}}{K_A\barN_A^{\ell}+K_B \barN_B^{\ell}} \cL^{\ell}_B\left(\bTheta^{\ell}, \bW^{\ell}\right), $$
with
$$ \cL^{\ell}_A\left(\bTheta^{\ell}, \bW^{\ell}\right) := \frac{1}{K_A\barN_A^{\ell}} \sum_{k=1}^{K_A} \sum_{i=1}^{\barN_A^{\ell}} \cL\left( \bW^{\ell}\btheta_{k,i}^{\ell}, \by_k \right)  $$
and 
$$ \cL^{\ell}_B\left(\bTheta^{\ell}, \bW^{\ell}\right) := \frac{1}{K_B\barN_B^{\ell}} \sum_{k=K_A+1}^{K} \sum_{i=1}^{\barN_B^{\ell}} \cL\left( \bW^{\ell}\btheta_{k,i}^{\ell}, \by_k \right). $$
 We define $\left(\bH^{\ell,A}, \bW^{\ell,A} \right)$ as a minimizer of the optimization program:
\begin{equation}\label{eq: NN unblance limit A}
\begin{aligned}
\min_{\bH^{\ell}, \bW^{\ell}} \quad& \cL^{\ell}_A\left(\bTheta^{\ell}, \bW^{\ell}\right)\\
\mathrm{s.t.}\quad&\frac{1}{K}\sum_{k=1}^K \left\|\bw_k^{\ell} \right\|^2 \leq E_{W},\\
& \frac{1}{K}\sum_{k=1}^{K_A} \frac{1}{n_A^{\ell}}\sum_{i=1}^{n_A^{\ell}}\left\|\btheta_{k,i}^{\ell} \right\|^2 + \frac{1}{K}\sum_{k=K_A+1}^K \frac{1}{n_B^{\ell}}\sum_{i=1}^{n_B^{\ell}}\left\|\btheta_{k,i}^{\ell} \right\|^2\leq E_{\Theta},
\end{aligned}
\end{equation}
and $\left(\bH^{\ell,B}, \bW^{\ell,B} \right)$ as a minimizer of the optimization program:
\begin{equation}\label{eq: NN unblance limit B}
\begin{aligned}
\min_{\bH^{\ell}, \bW^{\ell}} \quad& \cL^{\ell}_B\left(\bTheta^{\ell}, \bW^{\ell}\right)\\
\mathrm{s.t.}\quad&  \frac{1}{K}\sum_{k=1}^K \left\|\bw_k^{\ell} \right\|^2 \leq \Ew,\\
&    \frac{1}{K}\sum_{k=1}^{K_A} \frac{1}{n_A^{\ell}}\sum_{i=1}^{n_A^{\ell}}\left\|\btheta_{k,i}^{\ell} \right\|^2 + \frac{1}{K}\sum_{k=K_A+1}^K \frac{1}{n_B^{\ell}}\sum_{i=1}^{n_B^{\ell}}\left\|\btheta_{k,i}^{\ell} \right\|^2\leq E_{\Theta}.
\end{aligned}
\end{equation}
Note that  Programs  \eqref{eq: NN unblance limit A} and \eqref{eq: NN unblance limit B} and their minimizers have been studied in Lemma \ref{lemma:mini unblance limit}. We define:
$$ L_A :=  \cL^{\ell}_A\left(\bTheta^{\ell,A}, \bW^{\ell,A}\right) \quad\text{and}\quad  L_B :=  \cL^{\ell}_B\left(\bTheta^{\ell,B}, \bW^{\ell,B}\right). $$
Then Lemma \ref{lemma:mini unblance limit} implies  that $L_A$ and $L_B$ only depend on $K_A$, $K_B$, $\Etheta$, and $\Ew$,  and are independent of $\ell$. Moreover, since $\bh_{k,i}^{\ell,A} = \mathbf{0}_p$ for all $k\in[K_A+1:K]$ and $i\in[\barN_B]$, we have
\begin{equation}\label{eq:lblog}
 \cL^{\ell}_B\left(\bTheta^{\ell,A}, \bW^{\ell,A}\right) = \log(K). 
\end{equation}

Now we prove Theorem \ref{theo:imbalance limit} by contradiction. Suppose there exists a pair $(k,\kk)$ such that $\lim_{\ell\to \infty}\bw^{\ell,\star}_k - \bw^{\ell,\star}_\kk \neq \mathbf{0}_p $.  Then there exists $\epsilon>0$ such that for a subsequence $\left\{\left(\bTheta^{a_\ell,\star}, \bW^{a_\ell,\star}\right)\right\}_{\ell=1}^{\infty}$ and an index $\ell_0$ when $\ell\geq \ell_0$, we have  $\left\|\bw^{a_\ell,\star}_k - \bw^{a_{\ell},\star}_\kk\right\|  \geq \epsilon$. Now we figure out a contradiction by estimating the objective function value on $\left(\bTheta^{a_\ell,\star}, \bW^{a_\ell,\star}\right)$. In fact,  because  $\left(\bTheta^{a_\ell,\star}, \bW^{a_\ell,\star}\right)$ is a minimizer of $\cL^{\ell}(\bTheta^{\ell}, \bW^{\ell})$, we have
\begin{align}\label{eq: imbalance upper for L}
\cL^{a_\ell}\left(\bTheta^{a_\ell,\star}, \bW^{a_\ell,\star}\right) \leq  \cL^{a_\ell}\left(\bTheta^{a_\ell,A}, \bW^{a_\ell,A}\right) &\overset{\eqref{eq:lblog}}=    \frac{K_A\barN_A^{a_\ell}}{K_A\barN_A^{a_\ell}+K_B \barN_B^{a_\ell}} L_A +  \frac{K_B\barN_B^{a_\ell}}{K_A\barN_A^{a_\ell}+K_B \barN_B^{a_\ell}}\log(K)\notag\\
&= L_A  + \frac{1}{K_R R^{a_\ell} +1} \left(\log(K) - L_A\right) \overset{\ell\to\infty}{\to} L_A,
\end{align}
where we define $K_R := K_A/K_B$ and use $R^{\ell} = \barN_A^{\ell}/\barN_B^{\ell}$.

However, when $\ell>\ell_0$,  because $\left\|\bw^{a_\ell,\star}_k - \bw^{a_{\ell},\star}_\kk\right\|  \geq \epsilon>0$,  Lemma \ref{lemma:mini unblance limit} implies that 
$$  \cL^{a_\ell}_A\left(\bTheta^{a_\ell,\star}, \bW^{a_\ell,\star}\right) \geq L_A+\epsilon_2, $$
where $\epsilon_2>0$ only depends on $\epsilon$, $K_A$, $K_B$, $\Etheta$, and $\Ew$,  and is independent of $\ell$. We obtain
\begin{align}\label{eq: imbalance lower for L}
\cL^{a_\ell}\left(\bTheta^{a_\ell,\star}, \bW^{a_\ell,\star}\right)  &=   \frac{K_A\barN_A^{a_\ell}}{K_A\barN_A^{a_\ell}+K_B \barN_B^{a_\ell}}\cL^{a_\ell}_A\left(\bTheta^{a_\ell,\star}, \bW^{a_\ell,\star}\right)+\frac{K_B\barN_B^{a_\ell}}{K_A\barN_A^{a_\ell}+K_B \barN_B^{a_\ell}} \cL^{a_\ell}_B\left(\bTheta^{a_\ell,\star}, \bW^{a_\ell,\star}\right)\notag\\
&\overset{a}\geq   \frac{K_A\barN_A^{a_\ell}}{K_A\barN_A^{a_\ell}+K_B \barN_B^{a_\ell}}\cL^{a_\ell}_A\left(\bTheta^{a_\ell,\star}, \bW^{a_\ell,\star}\right)+ \frac{K_B\barN_B^{a_\ell}}{K_A\barN_A^{a_\ell}+K_B \barN_B^{a_\ell}}\cL^{a_\ell}_B\left(\bTheta^{a_\ell,B}, \bW^{a_\ell,B}\right)\notag\\
&=  \frac{K_A\barN_A^{a_\ell}}{K_A\barN_A^{a_\ell}+K_B \barN_B^{a_\ell}} (L_A+ \epsilon_2) + \frac{K_B\barN_B^{a_\ell}}{K_A\barN_A^{a_\ell}+K_B \barN_B^{a_\ell}}L_B\notag\\
 &= L_A + \epsilon_2 + \frac{1}{K_R R^{a_\ell}+1}(  L_B - L_A - \epsilon_2) \overset{\ell\to\infty}{\to} L_A+\epsilon_2,
\end{align}
where $\overset{a}\geq$ uses $\left(\bTheta^{a_\ell,B}, \bW^{a_\ell,B}\right)$ is the minimizer of  \eqref{eq: NN unblance limit B}.
Thus we meet contradiction by
comparing \eqref{eq: imbalance upper for L} with \eqref{eq: imbalance lower for L} and achieve Theorem \ref{theo:imbalance limit}.  
\end{proof}

\begin{proof}[Proof of Lemma \ref{lemma:mini unblance limit}]
For any constants $C_a>0$, $C_b>0$, and $C_c>0$, define $C_a' :=  \frac{C_a}{C_a+(K_A-1)C_b+K_B C_c} \in(0,1)$,   $C_b' :=  \frac{C_b}{C_a+(K_A-1)C_b+K_B C_c}\in(0,1)$, and $C_c' :=  \frac{C_c}{C_a+(K_A-1)C_b+K_B C_c}\in(0,1)$,   $C_d := -C_a'\log(C_a')-C_b'(K_A-1)\log(C_b')-K_B C_c'\log(C_c') $, $C_e:=\frac{K_A C_b}{K_A C_b+ K_B C_c}\in(0,1)$, $C_f:= \frac{K_B C_c}{K_A C_b+ K_B C_c  } \in(0,1)$, and  $C_g:=  \frac{K_AC_b+K_B C_c}{C_a +(K_A -1)C_b + K_B C_c}>0$.   Using a similar argument  as Theorem \ref{theo: cross-entropy balance}, we show in Lemma \ref{lemma:imbalance lemma fun1} (see the end of the proof), for any feasible solution $(\bH, \bW)$ of  \eqref{eq: NN unblance limit}, the objective value can be  bounded from below by:
\begin{align}
 &\frac{1}{K_A\barN_A} \sum_{k=1}^{K_A} \sum_{i=1}^{\barN_A} \cL( \bW\btheta_{k,i}, \by_k )\label{eq: imbalance lower12}\\
\overset{a}\geq& - \frac{C_g}{K_A}  \sqrt{K \Etheta}\sqrt{ \sum_{k=1}^{K_A} \left\|C_e \bw_A +C_f \bw_B - \bw_{k}\right\|^2 } + C_d\notag\\
\overset{b}\geq &  -\frac{C_g}{K_A}  \sqrt{ K\Etheta}\sqrt{ K\Ew  - K_A\left( 1/K_R -C_f^2 -  \frac{C_f^4}{C_e(2 -C_e)}  \right) \|\bw_B \|^2- \sum_{k=K_A+1}^{K}\left\|\bw_k -\bw_B \right\|^2 } + C_d, \notag
\end{align}
where $\bw_A := \frac{1}{K_A}\sum_{k=1}^{K_A} \bw_k$,  $\bw_B := \frac{1}{K_B}\sum_{k=K_A+1}^{K} \bw_k$, and $K_R := \frac{K_A}{K_B} $. Moreover, the equality in $\overset{a}\geq$ holds only if $\bh_{k,i} = \mathbf{0}_p$ for all $k\in[K_A+1:K]$ and $i\in[\barN_B]$.

Though $C_a$, $C_b$, and $C_c$ can be any positive numbers, we need to carefully pick  them to exactly reach the  global minimum of   \eqref{eq: NN unblance limit}. In the following, we separately consider three cases  according to the values of $K_A$, $K_B$, and $\Etheta\Ew$. 

\begin{itemize}
   \item[]  \hypertarget{Case:A}{(i)}  Consider the case when $K_A=1$.
   We pick $C_a := \exp\left(\sqrt{K_B(1+K_B)\Etheta\Ew}\right)$, $C_b :=1$, and $C_c := \exp\left(-\sqrt{(1+K_B)\Etheta\Ew/K_B} \right)$.  
   
Then from $\overset{a}\geq$ in \eqref{eq: imbalance lower12}, we have
   \begin{align}\label{eq:imbalance case 0}
        &\frac{1}{K_A\barN_A} \sum_{k=1}^{K_A} \sum_{i=1}^{\barN_A} \cL( \bW\btheta_{k,i}, \by_k )\notag\\
        \overset{a}\geq& - C_gC_f  \sqrt{K \Etheta} \sqrt{\left\| \bw_1  - \bw_B\right\|^2} + C_d\notag\\
      =& - C_gC_f  \sqrt{K \Etheta} \sqrt{\| \bw_1\|^2-2\bw_1^\top\bw_B  + \|\bw_B\|^2} + C_d\notag\\
      \overset{b}\geq&  - C_gC_f  \sqrt{K \Etheta} \sqrt{(1+1/K_B)(\| \bw_1\|^2+K_B \|\bw_B\|^2)} + C_d\notag\\
      \overset{c}\geq&  -C_gC_f  \sqrt{K \Etheta}\sqrt{(1+1/K_B)\left(K\Ew-\sum_{k=2}^{K}\|\bw_k - \bw_B\|^2 \right) }+C_d\notag\\
      \geq&  -C_gC_f  \sqrt{K \Etheta}\sqrt{(1+1/K_B)K\Ew}+C_d:=L_1,
   \end{align}
   where $\overset{a}\geq$ uses $C_e+C_f = 1$,   $\overset{b}\geq$ follows from $2ab \le a^2 + b^2$, i.e., $-2\bw_1^\top\bw_B \leq (1/K_B)\|\bw_1 \|^2+ K_B \|\bw_B\|^2 $, and $\overset{c}\geq$ follows from  $\sum_{k=2}^{K}\|\bw_k\|^2 = K_B \| \bw_B\|^2 +  \sum_{k=2}^{K}\|\bw_k - \bw_B\|^2$ and the constraint that $\sum_{k=1}^{K}\|\bw_k\|^2\leq K\Ew$.
   
 On the other hand, when $(\bTheta, \bW)$ satisfies  that
 \begin{eqnarray}\label{eq:imbalance solution0}
 \begin{aligned}
  \bw_1&= \sqrt{K_B\Ew}\bu, \quad  \bw_k=- \sqrt{1/K_B\Ew}\bu, ~k\in[2:K], \notag\\
\btheta_{1,i} =& \sqrt{(1+K_B)\Etheta}\bu, ~ i\in[n_A],\quad\quad
\btheta_{k,i} =  \mathbf{0}_p, ~ k\in[2:K],~ i\in[\barN_B]  \notag,\\
 \end{aligned}
 \end{eqnarray}
 where $\bu$ is any unit vector,  the inequalities in \eqref{eq:imbalance case 0} reduce to equalities. So $L_1$ is the global minimum  of  \eqref{eq: NN unblance limit}.   Moreover, $L_1$ is achieved only if  $\overset{a}\geq$  in  \eqref{eq: imbalance lower12} reduces to inequality. From Lemma \ref{eq: NN unblance limit}, we have that any minimizer  satisfies that   $\bh_{k,i} = \mathbf{0}_p$ for all $k\in[K_A+1:K]$ and $i\in[\barN_B]$.
 
Finally, for any feasible solution  $\left(\bH', \bW'\right)$, if there exist $k,\kk \in[K_A+1:K]$ such that  $\left\|\bw_k - \bw_\kk\right\| = \epsilon>0 $, we have 
 \begin{equation}\label{eq:dif eps}
  \sum_{k=K_A+1}^{K}\|\bw_k - \bw_B\|^2 \geq \|\bw_k - \bw_B\|^2+ \|\bw_\kk - \bw_B\|^2\geq \frac{\|\bw_k - \bw_\kk\|^2}{2} = \epsilon^2/2.
  \end{equation}
  It follows from $\overset{c}\geq$ in \eqref{eq:imbalance case 0} that 
\begin{equation}
       \frac{1}{K_A\barN_A} \sum_{k=1}^{K_A} \sum_{i=1}^{\barN_A} \cL( \bW\btheta_{k,i}, \by_k )\geq  -C_gC_f  \sqrt{K \Etheta}\sqrt{(1+1/K_B)\left(K\Ew-\epsilon^2/2 \right) }+C_d:= L_1+\epsilon_1 \notag
\end{equation}
with $\epsilon_1>0$ depending on $\epsilon$, $K_A$, $K_B$, $\Etheta$, and $\Ew$.

    \item  \hypertarget{Case:B}{(ii)}   Consider the case when  $K_A>1$ and  $\exp\left( (1+1/K_R) \sqrt{\Etheta\Ew} /(K_A-1) \right) <  \sqrt{1+K_R}+1$. Let us pick $C_a := \exp\left( (1+1/K_R)\sqrt{ \Etheta\Ew }\right)$, $C_b := \exp\left( -\frac{1}{K_A -1} (1+1/K_R) \sqrt{ \Etheta\Ew } \right) $, and $C_c :=1$. 
    
  Following from $\overset{b}\geq$ in \eqref{eq: imbalance lower12},  we know if  $ 1/K_R -C_f^2 -  \frac{C_f^4}{C_e(2 -C_f)}>0$,  then
    \begin{eqnarray}\label{eq:imbalance lower bound f}
    &\frac{1}{K_A\barN_A}  \sum_{k=1}^{K_A} \sum_{i=1}^{\barN_A} \cL( \bW\btheta_{k,i}, \by_k ) \geq  -C_g(1+1/K_R)\sqrt{\Etheta\Ew} + C_d:=L_2.
    \end{eqnarray}
In fact,  we do have $1/K_R -C_f^2 -  \frac{C_f^4}{C_e(2 -C_f)}>0$ because
\begin{eqnarray}\label{eq:unblance case simple}
\begin{aligned}
 \quad&1/K_R >  C_f^2 -  \frac{C_f^4}{C_e(2 -C_e)}\quad\quad\quad \left(\text{by~} C_e+C_f =1 \right)\\
\iff\quad&  C_e > \sqrt{\frac{1}{1+ K_R}} \quad\quad\quad \left(\text{by~} C_e = \frac{K_B C_c}{K_A C_b+K_B C_c} \right)\notag\\
\iff \quad &  \frac{C_b}{C_c} > \frac{1}{\sqrt{1+K_R}+1}\notag\\
\iff \quad  & \exp\left( (1+1/K_R) \sqrt{\Etheta\Ew} /(K_A-1) \right) <  \sqrt{1+K_R}+1.
\end{aligned}
\end{eqnarray}
 On the other hand, when $(\bTheta, \bW)$ satisfies  that
 \begin{eqnarray}\label{eq:imbalance solution}
 \begin{aligned}
  \left[\bw_1, \bw_2, \ldots, \bw_{K_A}\right] =&  \sqrt{\frac{\Ew}{\Etheta}} ~ \bigg[\btheta_{1},\ldots,\btheta_{K_A}\bigg]^\top  = \sqrt{(1+1/K_R)\Ew} ~(\mathbf{M}_A^{\star})^\top,\notag\\
\btheta_{k,i} =& \btheta_{k}, \quad k\in[K_A],~ i\in[\barN_A] \notag\\
\btheta_{k,i} =& \bw_k = \mathbf{0}_p, \quad k\in[K_A+1:K],~ i\in[\barN_B]  \notag,\\
 \end{aligned}
 \end{eqnarray}
 where $\mathbf{M}_A^{\star}$ is a $K_A$-simplex ETF,  \eqref{eq:imbalance lower bound f} reduces to equality. So $L_2$ is the global minimum  of  \eqref{eq: NN unblance limit}. Moreover, $L_2$ is achieved only if  $\overset{a}\geq$  of \eqref{eq: imbalance lower12} reduces to equality. From Lemma \ref{lemma:imbalance lemma fun1}, we have that any minimizer  satisfies that   $\bh_{k,i} = \mathbf{0}_p$ for all $k\in[K_A+1:K]$ and $i\in[\barN_B]$.
 
Finally, for any feasible solution  $\left(\bH', \bW'\right)$, if there exist $k,\kk \in[K_A+1:K]$ such that  $\left\|\bw_k - \bw_\kk\right\| = \epsilon>0 $,  plugging \eqref{eq:dif eps} into $\overset{b}\geq$ in \eqref{eq: imbalance lower12}, we have 
\begin{equation}\label{eq:dif ep22}
        \frac{1}{K_A\barN_A}  \sum_{k=1}^{K_A} \sum_{i=1}^{\barN_A} \cL( \bW\btheta_{k,i}, \by_k ) \geq  -\frac{C_g}{K_A}\sqrt{K\Etheta}\sqrt{K\Ew - \epsilon^2/2} + C_d:=L_2+\epsilon_2, 
\end{equation}
with $\epsilon_2>0$ depending on $\epsilon$, $K_A$, $K_B$, $\Etheta$, and $\Ew$.

    \item  \hypertarget{Case:C}{(iii)} Consider the case when $K_A>1$ and   $\exp((1+1/K_R) \sqrt{\Etheta\Ew} /(K_A-1) ) \geq  \sqrt{1+K_R}+1$.
Let $C_f' :=  \frac{1}{\sqrt{K_R +1}}$ and $C_e' := 1 - C_f'$. For $x\in[0,1]$,  we define:
\begin{eqnarray}
\begin{aligned}
 g_N(x): &= \sqrt{\frac{(1+K_R)\Ew}{K_R x^2+ (K_R+K_R^2)(1-x)^2} },\notag\\
 g_a(x): &=  \exp\left(\frac{g_N(x) \sqrt{(1+K_R)\Etheta/K_R} }{ \sqrt{x^2 +\left(1+ \frac{C_e'}{C_f'}\right)^2 (1-x)^2}}\left[  x^2 + \left(1+ \frac{C_e'}{C_f'}\right)(1-x)^2  \right]\right),\notag\\
 g_b(x): &=  \exp\left(\frac{g_N(x) \sqrt{(1+K_R)\Etheta/K_R} }{ \sqrt{x^2 +\left(1+ \frac{C_e'}{C_f'}\right)^2 (1-x)^2}}\left[ - \frac{1}{K_A -1} x^2 + \left(1+ \frac{C_e'}{C_f'}\right)(1-x)^2  \right]\right),\notag\\
 g_c(x): &=  \exp\left(\frac{g_N(x) \sqrt{(1+K_R)\Etheta/K_R} }{\sqrt{x^2 +\left(1+ \frac{C_e'}{C_f'}\right)^2 (1-x)^2}}\left[  -\left(1+ \frac{C_e'}{C_f'}\right)K_R(1-x)^2  \right]\right).
\end{aligned}
\end{eqnarray}
Let $x_0\in[0,1]$ be a root of the equation
$$ g_b(x) / g_c(x) = \frac{1/C_f' - 1}{K_R}. $$
We first show that the solution $x_0$ exists.  First of all, one can directly verify when $x\in[0,1]$,  $ g_b(x) / g_c(x)$ is continuous. It suffices to prove that (A) $g_b(0) / g_c(0)  \geq \frac{1/C_f' - 1}{K_R} $ and  (B) $g_b(1) / g_c(1)  \leq \frac{1/C_f' - 1}{K_R}$.

\begin{itemize}
    \item[(A)]  When $x = 0$, we have  $g_b(x) / g_c(x) \geq \exp(0)=1$. At the same time, 
     $\frac{1/C_f' - 1}{K_R} =\frac{\sqrt{K_R+1}-1}{K_R} = \frac{1}{\sqrt{K_R+1}+1} \leq 1$. Thus  $(i)$ is achieved. 
    \item[(B)]  When $x = 1$,  we have $g_N(1) =  \sqrt{(1+1/K_R)\Ew}$, so
    \begin{eqnarray}
    \begin{aligned}
      g_b(1) /  g_c(1) =  \exp\left( -(1+1/K_R) \sqrt{\Etheta\Ew} /(K_A-1) \right) \overset{a}\leq \frac{1}{\sqrt{K_R+1}+1}  = \frac{1/C_f' - 1}{K_R}.\notag
    \end{aligned}
    \end{eqnarray}
    where $\overset{a}\leq$ is obtained by the condition that  $$\exp\left( (1+1/K_R) \sqrt{\Etheta\Ew} /(K_A-1) \right) \geq  \sqrt{1+K_R}+1.$$
\end{itemize}
Now we pick $C_a := g_a(x_0) $,  $C_b := g_b(x_0)$, and $C_c := g_c(x_0)$, because $\frac{C_b }{C_c} = \frac{1/C_f' - 1}{K_R}$, we have $C_e = C_e'$ and $C_f = C_f'$ and   $1/K_R =C_f^2 +  \frac{C_f^4}{C_e(2 -C_e)}$. Then it follows from $\overset{b}\geq$ in \eqref{eq: imbalance lower12} that
    \begin{eqnarray}\label{eq:imbalance lower bound f2}
    &\frac{1}{K_A\barN_A}  \sum_{k=1}^{K_A} \sum_{i=1}^{\barN_A} \cL( \bW\btheta_{k,i}, \by_k ) \geq  -C_g(1+1/K_R)\sqrt{\Etheta\Ew} + C_d=L_2.
    \end{eqnarray}
    On the other hand,  consider the solution $(\bTheta, \bW)$ that satisfies 
\begin{eqnarray}\label{eq: imbalance solution 2}
\begin{aligned}
&\bw_{k} =  g_N(x_0) \bP_A\left[   \frac{x_0}{\sqrt{(K_A-1)K_A}}(K_A\by_k -\cI_{K_A})    +\frac{1-x_0}{\sqrt{K_A}} \cI_{K_A} \right],  \quad k\in[K_A],\notag\\
&\bw_{k} =  - \frac{C_e(2- C_e)}{C_f^2 K_A}\bP_A \sum_{k=1}^{K_A} \bw_k, \quad k\in[K_A+1:K],\notag\\
&\btheta_{k,i} =  \frac{\sqrt{(1+1/K_R)\Etheta}}{\| \bw_{i} + \frac{C_e}{C_fK_A} \sum_{k=1}^{K_A} \bw_k\|} \bP_A \left[\bw_{i} + \frac{C_e}{C_fK_A} \sum_{k=1}^{K_A} \bw_k\right], \quad k\in[K_A], ~i\in[\barN_A],\notag\\
&\btheta_{k,i} =\mathbf{0}_p,  \quad k\in[K_A+1:K], ~i\in[\barN_B],
\end{aligned}
\end{eqnarray}
where $\by_k\in\RR^K$ is the vector containing one in the $k$-th entry and zero elsewhere and 
$\bP_A\in \RR^{p\times K_A}$ is a partial orthogonal matrix such that $\bP^\top_A \bP_A = \bI_{K_A}$.
We have  $\exp\left(\btheta_{k,i}^\top \bw_k\right) = g_a(x_0)$ for $i\in[\barN_A]$ and $k\in[K_A]$, $\exp\left(\btheta_{k,i}^\top \bw_\kk\right) = g_b(x_0)$ for   $i\in[\barN_A]$ and $k,\kk\in[K_A]$ such that  $k\neq \kk$, and $\exp\left(\btheta_{k,i}^\top \bw_\kk\right) = g_c(x_0)$ for $i\in[\barN_A]$,  $k\in[K_A]$, and $\kk\in[K_B]$.  Moreover, $(\bTheta, \bW)$ can achieve the equality in \eqref{eq:imbalance lower bound f2}. Finally, following the same argument as Case (\hyperlink{Case:B}{ii}), we have that (1)  $L_2$ is the global minimum  of  \eqref{eq: NN unblance limit}; (2) any minimizer  satisfies that   $\bh_{k,i} = \mathbf{0}_p$ for all $k\in[K_A+1:K]$ and $i\in[\barN_B]$; (3) for any feasible solution  $\left(\bH', \bW'\right)$, if there exist $k,\kk \in[K_A+1:K]$ such that  $\left\|\bw_k - \bw_\kk\right\| = \epsilon>0 $,  then \eqref{eq:dif ep22} holds. 
\end{itemize}
Combining the three cases, we  obtain Lemma \ref{lemma:mini unblance limit}, completing the proof.
\end{proof}

\begin{lemma}\label{lemma:imbalance lemma fun1}
For any constants $C_a>0$, $C_b>0$, and $C_c>0$,  define $C_a' :=  \frac{C_a}{C_a+(K_A-1)C_b+K_B C_c}\in(0,1)$,   $C_b' :=  \frac{C_b}{C_a+(K_A-1)C_b+K_B C_c}\in(0,1)$, and $C_c' :=  \frac{C_c}{C_a+(K_A-1)C_b+K_B C_c}\in(0,1)$,   $C_d := -C_a'\log(C_a')-C_b'(K_A-1)\log(C_b')-K_B C_c'\log(C_c') $, $C_e:=\frac{K_A C_b}{K_A C_b+ K_B C_c}\in(0,1)$, $C_f:= \frac{K_B C_c}{K_A C_b+ K_B C_c  } \in(0,1)$, and  $C_g:=  \frac{K_AC_b+K_B C_c}{C_a +(K_A -1)C_b + K_B C_c}>0$.    For any feasible solution $(\bH, \bW)$ of  \eqref{eq: NN unblance limit}, the objective value of   \eqref{eq: NN unblance limit} can be  bounded from below by:
\begin{align}
 &\frac{1}{K_A\barN_A} \sum_{k=1}^{K_A} \sum_{i=1}^{\barN_A} \cL( \bW\btheta_{k,i}, \by_k )\label{eq: imbalance lower1}\\
\overset{a}\geq& - \frac{C_g}{K_A}  \sqrt{K \Etheta}\sqrt{ \sum_{k=1}^{K_A} \left\|C_e \bw_A +C_f \bw_B - \bw_{k}\right\|^2 } + C_d\notag\\
\overset{b}\geq &  -\frac{C_g}{K_A}  \sqrt{ K\Etheta}\sqrt{ K\Ew\!  - K_A\left( 1/K_R -C_f^2 -  \frac{C_f^4}{C_e(2 -C_e)}  \right) \|\bw_B \|^2-\! \sum_{k=K_A+1}^{K}\left\|\bw_k -\bw_B \right\|^2} + C_d, \notag
\end{align}
where $\bw_A := \frac{1}{K_A}\sum_{k=1}^{K_A} \bw_k$,  $\bw_B := \frac{1}{K_B}\sum_{k=K_A+1}^{K} \bw_k$, and $K_R := \frac{K_A}{K_B} $. Moreover, the equality in $\overset{a}\geq$ hold only if $\bh_{k,i} = \mathbf{0}_p$ for all $k\in[K_A+1:K]$.
\end{lemma}

\begin{remark}
Note that the case $\bh_{k,i} = \mathbf{0}_p$ does not imply that the network activations all die for the classes $k\in[K_A+1:K]$. This is because our analysis does not include the bias term for simplicity.
\end{remark}

\begin{proof}[Proof of Lemma \ref{lemma:imbalance lemma fun1}]
For $k\in[K_A]$ and $i\in[n_k]$, we introduce $\bz_{k,i} = \bW\btheta_{k,i}$. Because that $C_a'+(K_A-1)C_b'+K_B C_c'=1$, $C_a'>0$, $C_b'>0$, and $C_c'>0$,   by the concavity of $\log(\cdot)$,  we have
\begin{align}\label{eq:imbalance log}
  &- \log\left(\frac{\exp(\bz_{k,i}(i))}{\sum_{\kk=1}^{K} \exp(\bz_{\kk,i}(k))}\right)\\
  =&  -\bz_{k,i}(k) + \log\left(  C_a' \left( \frac{\exp(z_{k,i}(k))}{C_a'} \right) +\sum_{\kk=1,~\kk\neq k}^{K_A} C_b' \left( \frac{\exp(z_{k,i}(\kk))}{C_b'} \right) +\sum_{\kk=K_A+1}^{K}  C_c' \left( \frac{\exp(z_{k,i}(\kk))}{C_c'} \right) \right)\notag\\
  \geq& -  \bz_{k,i}(k)+ C_a' \bz_{k,i}(k)+  C_b'\sum_{\kk=1,~\kk\neq k}^{K_A}\bz_{k,i}(\kk)+ C_C' \sum_{\kk=K_A+1}^{K} \bz_{i,j}(k) +C_d\notag\\
  =& C_g  C_e\left( \frac{1}{K_A}\sum_{\kk=1}^{K_A}\bz_{k,i}(\kk) - \bz_{k,i}(k) \right)+  C_g  C_f\left( \frac{1}{K_B}\sum_{\kk=K_A+1}^{K}\bz_{k,i}(\kk) - \bz_{k,i}(k) \right) +C_d.\notag
\end{align}
 Therefore,  integrating \eqref{eq:imbalance log} with $k\in[K_A]$ and $i\in[n_A]$,  recalling that $\bw_A = \frac{1}{K_A}\sum_{k=1}^{K_A} \bw_k$ and $\bw_B = \frac{1}{K_B}\sum_{k=K_A+1}^{K} \bw_k$, 
 we have
  \begin{align}
 &\frac{1}{K_A\barN_A} \sum_{k=1}^{K_A} \sum_{i=1}^{\barN_A} \cL( \bW\btheta_{k,i}, \by_k )\label{eq: imbalance lower}\\
 \geq& \frac{1}{K_A\barN_A} \sum_{k=1}^{K_A} \sum_{i=1}^{\barN_A} C_g\left[ C_e(\btheta_{k,i}\bw_A  -  \btheta_{k,i}\bw_k ) +C_f(\btheta_{k,i}\bw_B  -  \btheta_{k,i}\bw_k )      \right]  +C_d\notag\\
 \overset{a}{=}& \frac{C_g}{K_A}\sum_{k=1}^{K_A} \btheta_{k}^\top (C_e \bw_A +C_f \bw_B - \bw_{k})+C_d,\notag
  \end{align}
where in $\overset{a}=$, we introduce $\btheta_{k} := \frac{1}{n_k} \sum_{i=1}^{n_k}\btheta_{k,i}$ for $k\in[K]$, and  use $C_e+C_f = 1$. Then it is sufficient to bound $\sum_{k=1}^{K_A} \btheta_{k}^\top (C_e \bw_A +C_f \bw_B - \bw_{k})$.
By the Cauchy–Schwarz inequality, we have
 \begin{align}
 \sum_{k=1}^{K_A} \btheta_{k}^\top (C_e \bw_A +C_f \bw_B - \bw_{k})\geq& - \sqrt{\sum_{k=1}^{K_A}\|\btheta_{k}\|^2 }\sqrt{ \sum_{k=1}^{K_A} \left\|C_e \bw_A +C_f \bw_B - \bw_{k}\right\|^2   }\notag\\
 \overset{a}\geq& - \sqrt{\sum_{k=1}^{K_A} \frac{1}{n_k} \sum_{i=1}^{n_k} \|\btheta_{k,i}\|^2 }\sqrt{ \sum_{k=1}^{K_A} \left\|C_e \bw_A +C_f \bw_B - \bw_{k}\right\|^2 }\notag\\
 \overset{b}\geq& - \sqrt{K\Etheta} \sqrt{ \sum_{k=1}^{K_A} \left\|C_e \bw_A +C_f \bw_B - \bw_{k}\right\|^2 },   \label{eq: imbalance bound theta}
 \end{align}
where $\overset{a}\geq$ follows from Jensen's inequality $\frac{1}{n_k}\sum_{i=1}^{n_k} \|\btheta_{k,i}\|^2 \geq \btheta_{k}$ for $k\in[K_A]$  and  $\overset{b}\geq$ uses the constraint that $\frac{1}{K}\sum_{k=1}^K \frac{1}{n_k}\sum_{i=1}^{n_k}\left\|\btheta_{k,i} \right\|^2 \leq E_{\Theta}$. Moreover,  we have $\sum_{k=1}^{K_A} \frac{1}{n_k}\sum_{i=1}^{n_k}\left\|\btheta_{k,i} \right\|^2 = E_{\Theta}$ only if $\bh_{k,i} = \mathbf{0}_p$ for all $k\in[K_A+1:K]$.   Plugging \eqref{eq: imbalance bound theta} into \eqref{eq: imbalance lower}, we obtain $\overset{a}\geq$ in  \eqref{eq: imbalance lower1}.

We then bound $\sum_{k=1}^{K_A} \left\|C_e \bw_A +C_f \bw_B - \bw_{k}\right\|^2$.  First, we  have
     \begin{align}\label{eq: imbalance bound w}
      &\frac{1}{K_A}\sum_{k=1}^{K_A} \left\|C_e \bw_A +C_f \bw_B - \bw_{k}\right\|^2\notag\\
      =&\frac{1}{K_A}\sum_{k=1}^{K_A}\| \bw_k\|^2  - 2\frac{1}{K_A} \sum_{k=1}^{K_A}\bw_k \cdot (C_e\bw_A+C_f\bw_B) + \| C_e\bw_A+C_f\bw_B \|^2 \notag\\
      \overset{a}=& \frac{1}{K_A}\sum_{k=1}^{K_A}\| \bw_k\|^2  - 2C_f^2 \bw_A^\top \bw_B  - C_e(2 - C_e) \|\bw_A\|^2+    C_f^2 \| \bw_B\|^2.
     \end{align}
where $ \overset{a}=$ uses $\sum_{k=1}^{K_A} \bw_k = K_A \bw_A$.
Then using the constraint that $\sum_{k=1}^K\|\bw_k \| \leq K\Ew$ yields that
    \begin{align}\label{eq: imbalance fun l}
      &\frac{1}{K_A}\sum_{k=1}^{K_A}\| \bw_k\|^2  - 2C_f^2 \bw_A^\top \bw_B  - C_e(2 - C_e) \|\bw_A\|^2+    C_f^2 \| \bw_B\|^2\\
      \leq&  \frac{K}{K_A}\Ew^2  -  \frac{1}{K_A}\sum_{k=K_A+1}^{K}\!\| \bw_k\|^2 -  C_e(2 -C_f)\left\|\bw_A + \frac{C_f^2}{C_e(2 -C_e)}\bw_B\right\|^2 \!\!+\! \left(C_f^2 +  \frac{C_f^4}{C_e(2 -C_e)} \right)\| \bw_B\|^2\notag\\
      \overset{a}{=}&   \frac{K}{K_A}\Ew^2 - \left( 1/K_R -C_f^2 -  \frac{C_f^4}{C_e(2 -C_e)}  \right) \|\bw_B \|^2 - \frac{1}{K_A}\sum_{k=K_A+1}^{K}\!\left\|\bw_k -\bw_B \right\|^2,\notag
      \end{align}
   where  $\overset{a}\geq$ applies $\sum_{k=K_A+1}^{K} \|\bw_k\|^2 = K_B\|\bw_B\|^2  + \sum_{k=K_A+1}^{K}\left\|\bw_k -\bw_B \right\|^2 $. Plugging \eqref{eq: imbalance bound w} and \eqref{eq: imbalance fun l} into $\overset{a}\geq$ in \eqref{eq: imbalance lower1}, we obtain  $\overset{b}\geq$ in \eqref{eq: imbalance lower1}, completing the proof.
\end{proof}

%% file: 08-additional.tex
\subsection{Additional Results}\label{sec:more convex relation}

\paragraph{Comparison of Oversampling and Weighted Adjusting.}
 Oversampling and weight adjusting are two 
commonly-used tricks in deep learning \cite{johnson2019survey}.   Both of them actually consider the  same objective as \eqref{eq: loss:adjust weight}, but applies different optimization algorithms to minimize the objective.
  It was observed that oversampling is more stable than weight adjusting in optimization.  As a by product of this work, we compare the two algorithms below and 
  shows that  the variance of updates for  oversampling will be potentially much smaller than  that of  weight adjusting. It was well-known in stochastic optimization field that the variance of the updates decides the convergence of an optimization algorithm (see e.g, \cite{bottou2018optimization,fang2018spider,fang2019sharp}). Thus we  offer a reasonable justification for the stability of  the oversampling technique. 
We simply consider sampling  the training data  without replacement. It slightly differs from the deep learning training methods in practice.   Besides, we only  consider sampling a single data in each update.  The analysis can be directly extended to the mini-batch setting. 

We first introduce the two methods.
    The weight adjusting algorithm in each update randomly samples a training data, and updates the parameters $\wf$ by the Stochastic Gradient Descent algorithm as
\begin{eqnarray}
 \wf^{t+1} =  \wf^{t}- \eta_w \bv_w^t, \quad t=0,1,2,\dots,
\end{eqnarray}
where $\wf^t$ denotes the parameters at iteration step $t$, $\eta_w$ is a positive step size, and the stochastic gradient $\bv_w^t$ satisfies that
\begin{equation}
\bv_w^t =  
\begin{cases}
\nabla_{\wf} \cL(f(\bx_{k,i}; \wf^t ),\by_k),& k\in[K_A], i\in[\barN_{A}],  \text{~with probability~} \frac{1}{K_A n_A+K_B n_B},\\
w_r\nabla_{\wf} \cL(f(\bx_{k,i}; \wf^t ),\by_k),& k\in[K_A+1:K_B], i\in[\barN_B], \text{~with probability~} \frac{1}{K_A n_A+K_B n_B}.  \notag
\end{cases}
\end{equation}
We have
\begin{align}
&\E \left[\bv_w^t \mid \wf^t \right]\\
=&   \frac{1}{\barN_AK_A+\barN_BK_B }\left[ \sum_{k=1}^{K_A}\sum_{i=1}^{\barN_A} \nabla_{\wf} \cL(f(\bx_{k,i}; \wf^t ),\by_k) +  w_r\!\!\sum_{k=K_A+1}^{K}\!\sum_{i=1}^{\barN_B} \nabla_{\wf} \cL(f(\bx_{k,i}; \wf^t ),\by_k)\right],  \notag
\end{align}
and 
\begin{align}\label{eq:second weight}
\E \left[\|\bv_w^t\|^2 \mid \wf^t \right]=&   \frac{1}{\barN_AK_A+\barN_BK_B } \sum_{k=1}^{K_A}\sum_{i=1}^{\barN_A} \left\|\nabla_{\wf} \cL(f(\bx_{k,i}; \wf^t ),\by_k)\right\|^2\notag\\
&+\frac{ w_r^2}{\barN_AK_A+\barN_BK_B }  \sum_{k=K_A+1}^{K}\sum_{i=1}^{\barN_B}\left\| \nabla_{\wf} \cL(f(\bx_{k,i}; \wf^t ),\by_k)\right\|^2.  
\end{align}

For the oversampling method, the algorithm in effect duplicates the data by $w_r$ times and runs  Stochastic Gradient Descent on the ``whole'' data. Therefore,   the update goes as
\begin{eqnarray}
 \wf^{t+1} =  \wf^{t}- \eta_s \bv_s^t, \quad t=0,1,2,\dots,
\end{eqnarray}
where $\bv_s^t$ satisfies that
\begin{equation}
\bv_s^t =  
\begin{cases}
\nabla_{\wf} \cL(f(\bx_{k,i}; \wf^t ),\by_k),& k\in[K_A], i\in[\barN_{A}],  \text{~with probability~} \frac{1}{K_A n_A+K_B w_r n_B},\\
\nabla_{\wf}  \cL(f(\bx_{k,i}; \wf^t ),\by_k),& k\in[K_A+1:K_B], i\in[\barN_B], \text{~with probability~} \frac{w_r}{K_A n_A+K_B w_r n_B}.  \notag
\end{cases}
\end{equation}
We obtain 
\begin{align}
\E \left[\bv_s^t \mid \wf^t \right]
=&   \frac{1}{\barN_AK_A+w_r\barN_BK_B }\sum_{k=1}^{K_A}\sum_{i=1}^{\barN_A} \nabla_{\wf} \cL(f(\bx_{k,i}; \wf^t ),\by_k)\notag\\
&+   \frac{w_r}{\barN_AK_A+w_r\barN_BK_B }\sum_{k=K_A+1}^{K}\sum_{i=1}^{\barN_B} \nabla_{\wf} \cL(f(\bx_{k,i}; \wf^t ),\by_k), \notag
\end{align}
and 
\begin{align}\label{eq:second up}
\E \left[\|\bv_s^t\|^2 \mid \wf^t \right]=&   \frac{1}{\barN_AK_A+w_r\barN_BK_B } \sum_{k=1}^{K_A}\sum_{i=1}^{\barN_A} \left\|\nabla_{\wf} \cL(f(\bx_{k,i}; \wf^t ),\by_k)\right\|^2\notag\\
&+\frac{ w_r}{\barN_AK_A+w_r\barN_BK_B }  \sum_{k=K_A+1}^{K}\sum_{i=1}^{\barN_B}\left\| \nabla_{\wf} \cL(f(\bx_{k,i}; \wf^t ),\by_k)\right\|^2.  
\end{align}

We suppose the two updates in expectation are in a same scale.  That means we  assume \ $\eta_w =  \frac{n_AK_A+w_rn_BK_B}{n_AK_A+n_BK_B} \eta_s$.  Then $\eta_w \E \left[ \bv_w^t \mid \wf^t \right] =  \eta_s \E \left[ \bv_s^t \mid \wf^t \right]$. In fact, if $K_A\asymp 1$, $K_B \asymp 1$, $n_A \gg n_B$, and $1\ll w_r \lesssim \left(n_A/n_B\right) $, we  have $\frac{n_AK_A+w_rn_BK_B}{n_AK_A+n_BK_B} \asymp 1$ and so $\eta_w \asymp \eta_s$.  Now by comparing \eqref{eq:second weight} with \eqref{eq:second up}, we obtain that the second moment of $\eta_w  \bv_w^t$ is much smaller than that of $\eta_s  \bv_s^t$ since the order of $w_r$ for the latter is larger by $1$.  For example, let us  assume that  all the  norms of the gradients are in a same order, i.e.,  $\left\| \nabla_{\wf} \cL(f(\bx_{k,i}; \wf^t ),\by_k)\right\| \asymp a $ for all $k$ and $i$, where $a>0$. Then 
\eqref{eq:second up} implies that
$\E \left[\|\eta_s \bv_s^t\|^2 \mid \wf^t \right] \asymp  \eta_s^2 a^2  $. However, \eqref{eq:second weight} reads that $\E \left[\|\eta_w \bv_w^t\|^2 \mid \wf^t \right] \asymp   \eta_s^2  \frac{n_AK_A+w_r^2n_BK_B }{n_AK_A+w_rn_BK_B} a^2 $. Furthermore,  if we set $w_r  \asymp n_A/n_B$, then $\E \left[\|\eta_w \bv_w^t\|^2 \mid \wf^t \right] \asymp   \eta_s^2  w_r a^2 $. Thus the second moment for $\eta_w  \bv_w^t$ is around $w_r$ times of that for $\eta_s  \bv_s^t$. And  this fact also holds for the variance because $\left\|\eta_s \E \left[\bv_s^t \mid \wf^t \right]\right\| \asymp \eta_s a$ and the property that  $  \E\|\bx - \E [\bx]\|^2 =  \E\|\bx \|^2 -  \|\E [\bx]\|^2$ for any random variable $\bx$. Therefore, we can conclude that the variance of updates for  oversampling is potentially much smaller than   that of  weight adjusting.

\paragraph{More Discussions on Convex Relaxation and Cross-Entropy Loss.}

We show Program
\eqref{eq: NN simplified model} can also be relaxed as a  nuclear norm-constrained convex optimization.  The result  heavily relies on the progress of matrix decomposition, e.g. \cite{bach2008convex, haeffele2019structured}. We will  use the equality (see e.g., \cite[Section 2]{bach2008convex}) that for any matrix $\bZ$ and $a>0$,  
\begin{equation}\label{eq:nuclear rel}
    \| \bZ\|_* = \inf_{r\in  \mathbb{N}_+} \inf_{\bU,\bV: \bU\bV^\top = \bZ}  \frac{a}{2}\|\bU\|^2 + \frac{1}{2a} \|\bV\|^2,
\end{equation}
where  $r$ is the  number of columns for $\bU$ and $\| \cdot\|_*$ denotes the nuclear norm. 

 For any feasible solution $\left(\bH, \bW\right)$ for the original program  \eqref{eq: NN simplified model},  we define 
\begin{equation}\label{eq:define z}
    \bh_k = \frac{1}{n_k}\sum_{i=1}^{n_k} \bh_{k,i},~k\in[K], \quad\tbH=[\bh_1, \bh_2, \dots, \bh_K]\in\RR^{p\times K}, ~~ \text{and} ~~  \bZ = \bW\tbH\in\RR^{K\times K}. 
\end{equation}
We consider the  convex program:
\begin{equation}\label{eq:convex nuclear problem}
    \begin{aligned}
     \min_{\bZ\in \RR^{K\times K}}\quad& \sum_{k=1}^K \frac{n_k}{N}  \cL( \bZ_k, \by_k  )\\
\mathrm{s.t.}\quad&  \| \bZ\|_* \leq K \sqrt{\Etheta\Ew}.
    \end{aligned}
\end{equation}
   where $\bZ_k$ denotes the $k$-th column of $\bZ$ for $k\in[K]$.

\begin{lemma}\label{theo:to convex NUCLEAR}
Assume $p\geq K$ and the loss function $\cL$ is convex on the first argument.   Let $\bZ^{\star}$ be a minimizer of the convex program \eqref{eq:convex nuclear problem}.
Let $r$ be the rank of $\bZ^{\star}$ and consider  thin  
Singular Value Decomposition (SVD) of $\bZ^{\star}$ as $\bZ^{\star}  = \bU^{\star} \bSigma^{\star} \bV^{\star}$.
Introduce two diagonal matrices $\bSigma_1^{\star}$ and $\bSigma_2^{\star}$ with the entries defined as $\bSigma_1^{\star}(i,i) =  \sqrt{\frac{\Ew}{\Etheta}} \sqrt{|\bSigma^{\star}(i,i)|} $ and  $\bSigma_2^{\star}(i,i) =  \sqrt{\frac{\Etheta}{\Ew}} \bSigma^{\star}(i,i)/ \sqrt{|\bSigma^{\star}(i,i)|}$ for $i\in[r]$, respectively.
Let $\left(\bTheta^{\star}, \bW^{\star}\right)$ be
\begin{equation}
    \begin{aligned}
   & \bW = 
    \bU^{\star} \bSigma_1^{\star}\bP^\top, \quad \left[\btheta_{1}^{\star},\btheta_{2}^{\star},\dots,\btheta_{K}^{\star}\right] = \bP \bSigma_2^{\star}\bV^{\star},\\
&\btheta_{k,i}^{\star} = \btheta_{k}^{\star}, \quad k\in[K],~ i\in[n_k], \label{eq:general solution 2} 
    \end{aligned}
\end{equation}
where $\bP\in \RR^{p\times r}$ is any partial orthogonal matrix such that $\bP^\top \bP = \bI_{r}$.  Then $(\bTheta^{\star}, \bW^{\star})$ is a minimizer of   \eqref{eq: NN simplified model}.  
\end{lemma}
\begin{proof}[Proof of Lemma \ref{theo:to convex NUCLEAR}]
 For any feasible solution $\left(\bH, \bW\right)$ for the original program  \eqref{eq: NN simplified model},  define $\bh_k$ for $k\in[K]$, $\tbH$, and $\bZ$ by \eqref{eq:define z}.
We show $\bZ$ is a feasible solution for the convex program
\eqref{eq:convex nuclear problem}. In fact,   by \eqref{eq:nuclear rel} with $r =K$ and $a = \sqrt{\Etheta/\Ew}$, we have 
\begin{align}\label{eq:z bound}
    \left\|\bZ\right\|_* &\leq \frac{\sqrt{\Etheta/\Ew}}{2}\left\| \bW\right\|^2 +   \frac{\sqrt{\Ew/\Etheta}}{2}\left\| \tbH\right\|^2\notag\\
    &\overset{a}\leq   \frac{ \sqrt{\Etheta/\Ew} }{2}\sum_{k=1}^K \|\bw_k\|^2  + \frac{\sqrt{\Ew/\Etheta}}{2} \sum_{k=1}^K \frac{1}{n_k}\sum_{i=1}^{n_k}\left\|\btheta_{k,i} \right\|^2\notag\\ &\leq K \sqrt{\Etheta\Ew}, 
\end{align}
where $\overset{a}\leq$ applies Jensen's inequality as: 
$$
 \left\|\tbH\right\|^2 =  \sum_{k=1}^K\|\bh_k \|^2
 \leq \sum_{k=1}^K \frac{1}{n_k}\sum_{i=1}^{n_k}\left\|\btheta_{k,i} \right\|^2.
$$
Let $L_0$ be the global minimum  of the convex problem  \eqref{eq:convex nuclear problem}. 
Since $\cL$ is convex on the first argument, by the same argument as \eqref{eq:convex fea1}, we obtain, for any feasible solution $\left(\bH, \bW\right)$,
 \begin{align}\label{eq:convex fea1nu}
    \frac{1}{N} \sum_{k=1}^K \sum_{i=1}^{n_k} \cL( \bW\btheta_{k,i}, \by_k ) &=  \sum_{k=1}^K\frac{\barN_k}{N}  \left[\frac{1}{\barN_k}\sum_{k=1}^{n_k} \cL( \bW\btheta_{k,i}, \by_k )\right]\notag\\
    &\geq   \sum_{k=1}^{K} \frac{\barN_k}{N} \cL( \bW\btheta_{k}, \by_k )
    = \sum_{k=1}^K \frac{n_k}{N}  \cL( \bZ_k, \by_k  ) \geq L_0.
 \end{align}
  On the other hand,  for the solution $\left(\bH^{\star}, \bW^{\star}\right)$ defined in \eqref{eq:general solution 2} with $\bZ^{\star}$, we can verify that $\left(\bH^{\star}, \bW^{\star}\right)$ is a feasible solution for  \eqref{eq: NN simplified model} and
\begin{equation}\label{eq:convex fea2nu}
     \frac{1}{N} \sum_{k=1}^K \sum_{i=1}^{n_k} \cL( \bW^{\star}\btheta_{k,i}^{\star}, \by_k ) =  \sum_{k=1}^K  \frac{n_k}{N}\cL( \bZ_k^{\star}, \by_k  ) = L_0.
\end{equation}
 Combining \eqref{eq:convex fea1nu}  and \eqref{eq:convex fea2nu}, we have that  $L_0$ is the global minimum of \eqref{eq: NN simplified model} and $(\bTheta^{\star}, \bW^{\star})$ is a minimizer.
\end{proof}

\begin{property}\label{property:reach const}
For the cross-entropy loss,   we have the following properties.
\begin{itemize}
    \item[] \hypertarget{pro:reach  A}{(i)}   Any   minimizer $\bZ^{\star}$ of \eqref{eq:convex nuclear problem} satisfies that  $ \| \bZ\|_* = \sqrt{\Etheta\Ew}$.
    \item[] \hypertarget{pro:reach B}{(ii)}  Any   minimizer $(\bH^{\star},\bW^{\star})$ of  \eqref{eq: NN simplified model} satisfies $$\frac{1}{K}\sum_{k=1}^K\frac{1}{\barN}\sum_{i=1}^{\barN}\left\|\btheta_{k,i}^{\star} \right\|^2 = \Etheta,\quad \text{and}\quad\quad
     \frac{1}{K}\sum_{k=1}^K\left\|\bw_k^{\star} \right\|^2 = \Ew.$$
    \item[] \hypertarget{pro:reach C}{(iii)}  Any   minimizer $\bX^{\star}$ of  \eqref{eq:convex sdp problem}  satisfies that $$\frac{1}{K}\sum_{k=1}^K \bX^{\star}(k,k) = \Etheta, \quad\text{and}\quad\quad \frac{1}{K}\sum_{k=K+1}^{2K} \bX^{\star}(k,k) =\Ew.$$
\end{itemize}
\end{property}
\begin{proof}[Proof of Property \ref{property:reach const}]
We first prove (\hyperlink{pro:reach A}{i}).  Let $\bZ^{\star}$ be any minimizer of  \eqref{eq:convex nuclear problem}.  Then by the 
Karush–Kuhn–Tucker conditions,  there is a pair $(\lambda, \bXi)$ with  $\lambda\geq0$ and $\bXi\in \partial \|\bZ^{\star}\|_*$ such that
$$ \nabla_{\bZ} \left[\sum_{k=1}^K \frac{n_k}{N}  \cL( \bZ_k^{\star}, \by_k  )\right] + \lambda \bXi = \mathbf{0}^{K\times K}, $$
where $\partial \|\bZ\|_* $ denotes the set of sub-gradient of $\|\bZ\|_*$.  For the cross-entropy loss, one can verify that $ \nabla_{\bZ} \left[\sum_{k=1}^K \frac{n_k}{N}  \cL( \bZ_k, \by_k  )\right]\neq \mathbf{0}^{K\times K}  $ for all $\bZ$. So $\lambda\neq0$. By the complementary slackness condition, we have that $\bZ$ will reach  the boundary of the constraint, achieving (\hyperlink{pro:reach A}{i}).

For (\hyperlink{pro:reach B}{ii}), suppose there is a minimizer $(\bH^{\star}, \bW^{\star})$ of  \eqref{eq: NN simplified model}  such that  $\frac{1}{K}\sum_{k=1}^K\frac{1}{\barN}\sum_{i=1}^{\barN}\left\|\btheta_{k,i}^{\star} \right\|^2 < \Etheta$ or $\frac{1}{K}\sum_{k=1}^K\left\|\bw_k^{\star} \right\|^2 < \Ew$.  Letting $\bZ^{\star}$ defined by \eqref{eq:define z},  it follows from \eqref{eq:convex fea1nu} that  $\bZ^{\star}$ is a minimizer of  \eqref{eq:convex nuclear problem}. However,  by \eqref{eq:z bound},  we have $\| \bZ^{\star}\|_* < \sqrt{\Etheta\Ew}$, which is contradictory to (\hyperlink{pro:reach A}{i}). We obtain (\hyperlink{pro:reach B}{ii}).

For (\hyperlink{pro:reach C}{iii}), suppose there is a minimizer $\bX^{\star}$ of  \eqref{eq:convex sdp problem}  such that $\frac{1}{K}\sum_{k=1}^K \bX^{\star}(k,k) < \Etheta$ or  $ \frac{1}{K}\sum_{k=K+1}^{2K} \bX^{\star}(k,k) <\Ew$. Then letting $(\bH^{\star}, \bW^{\star})$ defined by \eqref{eq:general solution},   $(\bH^{\star}, \bW^{\star})$ is a minimizer of  \eqref{eq: NN simplified model} from Theorem \ref{theo:to convex}. However,  we have  $\frac{1}{K}\sum_{k=1}^K\frac{1}{\barN}\sum_{i=1}^{\barN}\left\|\btheta_{k,i}^{\star} \right\|^2 < \Etheta$ or $\frac{1}{K}\sum_{k=1}^K\left\|\bw_k^{\star} \right\|^2 < \Ew$, which  contradicts to (\hyperlink{pro:reach B}{ii}). We complete the proof.

\end{proof}

\subsection{Additional Experimental Results}
In this part, we provide some additional experimental results for Minority Collapse. As for the experiments for Minority Collapse in Figure \ref{fig:collapse}, the corresponding training and test accuracy are shown in Tables \ref{table:training-accuracy}-\ref{table:test-accuracy}. Furthermore, we find that the pre-trained neural networks on ImageNet (an imbalanced dataset with $K=1000$ classes) that are officially released by Pytorch\footnote{\url{https://pytorch.org/vision/stable/models.html}.} also do not converge to a Simplex ETF, indicating that neural collapse does not emerge during the terminal phase of imbalanced training. Specifically, the minimal (maximal) between-class angle of pre-trained classifiers for VGG19 and ResNet152 are $43^\circ$ ($103^\circ$) and $37^\circ$ ($102^\circ$), respectively. The corresponding standard deviation of between-class angles of pre-trained classifiers for VGG19 and ResNet152 are $4.1^\circ$ and $3.6^\circ$, respectively. More details can be
found in Figure \ref{fig:imagenet}. The phase transition point of the imbalance ratio is in Figure \ref{fig:LPM-simulations-extra} with multiple choices of $\Ew$ and $\Etheta$.

\begin{table}[!htp]
\centering
\scalebox{0.7}{
\begin{tabular}{c||c|c|c||c|c|c||c|c|c||c|c|c}
\hline
Dataset &\multicolumn{6}{c||}{FashionMNIST} & \multicolumn{6}{c}{CIFAR10} \\ \hline
Network architecture & \multicolumn{3}{c||}{VGG11} & \multicolumn{3}{c||}{ResNet18} & \multicolumn{3}{c||}{VGG13} & \multicolumn{3}{c}{ResNet18} \\ \hline
No.~of majority classes & $K_A=3$ & $K_A=5$ & $K_A=7$ & $K_A=3$ & $K_A=5$ & $K_A=7$  & $K_A=3$ & $K_A=5$ & $K_A=7$ & $K_A=3$ & $K_A=5$ & $K_A=7$ \\ \hline
$R=1$  &100 & 100& 100 &100 &100 &100 & 100 & 100 & 100 & 100& 100 & 100\\ \hline
$R=10$ & 100 &100 & 100& 100& 100& 100 & 100 & 100 &100 & 100 & 100& 100\\ \hline
$R=100$ & 100& 100 & 100 &100 &100 & 100 &100 & 100& 100 &100 & 100 & 100 \\ \hline
$R=1000$ &99.87 &99.94 & 99.97 & 99.97 & 99.93 & 99.97 & 99.80& 99.90& 99.96 & 99.90&100 & 99.97 \\ \hline
$R=inf$ &100 &100 &100 & 100 & 100 & 100 & 100& 100 &100 &100 &100   & 100\\ \hline
\end{tabular}
}
\caption{Training accuracy (\%) for different settings.}
\label{table:training-accuracy}
\end{table}

\begin{table}[!htp]
\centering
\scalebox{0.7}{
\begin{tabular}{c||c|c|c||c|c|c||c|c|c||c|c|c}
\hline
Dataset &\multicolumn{6}{c||}{FashionMNIST} & \multicolumn{6}{c}{CIFAR10} \\ \hline
Network architecture & \multicolumn{3}{c||}{VGG11} & \multicolumn{3}{c||}{ResNet18} & \multicolumn{3}{c||}{VGG13} & \multicolumn{3}{c}{ResNet18} \\ \hline
No.~of majority classes & $K_A=3$ & $K_A=5$ & $K_A=7$ & $K_A=3$ & $K_A=5$ & $K_A=7$  & $K_A=3$ & $K_A=5$ & $K_A=7$ & $K_A=3$ & $K_A=5$ & $K_A=7$ \\ \hline
$R=1$  &93.02 & 93.02 & 93.02 &93.80 &93.80 & 93.80&88.62 & 88.62 & 88.62& 88.72& 88.72&  88.72\\ \hline 
$R=10$ & 87.12 & 89.79& 92.00 &86.07 & 88.77 & 92.78 &65.55 &71.80 &80.41 &58.66 & 66.44 & 78.79\\ \hline
$R=100$ & 73.48& 85.00& 88.03 &70.82 &84.62 & 86.24 & 30.87& 48.36& 64.52 &28.90 &45.97 & 62.91 \\ \hline
$R=1000$ &40.10 & 57.61 & 69.09 &45.51 & 57.95 & 66.13& 28.48& 45.44& 61.82 &28.57 &45.10 & 60.89\\ \hline
$R=inf$ &29.39 &47.61 & 63.86 &29.44 &47.72 & 64.59 &28.31 & 44.87 &61.40 & 28.44 & 45.27 & 61.16 \\ \hline
\end{tabular}
}
\caption{Test accuracy (\%) for different settings.}
\label{table:test-accuracy}
\end{table} 

\begin{figure}
		\centering
        \hspace{0.01in}
		\subfigure[$\Ew=0.5$, $\Etheta=5$]{
			\centering
			\includegraphics[scale=0.21]{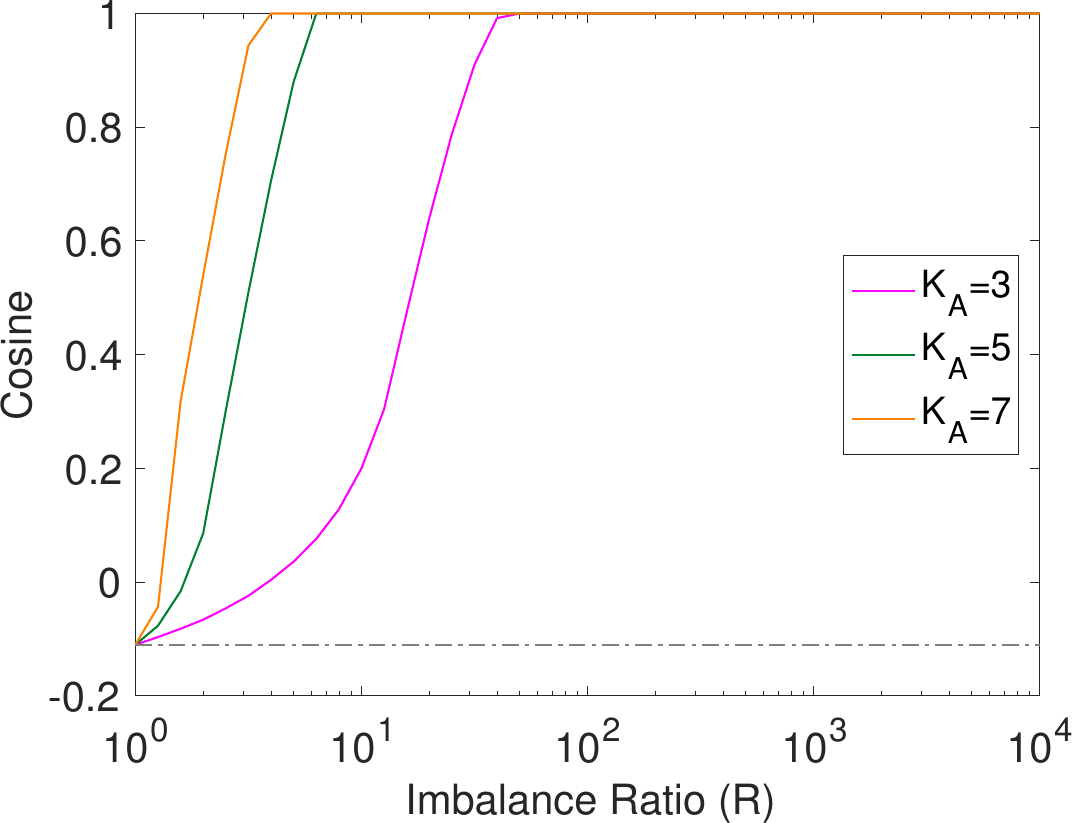}
			\label{fig:simulation-a}}
        \hspace{0.01in}
        \subfigure[$\Ew=0.5$, $\Etheta=10$]{
			\centering
			\includegraphics[scale=0.21]{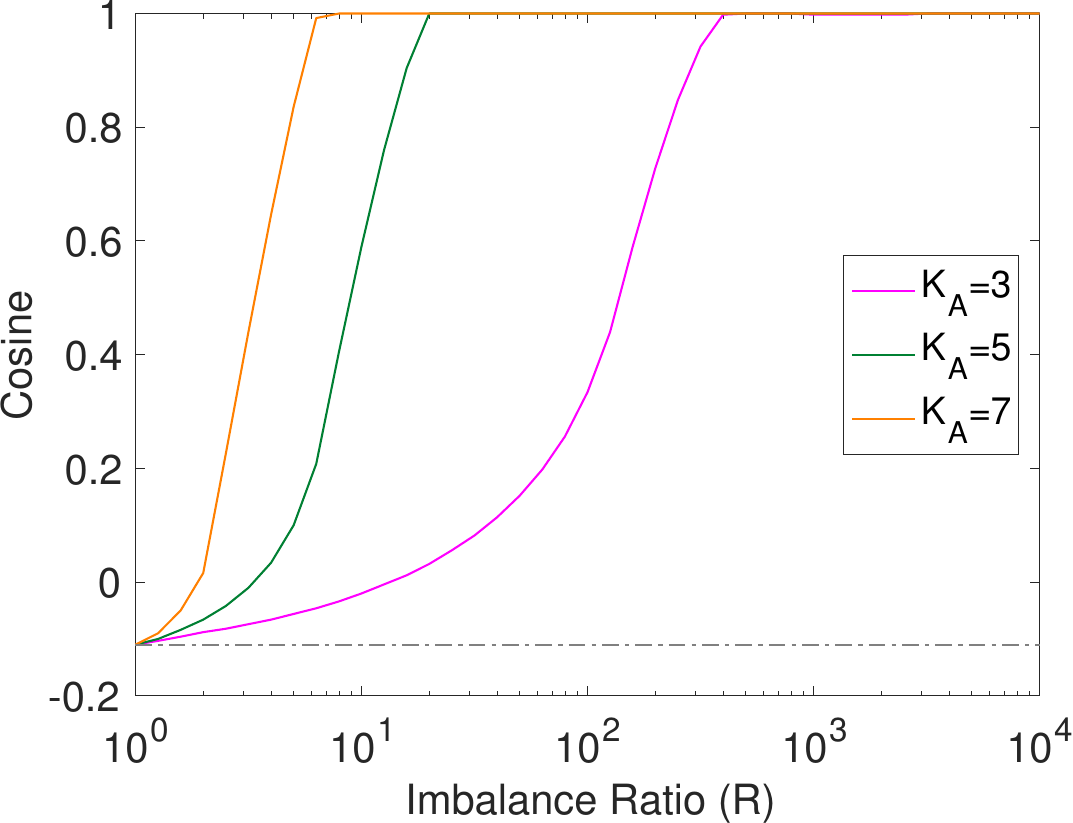}
			\label{fig:simulation-b}}
		\hspace{0.01in}
		\subfigure[$\Ew=1$, $\Etheta=5$]{
			\centering
			\includegraphics[scale=0.21]{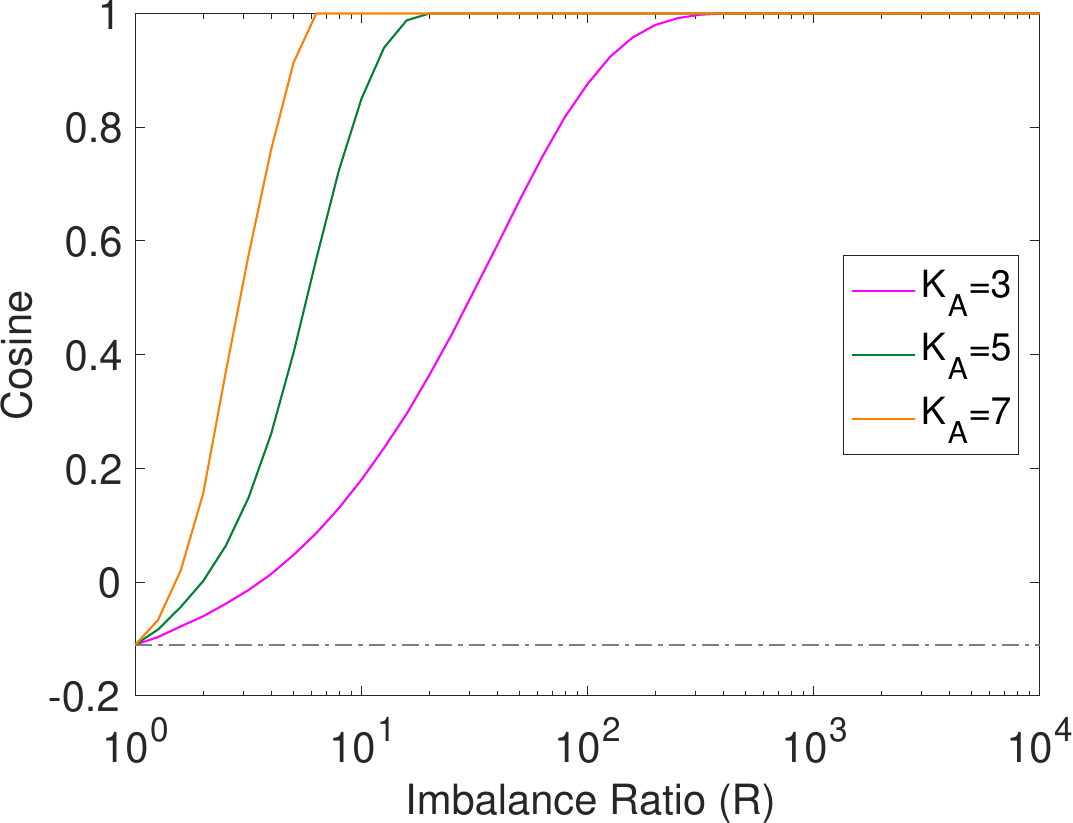}
			\label{fig:simulation-c}}
        \hspace{0.01in}
        \subfigure[$\Ew=1$, $\Etheta=10$]{
			\centering
			\includegraphics[scale=0.21]{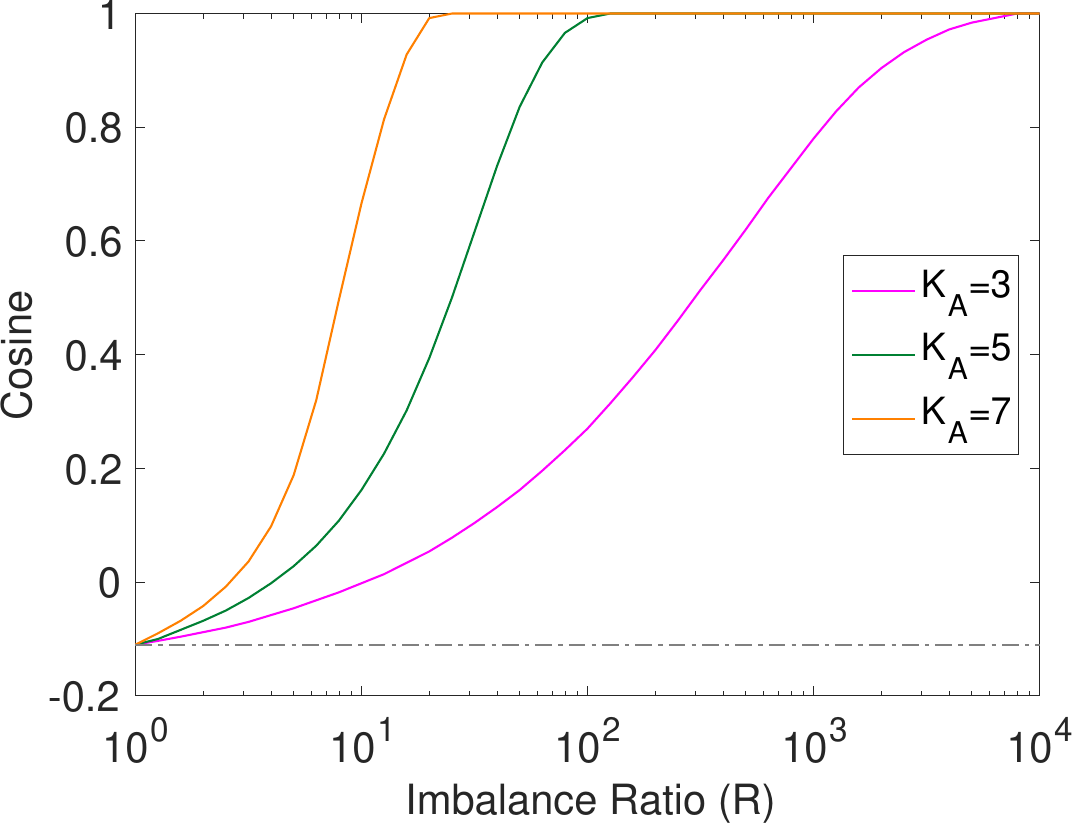}
			\label{fig:simulation-d}}
		\caption{
		 The average cosine of the angles between any pair of the minority classifier solved from the Layer-Peeled Model. The average cosine reaches $1$ once $R$ is above some threshold. The total number of classes $K_A + K_B$ is fixed to $10$. The gray dash-dotted line indicates the value of $-\frac{1}{K-1}$, which is given by \eqref{eq:cos}.   
		}	
\label{fig:LPM-simulations-extra}
\end{figure}

\begin{figure}[t]
		\centering
		\hspace{0.01in}
		\subfigure[Pre-trained VGG19 on ImageNet ]{
			\centering
			\includegraphics[scale=0.45]{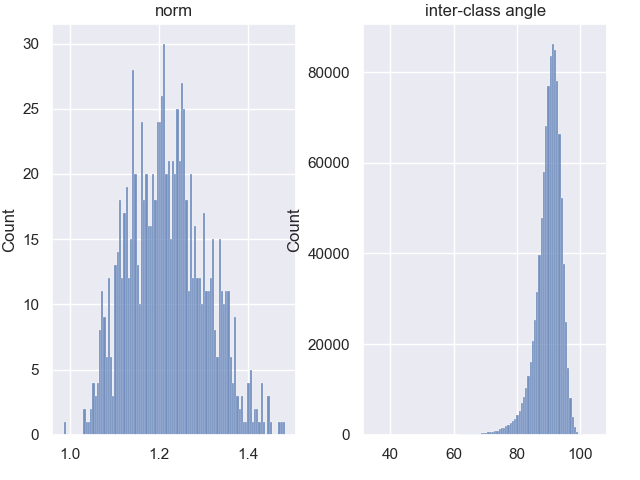}
			\label{fig:imagenet-vgg19}}  
        \hspace{0.01in}
        	\subfigure[Pre-trained ResNet152 on ImageNet]{
			\centering
			\includegraphics[scale=0.45]{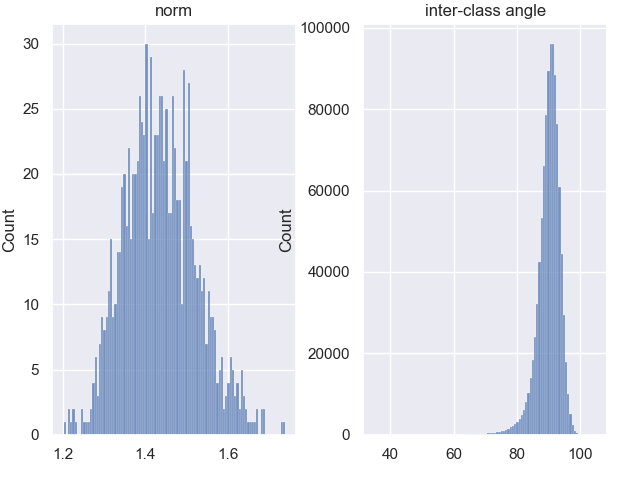}
			\label{fig:imagenet-resnet152}}
		\caption{The neural networks that are pre-trained on ImageNet by PyTorch do not converge to a Simplex ETF.}		
\label{fig:imagenet}
\end{figure}